\documentclass[11pt,letterpaper]{article}
\usepackage[margin=1in]{geometry}
\usepackage[T1]{fontenc}
\usepackage{times}
\usepackage{microtype}
\usepackage[authoryear,round]{natbib}
\setcitestyle{citesep={;},aysep={,},yysep={;}}
\usepackage{amsmath}
\let\amsmathEqref\eqref
\usepackage{amsmath,amsfonts,bm}

\def\eqref#1{equation~\ref{#1}}
\def\1{\bm{1}}

\DeclareMathAlphabet{\mathsfit}{\encodingdefault}{\sfdefault}{m}{sl}
\SetMathAlphabet{\mathsfit}{bold}{\encodingdefault}{\sfdefault}{bx}{n}

\def\gA{{\mathcal{A}}}

\def\gC{{\mathcal{C}}}
\def\gD{{\mathcal{D}}}

\def\gI{{\mathcal{I}}}

\def\gT{{\mathcal{T}}}
\def\gU{{\mathcal{U}}}

\def\gW{{\mathcal{W}}}
\def\gX{{\mathcal{X}}}
\def\gY{{\mathcal{Y}}}

\newcommand{\R}{\mathbb{R}}

\let\eqref\amsmathEqref

\usepackage{amssymb,amsthm}
\usepackage{booktabs,tabularx,longtable,array}
\usepackage{graphicx}
\usepackage{tikz,pgfplots}
\usepgfplotslibrary{groupplots}
\pgfplotsset{compat=1.18}
\usepackage{algorithm,algpseudocode}
\usepackage[hidelinks]{hyperref}
\hypersetup{pdftitle={Hierarchical Utility Calibration for Structured Multiclass Decisions},pdfauthor={Futoshi Futami, Jerry Huang, Ichiro Takeuchi}}
\usepackage{url}

\newtheorem{definition}{Definition}[section]
\newtheorem{theorem}[definition]{Theorem}
\newtheorem{proposition}[definition]{Proposition}
\newtheorem{lemma}[definition]{Lemma}
\newtheorem{corollary}[definition]{Corollary}
\newtheorem{assumption}[definition]{Assumption}
\theoremstyle{remark}
\newtheorem{remark}[definition]{Remark}
\newtheorem{example}[definition]{Example}

\DeclareMathOperator{\UC}{UC}
\DeclareMathOperator{\HUC}{HUC}

\definecolor{MainBlue}{HTML}{244A73}
\definecolor{Purple}{HTML}{6A4C93}

\title{Hierarchical Utility Calibration\\for Structured Multiclass Decisions}
\author{%
Futoshi Futami\textsuperscript{1,4,5}\thanks{Email: \href{mailto:futami.futoshi.es@osaka-u.ac.jp}{\texttt{futami.futoshi.es@osaka-u.ac.jp}}.}\quad
Jerry Huang\textsuperscript{2,5,6}\quad
Ichiro Takeuchi\textsuperscript{3,5}\\[0.6em]
\small\textsuperscript{1}The University of Osaka\\
\small\textsuperscript{2}Mila -- Quebec AI Institute\\
\small\textsuperscript{3}Nagoya University\\
\small\textsuperscript{4}The University of Tokyo\\
\small\textsuperscript{5}RIKEN Center for Advanced Intelligence Project (AIP)\\
\small\textsuperscript{6}Universit\'e de Montr\'eal%
}
\date{\today}

\begin{document}
\maketitle

\begin{abstract}
In multiclass probabilistic prediction, Utility Calibration (UC), which focuses auditing on specified utilities, has recently received attention as a way to guarantee downstream decisions while controlling computational and sample requirements. At the same time, some multiclass problems have meaningful label hierarchies that play important roles in medicine and image classification, yet how UC evaluates utility within a hierarchy remains insufficiently understood. We show that the difference between realized utility and predicted mean utility admits an exact decomposition into a sum of contributions from the internal nodes of the label tree. This decomposition shows that positive and negative contributions from different nodes can cancel, and that even when UC is small, the utility errors remaining in parts of the hierarchy need not be small. To address this problem, we propose Hierarchical Utility Calibration (HUC), which evaluates each node contribution before summation while retaining the same target utility, subgroup, and predicted-utility interval. We further provide finite-sample evaluation over all predicted-utility intervals and propose HUC-Boost, which updates only violated internal nodes, with theoretical guarantees for both.
\end{abstract}

\section{Introduction}
\label{sec:introduction}

Calibration of probabilistic predictions is a basic requirement for using
predicted probabilities in decision-making, risk assessment, and communication
with human users. For multiclass prediction, prior work has studied marginal
and class-wise calibration, which evaluate individual components, as well as
multiclass recalibration methods that transform the entire probability
vector~\citep{Guo2017,Vaicenavicius2019,Kull2019,Fujisawa2025}. The resulting
guarantees and the difficulty of evaluation depend on what is conditioned on
and which directions of prediction error are examined. Canonical
calibration, a strong calibration condition, requires the true label
distribution given the complete prediction vector to agree with that vector
and therefore allows the vector to be reused post hoc for a broad range of
decisions. Yet \citet{Gopalan2024} showed that testing it without assumptions
requires exponentially many samples. 

When the goal is to guarantee only prespecified decisions rather than every possible future use, it is unnecessary to evaluate every direction of the complete prediction vector. Utility Calibration (UC) has recently received attention as a way to focus the audit on specified utilities~\citep{Hegazy2025,Rossellini2025}. UC evaluates, over intervals of predicted mean utility, the agreement between the utility realized in the data and the mean utility that the model computes under the assumption that its predictive distribution is correct. It thereby focuses the audit on decisions that will actually be used, rather than on the entire high-dimensional prediction vector. To guarantee several decisions simultaneously, one includes the corresponding utilities in a utility class.

Multiclass labels, however, may themselves have a semantic hierarchy. In medicine, disease groups, clinical states, and severity levels may be distinguished along a tree~\citep{Yu2025}; in image classification, prediction may proceed from coarse to fine-grained categories~\citep{Deng2009ImageNet,VanHorn2018}. Such predictions need not always traverse the tree all the way to a leaf and return a single final class~\citep{Goren2024,Mortier2022}. A coarse decision may use only an upper-level branch, whereas another decision may further inspect only a particular subtree~\citep{Goren2024,Plaud2025,Mortier2022}. In such hierarchical problems, it has been unclear how UC aggregates errors on the tree into a single utility residual.

We show for the first time that the difference between realized utility and predicted mean utility can be decomposed exactly into a sum of contributions from the internal nodes of the label tree. Because UC evaluates the residual after adding these contributions, positive and negative node contributions can cancel. Consequently, even when UC is small, nonzero utility error may remain in some part of the hierarchy, and UC may fail to identify the branch that should be inspected or corrected. Section~\ref{sec:single-utility-cancellation} demonstrates this phenomenon on a four-leaf tree.

This issue is not resolved by evaluating each leaf separately or by uniformly auditing the branch probabilities of all internal nodes. Calibration of individual leaves does not, in general, determine branch probabilities conditional on having reached the parent subtree. Conversely, auditing every branch either uses a different target population at each node or includes branches that do not affect the target utility. What is needed is to distinguish the node contributions that constitute the original residual without changing the target utility, subgroup, or predicted-utility interval.

We therefore propose Hierarchical Utility Calibration (HUC), which evaluates each node contribution before aggregation. HUC uses the same target utility, subgroup, and predicted-utility interval at every node, allowing node contributions to the utility residual to be compared within the same target population. Moreover, the local utility contrast is zero at nodes that do not affect the target utility, so such nodes are excluded automatically. We also provide a finite-sample estimator of HUC with a theoretical guarantee uniform over all predicted-utility intervals. We then propose HUC-Boost, a post-processing recalibration algorithm that updates only violated internal nodes, and establish preservation of the probabilistic structure, decreasing log loss, and finite termination at the population level. Our main contributions are as follows.
\begin{enumerate}
\item \textbf{Tree decomposition and HUC.} We show that the utility residual admits an exact decomposition into a sum of internal-node contributions (Section~\ref{sec:single-utility-cancellation} and Theorem~\ref{thm:tree-decomp-note}) and propose HUC based on this decomposition (Definition~\ref{def:huc-main}).
\item \textbf{HUC estimation and correction.} We give a finite-sample estimator of HUC over all predicted-utility intervals (Section~\ref{sec:huc-definition}) and propose HUC-Boost, a post-hoc recalibration algorithm (Algorithm~\ref{alg:huc-boost}) with a theoretical guarantee (Theorem~\ref{thm:finite-termination-explicit}).
\item \textbf{Systematic empirical evaluation.} We evaluate HUC-Boost across datasets with different modalities and hierarchical structures, and systematically compare direct UC-Boost and HUC-Boost, parametric recalibration methods, and two-stage combinations of the parametric methods with UC-Boost or HUC-Boost (Section~\ref{sec:numerical-experiments} and Appendices~\ref{app:toy-data} and~\ref{app:experimental-details}).
\end{enumerate}

\section{Preliminaries and a Limitation of Utility Calibration}
\label{sec:preliminary-settings}

\subsection{Multiclass Calibration and Utility Calibration}
\label{sec:calibration}

Calibration asks whether predicted probabilities agree with realized frequencies among observations that receive the same prediction. Let the input be $X\in\gX$, the label be $Y\in\gY:=\{1,\ldots,K\}$, and the probability simplex be $\Delta^{K-1}:=\{p\in[0,1]^K:\sum_{y\in\gY}p_y=1\}$. Let the predictor be $f:\gX\to\Delta^{K-1}$ and write its output as $P:=f(X)$. Unless otherwise specified, $\mathbb{P}$ and $\mathbb{E}$ denote probability and expectation under the data-generating distribution, respectively. For label $y$, let $e_y\in\R^K$ denote the corresponding standard basis vector. The fundamental residual for multiclass probabilistic prediction is $e_Y-P$. The strength of the calibration guarantee and the difficulty of auditing depend on what is conditioned on and which directions of this vector residual are inspected.

\begin{definition}[Canonical calibration]
A predictor $f$ is canonically calibrated if
$\mathbb{E}[e_Y\mid P]=P$ holds almost surely~\citep{Gopalan2024}.
\end{definition}
%Whenever a conditional quantity given $P=p$ appears below, we fix a version of the regular conditional distribution of $Y$ given $P$; all resulting identities are understood to hold for almost every $p$ with respect to the law of $P$.
We adopt an audit-moment formulation of calibration, following the weighted-calibration viewpoint of \citet{Gopalan2024}. For later use, we also include input-dependent subgroup weights, in the spirit of multicalibration~\citep{HebertJohnson2018}. Let $\gC\subset\{c:\gX\to[-1,1]\}$ be a class of measurable subgroup weights and $\gW\subset\{w:\Delta^{K-1}\to[-1,1]^K\}$ a class of measurable prediction-dependent audit directions. For $c\in\gC$ and $w\in\gW$, define
\begin{equation}
M_{c,w}(f)
:=\mathbb{E}\!\left[c(X)\left\langle w(P),e_Y-P\right\rangle\right]
\label{eq:multiclass-audit-moment}
\end{equation}
For $c\equiv1$, vanishing moments for all measurable $w:\Delta^{K-1}\to[-1,1]^K$ characterize canonical calibration; nonconstant $c$ extends the audit to input-defined subgroups.
We call an expected residual of this form an \emph{audit moment}, and a function that multiplies the residual an \emph{audit function}. Given prespecified classes $\gC$ and $\gW$ of audit functions, our goal is to make these audit moments simultaneously small. This naturally leads us to consider the worst-case audit moment $\sup_{c\in\gC,w\in\gW} |M_{c,w}(f)|$.

%The function $c$ specifies which subgroups receive weight, whereas $w$ specifies which direction of the residual is inspected. Multicalibration simultaneously controls these moments for prespecified function classes $\gC,\gW$. Below, $\gC$ denotes the class of subgroup weights used in the audit.

Approaching canonical calibration requires a sufficiently rich family of audit directions that depend on the complete prediction vector. However, \citet{Gopalan2024} showed that testing canonical calibration without additional assumptions requires a number of samples exponential in $K$.  %even when the alternative is separated from canonical calibration by a fixed positive constant. 
For a few known downstream decisions, it is more direct to restrict
auditing to the quantities they require.

Utility Calibration formalizes this idea. Following \citet{Hegazy2025}, we define a \emph{utility} as a bounded measurable function
\begin{equation}
u:\Delta^{K-1}\times\gY\to[-1,1].
\label{eq:general-utility-definition}
\end{equation}
It represents the final utility obtained when outcome $y$ is realized after prediction $p$ has been used, and may already incorporate an action-selection or output rule. For example, let the finite action set be $\gA$, let the realized utility of action $a\in\gA$ be $U(a,y)\in[-1,1]$, and let the prediction-based action rule be $\psi:\Delta^{K-1}\to\gA$. Then $u_\psi(p,y):=U(\psi(p),y)$ is a utility of the form in Eq.~\eqref{eq:general-utility-definition}.
We write $\gU$ for a collection of utilities to be guaranteed simultaneously and call it a utility class.  %Below, a utility included in $\gU$ as a target of the guarantee is called a \emph{target utility}. 
We specify this class before observing the audit sample, using the target actions or action comparisons.

For $p\in\Delta^{K-1}$, let $\widetilde Y\sim p$ be a virtual label drawn from $p$. Define the \emph{predicted mean utility} by%, which the model computes under the assumption that $p$ is correct, by
\begin{equation}
\mu_u(p)
:=\mathbb{E}_{\widetilde Y\sim p}[u(p,\widetilde Y)]
=\sum_{y\in\gY}p_yu(p,y)
\label{eq:pred-utility-note}
\end{equation}
Because $u$ takes values in $[-1,1]$, we have $\mu_u(p)\in[-1,1]$. Under canonical calibration, the conditional mean of the realized utility agrees with $\mu_u(P)$. Unlike canonical calibration, UC focuses auditing on the specified utility class by ranging over intervals of predicted mean utility rather than over all directions of the complete prediction vector. Let $\gI:=\{I\subseteq[-1,1]: I\text{ is an interval}\}\cup\{\varnothing\}$.%, where interval endpoints may be included or excluded.

\begin{definition}[Utility Calibration]
\label{def:uc-general}
The Utility Calibration error with respect to a subgroup class $\gC$ and utility class $\gU$ is
\begin{equation}
\UC(f;\gC,\gU)
:=
\sup_{\substack{c\in\gC,\,u\in\gU,\,I\in\gI}}
\left|
\mathbb{E}\!\left[
 c(X)\mathbf{1}\{\mu_u(P)\in I\}
 \{u(P,Y)-\mu_u(P)\}
\right]
\right|
\label{eq:uc-note}
\end{equation}
%We also call Eq.~\eqref{eq:uc-note}, which includes a general weight $c$, UC. If the set of audit candidates is empty, its supremum is defined as $0$.
\end{definition}
Here, $\mathbf{1}\{\mu_u(P)\in I\}$ equals $1$ if
$\mu_u(P)\in I$ and $0$ otherwise. Small UC means that predicted mean utility agrees with realized utility in each predicted-utility interval. \citet{Hegazy2025} showed that the risk improvement achievable by monotone post-processing preserving the ordering of the one-dimensional score $\mu_u(P)$ is at most $2\UC(f;\{1\},\{u\})$.
%Eq.~\eqref{eq:uc-note} measures the weighted difference between predicted mean utility and realized utility for the population specified by the subgroup weight $c(X)$ and the predicted-mean-utility interval $I$.

\subsection{Cancellation in UC and Requirements for Hierarchical Auditing}
\label{sec:single-utility-cancellation}

In applications with label trees, detecting only the final utility discrepancy is insufficient. For example, suppose the benefit of an intervention is overestimated among patients with the same predicted utility. Whether the error lies in an upper-level disease-group branch or in a subsequent severity branch determines which part of the model should be inspected and which probability should be corrected. We use a four-leaf tree to show that UC can lose this location information and explain why neither leaf-wise auditing nor direct auditing of all branches is sufficient to recover it.

\paragraph{Cancellation on a four-leaf tree.}
Consider the four-leaf labels in Figure~\ref{fig:four-leaf-hierarchy}, $\gY=\{y_1,y_2,y_3,y_4\}$, arranged in a binary tree. Condition on observations for which the predictor outputs the leaf-probability vector $P=p=(p_1,p_2,p_3,p_4)$, where $p_i>0$. The root $v_0$ separates $\gY_L:=\{y_1,y_2\}$ from $\gY_R:=\{y_3,y_4\}$, and $v_L$ and $v_R$ separate the two leaves in their respective subtrees. For $t\in\{0,L,R\}$, denote the left and right children of $v_t$ by $v_{t,L}$ and $v_{t,R}$. Thus $v_{0,L}=v_L$, $v_{0,R}=v_R$, and $(v_{L,L},v_{L,R},v_{R,L},v_{R,R})=(y_1,y_2,y_3,y_4)$.

For $j\in\{L,R\}$, let $q_{t,j}$ and $q_{t,j}^\star$ be the predicted and true probabilities of selecting $v_{t,j}$ conditional on reaching $v_t$, the latter within the group $P=p$. The left-branch probabilities satisfy $q_{t,L}=1-q_{t,R}$ and $q_{t,L}^\star=1-q_{t,R}^\star$. Assume $0<q_{0,R}^\star<1$, so both true lower-level branch probabilities exist.

\begin{figure}[htbp]
\centering
\begin{minipage}[c]{0.23\linewidth}
\centering
\begin{tikzpicture}[
  level distance=10mm,
  level 1/.style={sibling distance=18mm},
  level 2/.style={sibling distance=8mm},
  every node/.style={font=\small},
  edge from parent/.style={draw}
]
\node {$v_0$}
  child {node {$v_L$}
    child {node {$y_1$}}
    child {node {$y_2$}}}
  child {node {$v_R$}
    child {node {$y_3$}}
    child {node {$y_4$}}};
\end{tikzpicture}
\end{minipage}\hfill
\begin{minipage}[c]{0.74\linewidth}
\small
\textbf{Right-branch probabilities}\hfill $q$: predicted, $q^\star$: true
\[
\begin{aligned}
 v_0:\quad q_{0,R}&=p_3+p_4,
 &q_{0,R}^\star&=\mathbb{P}(Y\in\gY_R\mid P=p),\\
 v_L:\quad q_{L,R}&=\frac{p_2}{p_1+p_2},
 &q_{L,R}^\star&=\mathbb{P}(Y=y_2\mid Y\in\gY_L,P=p),\\
 v_R:\quad q_{R,R}&=\frac{p_4}{p_3+p_4},
 &q_{R,R}^\star&=\mathbb{P}(Y=y_4\mid Y\in\gY_R,P=p).
\end{aligned}
\]
\end{minipage}
\caption{A four-leaf tree with predicted and true conditional branch probabilities at each node.}
\label{fig:four-leaf-hierarchy}
\end{figure}
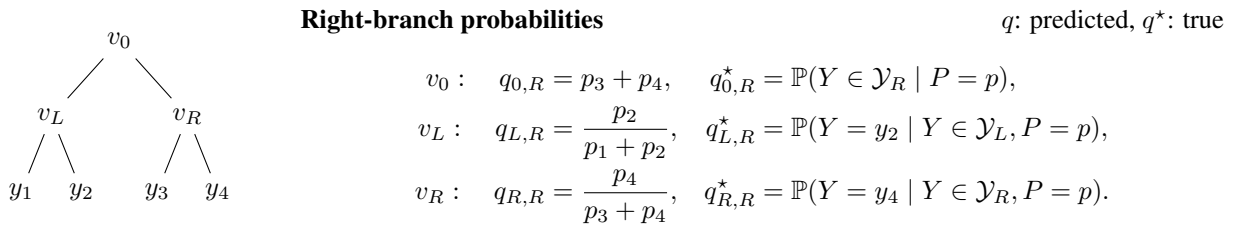

Let $\gY_{v_{t,j}}$ be the leaves below child $v_{t,j}$ and assign it the mean utility
\[
\textstyle\bar u_{t,j}:=\sum_{y\in\gY_{v_{t,j}}}p_yu(p,y)/\sum_{y\in\gY_{v_{t,j}}}p_y,
\qquad t\in\{0,L,R\},\quad j\in\{L,R\}.
\]
This marginalizes the target utility within the child subtree under the predictive conditional distribution; for a leaf child $y_i$, it equals $u_i:=u(p,y_i)$. We suppress dependence on the fixed $u,p$. The predicted mean utility at the root is $\mu_u(p)=\sum_{j=L,R}q_{0,j}\bar u_{0,j}$.

The UC residual $u(p,y)-\mu_u(p)$ is the difference between the realized utility at the leaf and the utility assessed at the root, before the label's location is refined. Its conditional mean $D^{\mathrm{full}}(p):=\mathbb{E}[u(p,Y)\mid P=p]-\mu_u(p)$ decomposes as $D^{\mathrm{full}}(p)
=\sum_{t\in\{\mathrm{top},L,R\}}D_t(p)$,
\begin{equation}
\begin{aligned}
D_{\mathrm{top}}(p)
:=\sum_{j=L,R}q_{0,j}^\star\bar u_{0,j}
\!-\!\sum_{j=L,R}q_{0,j}\bar u_{0,j},\quad D_t(p)
:=q_{0,t}^\star\Bigl[
\sum_{j=L,R}q_{t,j}^\star\bar u_{t,j}
\!-\!\sum_{j=L,R}q_{t,j}\bar u_{t,j}\Bigr],
%\qquad t\in\{L,R\}.
\end{aligned}
\label{eq:d-full-sec32}
\end{equation}
where $t\in\{L,R\}$. Each term compares the same marginalized child utility values $\bar u_{t,j}$ averaged under the true and predicted branch probabilities, weighted by the true probability of reaching the parent: $1$ at the root and $q_{0,L}^\star,q_{0,R}^\star$ at the lower nodes. Equivalently, if identifying child $v_{t,j}$ reveals utility $\bar u_{t,j}$, each term measures whether the parent-level utility assessment agrees on average with the assessment after this refinement. Thus, the root-to-leaf mean discrepancy splits into branching-step mean discrepancies; see Appendix~\ref{subsec:four-leaf-mean-matching} for the proof.

For a general predictor, let $D^{\mathrm{full}}(P):=\mathbb{E}[u(P,Y)-\mu_u(P)\mid P]$. Because the interval indicator is determined by $P$ alone, the tower property and Eq.~\eqref{eq:d-full-sec32} imply
\[
\UC(f;\{1\},\{u\})
\!=\!\sup_{I\in\gI}\big|\mathbb{E}[\mathbf{1}\{\mu_u(P)\in I\}D^{\mathrm{full}}(P)]\big|
\!=\!\sup_{I\in\gI}\big|\mathbb{E}[\mathbf{1}\{\mu_u(P)\in I\}\sum_{t\in\{\mathrm{top},L,R\}}D_t(P)]\big|.
\]
Thus, UC evaluates the node contributions only after adding them within the same predicted-utility interval. If positive and negative contributions cancel (e.g., $D_{\mathrm{top}}=0.1,\,D_L=-0.1,\,D_R=0$ gives $D^{\mathrm{full}}=0$), errors may remain at individual hierarchical levels even when UC is small. Moreover, even when UC is large, its value alone does not identify which internal decision should be inspected or corrected within the same target population.

\paragraph{From standard UC to branch auditing.}\footnote{Detailed discussions of these points are provided in Appendices~\ref{subsec:subtree-utility-counterexample}--\ref{subsec:node-specific-branch-cells}.}
In medicine, the final utility guides decisions, while intermediate hierarchical stages help interpret decision paths and locate errors~\citep{Yu2025}; in hierarchical selective classification, stage-wise confidence determines prediction granularity under uncertainty~\citep{Goren2024}. This motivates auditing the local contributions underlying utility assessment, in addition to final-utility consistency.

%Class-wise calibration of individual leaves does not directly guarantee the tree's conditional structure. Likewise, UC calibration of indicator utilities for nested subtrees generally does not determine branch probabilities conditional on reaching the parent. As another remedy, one could add utilities corresponding to higher-level conditional branch probabilities to the utility class and use standard UC to calibrate each internal-node branch probability. However, what matters is not the branch-probability discrepancy itself but its effect on the target utility. If the target utility is identical across a node's child subtrees, an incorrect branch probability does not change it. Hence the largest branch-probability discrepancy need not identify the largest utility contribution or the node to correct.
Class-wise calibration of individual leaves does not directly guarantee the tree's conditional structure. Likewise, calibrating indicator utilities for nested subtrees with UC generally does not determine the branch probability conditional on reaching a parent node. One could instead add utilities designed to represent such conditional branch probabilities and calibrate each internal node using standard UC. However, what matters is not the branch-probability discrepancy itself, but its effect on the target utility. If the mean utilities assigned to a node's child subtrees are equal, an incorrect branch probability has no effect on their average. Thus, the node with the largest branch-probability discrepancy need not have the largest utility contribution or be the node that should be corrected.

Moreover, branch-specific predicted values define different audit populations, whereas we compare error sources within the common cell $\{\mu_u(P)\in I\}$. Auditing all internal nodes includes utility-irrelevant nodes and their finite-sample fluctuations, giving approximately $|\gU||V^\circ|$ candidates, where $V^\circ$ is the internal-node set. Thus, we localize the utility-relevant contributions in Eq.~\eqref{eq:d-full-sec32} before cancellation, keeping the target utility and predicted-utility cell fixed.

%Moreover, using each branch's predicted value to form its target population audits a different population at every node, whereas we seek to compare error sources within the same group $\{\mu_u(P)\in I\}$ having the same predicted utility. Auditing all branches also includes utility-irrelevant nodes and their finite-sample fluctuations, yielding approximately $|\gU||V^\circ|$ candidates across multiple utilities. We do not claim that HUC is always computationally cheaper than standard UC; rather, obtaining local information about the same target utility requires restricting the audit to nodes that affect that utility. Accordingly, HUC retains the original target utility and predicted-utility cell, assigns each contribution in Eq.~\eqref{eq:d-full-sec32} to its hierarchical location before cancellation, and retains only the nodes that affect the target utility. Concrete counterexamples illustrating these distinctions are collected in Appendix~\ref{subsec:subtree-utility-counterexample}--\ref{subsec:node-specific-branch-cells}.

\section{Hierarchical Utility Calibration}
\label{sec:hierarchical-utility-calibration-main}
Section~\ref{sec:single-utility-cancellation} showed that UC can obscure
which internal nodes generate a target-utility error. We now extend the
four-leaf decomposition to a general tree and define HUC from node
contributions evaluated with the same subgroup, target utility, and
predicted-utility interval.

\subsection{Utility-Error Decomposition and HUC on General Trees}
\label{sec:huc-definition}
\begingroup

Let $\gT=(V,E_{\gT})$ be a fixed finite rooted tree with leaf set
$\gY$, root $\rho$, and internal-node set $V^\circ$.
For $a\in V$, let $\gY_a$ be the leaves in the subtree rooted at $a$,
with $\gY_y=\{y\}$ for a leaf $y$.
Each internal node $v$ has children $v_1,\ldots,v_{d_v}$, whose
leaf sets form a disjoint partition of $\gY_v$.

Assume $P=f(X)$ satisfies $P_y>0$ almost surely for every $y\in\gY$.
We first define local quantities at a fixed $p\in\Delta^{K-1}$ with strictly positive coordinates,
then substitute $p=P$ to audit the predictor.
For $a\in V$, $v\in V^\circ$, and $j=1,\ldots,d_v$, set
\begin{equation}
\pi_a(p):=\sum_{y\in\gY_a}p_y,\qquad
\mu_{u,a}(p):=\frac{\sum_{y\in\gY_a}p_yu(p,y)}{\pi_a(p)},\qquad
q_{v,j}(p):=\frac{\pi_{v_j}(p)}{\pi_v(p)}.
\label{eq:tree-local-quantities-compact}
\end{equation}
These are the predictive reach probability, subtree mean utility, and branch probability given the parent, respectively.
Write $\boldsymbol q_v(p):=(q_{v,1}(p),\ldots,q_{v,d_v}(p))^\top$.
Then $\sum_jq_{v,j}(p)=1$, $\mu_{u,\rho}(p)=\mu_u(p)$, and $\mu_{u,y}(p)=u(p,y)$ at each leaf $y$.

The child utility values in
Section~\ref{sec:single-utility-cancellation} are
$\bar u_{t,j}=\mu_{u,v_{t,j}}(p)$.
For each $a\in V$, let $J_a(y):=\mathbf{1}\{y\in\gY_a\}$
indicate whether $y$ lies in the subtree rooted at $a$.
For an internal node $v$ with child $v_j$, write
$J_{v,j}(y):=J_{v_j}(y)=\mathbf{1}\{y\in\gY_{v_j}\}$,
and define the branch residual and utility contrast by
\begin{equation}
\boldsymbol R_v(p,y)
:=\bigl(J_{v,j}(y)-J_v(y)q_{v,j}(p)\bigr)_{j=1}^{d_v},\qquad
\boldsymbol\Delta_{u,v}(p)
:=\bigl(\mu_{u,v_j}(p)-\mu_{u,v}(p)\bigr)_{j=1}^{d_v}.
\label{eq:node-utility-inner-product-main}
\end{equation}
For $a\in V$, let $\pi_a^\star(p):=\mathbb{P}(Y\in\gY_a\mid P=p)$ be the true reach probability.
If $\pi_v^\star(p)>0$, define $q_{v,j}^\star(p):=\pi_{v_j}^\star(p)/\pi_v^\star(p)$.
The conditional mean of the node contribution satisfies
\begin{equation}
\mathbb{E}\!\left[
\left\langle\boldsymbol\Delta_{u,v}(p),\boldsymbol R_v(p,Y)\right\rangle
\mid P=p\right]
=\pi_v^\star(p)\Bigl[
\sum_{j=1}^{d_v}q_{v,j}^\star(p)\mu_{u,v_j}(p)
-\sum_{j=1}^{d_v}q_{v,j}(p)\mu_{u,v_j}(p)\Bigr].
\label{eq:mean-node-contribution-main}
\end{equation}
If $\pi_v^\star(p)=0$, the conditional mean is zero without defining $q_{v,j}^\star(p)$.
See Lemma~\ref{lem:conditional-residual} in Appendix~\ref{subsec:general-tree-decomposition} for the proof. Each child utility value $\mu_{u,v_j}(p)$ marginalizes $u(p,y)$ within $\gY_{v_j}$ using the predictive conditional probabilities $p_y/\pi_{v_j}(p)$. The two sums in brackets of Eq.~\eqref{eq:mean-node-contribution-main} average these same values under the true and predicted branch probabilities. Their difference measures mean-utility disagreement at this branching step, weighted by the true reach probability $\pi_v^\star(p)$. These conditional means generalize $D_{\mathrm{top}}(p)$, $D_L(p)$, and $D_R(p)$ in Eq.~\eqref{eq:d-full-sec32}. The same inner products decompose the utility residual pointwise.
\begin{theorem}[Utility-error decomposition on a general tree]
\label{thm:tree-decomp-note}
For every target utility $u$, every $p\in\Delta^{K-1}$ with
strictly positive coordinates, and every $y\in\gY$,
\begin{equation}
u(p,y)-\mu_u(p)
=\sum_{v\in V^\circ}
\left\langle
\boldsymbol\Delta_{u,v}(p),\boldsymbol R_v(p,y)
\right\rangle.
\label{eq:tree-decomp-note}
\end{equation}
\end{theorem}
Each contribution equals the realized child's mean utility minus the parent's, or zero outside the parent subtree.
Along the root-to-leaf path, intermediate means cancel, leaving the realized leaf utility minus the predicted root mean.
This recovers the original residual and locates its branching-step contributions without changing the target utility.
See Theorem~\ref{thm:tree-decomposition} in Appendix~\ref{subsec:tree-decomposition-proof}; Appendix~\ref{subsec:binary-special-case} gives the binary-tree correspondence.

For fixed $\gT$ and $p$, the node contributions are mutually orthogonal under the predictive distribution $\widetilde Y\sim p$.
They form the unique decomposition whose components vanish outside their respective subtrees, are constant on each child subtree, and have zero $p$-mean; see Appendix~\ref{subsec:orthogonal-node-decomposition}.
Thus, the tree, predictive distribution, and target utility determine the node-wise attribution.
Under the true distribution, however, nonzero expected node contributions can still cancel.

We now substitute $(p,y)=(P,Y)$ into
Eq.~\eqref{eq:tree-decomp-note} to audit the input-dependent prediction
$P=f(X)$.
Using the same audit conditions $(c,u,I)$ as UC, define the node-wise
audit moment for $v\in V^\circ$ by
\begin{equation}
\Gamma_{c,u,I,v}(f)
:=\mathbb{E}\!\left[
c(X)\mathbf{1}\{\mu_u(P)\in I\}
\left\langle
\boldsymbol\Delta_{u,v}(P),\boldsymbol R_v(P,Y)
\right\rangle
\right].
\label{eq:gamma-note}
\end{equation}
Multiplying the pointwise decomposition by
$c(X)\mathbf{1}\{\mu_u(P)\in I\}$ and taking expectations gives
\begin{equation}
\mathbb{E}\!\left[
c(X)\mathbf{1}\{\mu_u(P)\in I\}\{u(P,Y)-\mu_u(P)\}
\right]
=\sum_{v\in V^\circ}\Gamma_{c,u,I,v}(f).
\label{eq:uc-huc-main-comparison}
\end{equation}
UC takes the absolute value after summing these moments.
In contrast, the HUC error defined below audits branching-step
changes in utility assessment separately under the same $(c,u,I)$,
preventing cancellation across nodes.
The sum in Eq.~\eqref{eq:uc-huc-main-comparison} can be restricted
to the utility-relevant nodes
\begin{equation}
\mathcal V_{\gT}(u)
:=\{v\in V^\circ:
\boldsymbol\Delta_{u,v}\not\equiv\boldsymbol 0\}.
\label{eq:relevant-node-set-local-closure}
\end{equation}
Here, $\boldsymbol\Delta_{u,v}\not\equiv\boldsymbol 0$ means that
the child means $\mu_{u,v_1}(p),\ldots,\mu_{u,v_{d_v}}(p)$
are not all equal for some strictly positive prediction $p$.
Thus, $\mathcal V_{\gT}(u)$ contains precisely the nodes where
identifying the child containing the label can change the
target-utility assessment; all other nodes contribute zero.

\begin{definition}[Hierarchical Utility Calibration]
\label{def:huc-main}
The HUC error is
\begin{equation}
\HUC(f;\gC,\gU)
:=\sup_{c\in\gC,\,u\in\gU,\,I\in\gI,\,v\in\mathcal V_{\gT}(u)}
|\Gamma_{c,u,I,v}(f)|.
\label{eq:huc-note}
\end{equation}
%If the audit-candidate set is empty, its supremum is defined to be zero.
\end{definition}

Intuitively, the utility assessed at the parent should agree on average,
within the audited population, with the utility assessed after learning
which child contains the realized label.
HUC evaluates this mean consistency node by node under the same $(c,u,I)$.
Partial sums control calibration errors for subtree-mean utilities at fixed
boundaries of varying branch depths, with bounds proportional to the included
utility-relevant node count; see Appendix~\ref{app:postprocessing-huc}.

The following proposition establishes the relationships among UC, HUC, and canonical calibration.

\begin{proposition}
\label{prop:uc-huc-full-main}
For every subgroup class $\gC$ and target utility $u$,
\begin{equation}
\UC(f;\gC,\{u\})
\le |\mathcal V_{\gT}(u)|\,\HUC(f;\gC,\{u\}).
\label{eq:uc-huc-main-bound}
\end{equation}

Moreover, if $f$ is canonically calibrated, then
$\HUC(f;\{1\},\gU)=0$.
\end{proposition}
See Appendix~\ref{subsec:full-huc-uc-proof} for the proof.
Appendix~\ref{app:postprocessing-huc} further relates HUC to
monotone post-processing guarantees when the target utility is
used only down to an internal-subtree boundary.

\endgroup

To estimate HUC, let $(X_i,Y_i)_{i=1}^n$ be independent and
identically distributed (i.i.d.) copies of $(X,Y)$, and write
$P_i:=f(X_i)$.
Define
\[
\begin{aligned}
\widehat{\HUC}(f;\gC,\gU)
&:=\sup_{\substack{
c\in\gC,\,u\in\gU,\,I\in\gI, v\in\mathcal V_{\gT}(u)}}
\left|\frac1n\sum_{i=1}^n
c(X_i)\mathbf{1}\{\mu_u(P_i)\in I\}
\left\langle
\boldsymbol\Delta_{u,v}(P_i),
\boldsymbol R_v(P_i,Y_i)
\right\rangle\right|.
\end{aligned}
\]

\begin{theorem}[Finite-sample evaluation of HUC]
\label{thm:finite-sample-main}
Assume that $f$, $\gT$, and the finite classes $\gC,\gU$ are fixed
independently of the audit sample.
For any $\delta\in(0,1)$, with probability at least $1-\delta$,
\begin{equation}
\left|\widehat{\HUC}(f;\gC,\gU)-\HUC(f;\gC,\gU)\right|
\le \frac{16}{\sqrt n}
+4\sqrt{\frac{\log(\max(1,|\gC|\sum_{u\in\gU}|\mathcal V_{\gT}(u)|)/\delta)}{2n}}.
\label{eq:finite-sample-main}
\end{equation}
\end{theorem}
The proof and a comparison of the computational costs of estimating UC and HUC are given in Appendices~\ref{subsec:finite-sample-proof} and~\ref{subsec:uc-huc-computational-cost}, respectively.

\subsection{HUC-Boost: Branch-Local Post-Hoc Correction}
\label{sec:postprocess}
Given a base predictor, we propose HUC-Boost, a boosting-style
post-hoc correction algorithm for reducing its HUC error;
Algorithm~\ref{alg:huc-boost} gives the population version.
As in the UC post-processing algorithm of \citet{Hegazy2025}
(hereafter UC-Boost), each iteration selects an audit candidate
whose moment exceeds a prespecified tolerance in magnitude and
updates the predictor in the corresponding direction.
Unlike UC-Boost, HUC-Boost changes branch probabilities
only at the internal node with the selected violation.

To implement this node-local update, we parameterize the predictor
by node-wise logits.
At iteration $t$, let
$\boldsymbol z_{t,v}(x)\in\mathbb R^{d_v}$ be the logits at node $v$,
with corresponding branch probabilities
$q_{t,v,j}(x):=e^{z_{t,v,j}(x)}/
\sum_{k=1}^{d_v}e^{z_{t,v,k}(x)}$ for $j=1,\ldots,d_v$.
For each leaf $y$, the prediction $f_t(x)_y$ is the product of these branch probabilities along the path
from the root to $y$.
Appendix~\ref{subsec:softmax-leaf-distribution} gives this
construction and verifies its consistency with the branch probabilities.

For $a=(c,u,I,v)$, the node-$v$ logit direction used in Step~4 of
Algorithm~\ref{alg:huc-boost} is
\[
\boldsymbol h_a^t(x)
:=c(x)\mathbf{1}\{\mu_u(f_t(x))\in I\}
\boldsymbol\Delta_{u,v}(f_t(x)).
\]

%In Algorithm~\ref{alg:huc-boost}, $J_v(Y)=\mathbf{1}\{Y\in\gY_v\}$ indicates whether the observed label lies in the subtree rooted at $v$.
\begin{algorithm}[htbp]
\caption{HUC-Boost: logit update (population version)}
\label{alg:huc-boost}
\begin{algorithmic}[1]
\Statex \textbf{Input:} Positive initial predictor $f_0$, candidate classes $\gC,\gU,\gI$, tolerance $\varepsilon>0$
\State $z_{0,v,j}(x)\gets\log q_{v,j}(f_0(x))$ for all $v,j$; $t\gets0$
\While{there exists a candidate with $|\Gamma_{c,u,I,v}(f_t)|>\varepsilon$}
  \State Select any such $a_t=(c_t,u_t,I_t,v_t)$
  \State $\boldsymbol h_t\gets\boldsymbol h_{a_t}^t$
  \State $\displaystyle \Gamma_t:=\Gamma_{c_t,u_t,I_t,v_t}(f_t),\;\Lambda_t:=\tfrac14\mathbb E\!\left[J_{v_t}(Y)\{\max_jh_{t,j}(X)-\min_jh_{t,j}(X)\}^2\right],\;\alpha_t:=\Gamma_t/\Lambda_t$
  \State $\boldsymbol z_{t+1,v_t}(x)\gets\boldsymbol z_{t,v_t}(x)+\alpha_t\boldsymbol h_t(x)$ for all $x$; keep all other nodes fixed
  \State Reconstruct $f_{t+1}$ from the softmax branch probabilities by path products; $t\gets t+1$
\EndWhile
\State \Return $f_t$
\end{algorithmic}
\end{algorithm}

For the candidate $a_t$ selected in Step~3, write $v:=v_t$.
The logit update in Step~6 and reconstruction in Step~7 give,
for $y\in\gY_{v_j}$,
\[
q_{t+1,v,j}(x)
=\frac{q_{t,v,j}(x)e^{\alpha_t h_{t,j}(x)}}
{\sum_kq_{t,v,k}(x)e^{\alpha_t h_{t,k}(x)}},
\qquad
f_{t+1}(x)_y
=f_t(x)_y\frac{q_{t+1,v,j}(x)}{q_{t,v,j}(x)}.
\]
The update redistributes the predicted mass reaching $v$ among its
child subtrees. It rescales all leaf probabilities within each child
subtree by the same factor, preserving their conditional distribution,
the reach probability of $v$, and all leaf probabilities outside $\gY_v$.

The step size in Algorithm~\ref{alg:huc-boost} is justified by
the global leaf log loss: $\Gamma_t$ is its negative directional
derivative along the fixed direction $\boldsymbol h_t$, and
$\Lambda_t>0$ bounds the curvature whenever an update is made.
This yields the following guarantee.
\begin{theorem}[Loss decrease and finite termination]
\label{thm:finite-termination-explicit}
Let $\mathcal{L}_{\log}(f):=\mathbb E[-\log f(X)_Y]$.
Assume $\mathcal{L}_{\log}(f_0)<\infty$.
Algorithm~\ref{alg:huc-boost} decreases population log loss monotonically.
It terminates within
$T\le\left\lceil 2\mathcal{L}_{\log}(f_0)/\varepsilon^2\right\rceil$
updates and returns $f_T$ satisfying $\HUC(f_T;\gC,\gU)\le\varepsilon$.
\end{theorem}

Appendix~\ref{subsec:huc-boost-theorem-proof} gives the derivations
and proof of Theorem~\ref{thm:finite-termination-explicit}.
The auditing implementation and theoretical analysis of finite-sample
HUC-Boost are given in Appendices~\ref{subsec:exact-interval-scan}
and~\ref{app:randomized-initialization}, respectively.
Appendix~\ref{subsec:leaf-brier-nonlocality} discusses comparisons
with UC-Boost and Brier-based alternatives.

\section{Related Work}
\label{sec:related-work}
Prior work covers confidence and class-wise
calibration~\citep{Guo2017,Vaicenavicius2019},
multicalibration within prespecified
subgroups~\citep{HebertJohnson2018}, and kernel-based
calibration tests~\citep{Widmann2019}.
Canonical calibration conditions on the entire predictive vector,
but its distribution-free testing can require exponentially many
samples in the number of classes~\citep{Gopalan2024}.
Utility Calibration (UC) audits prespecified utilities over
predicted-utility intervals~\citep{Hegazy2025}, generalizing
binary Cutoff Calibration Error~\citep{Rossellini2025}.
Other decision-oriented approaches include decision calibration
for specified losses~\citep{zhao_calibrating_2021},
task-specific kernel calibration
metrics~\citep{marx_calibration_2023}, and decision calibration
for nonlinear losses~\citep{tang_dimension-free_2026}.
These works study which decisions and losses predictions can
reliably support and which calibration conditions are needed.
Our focus is complementary: we decompose the residual of the
same specified target utility along a semantic label hierarchy
and evaluate its node contributions before cancellation,
under the same subgroup and predicted-utility interval.
Additional comparisons with related work appear in
Appendix~\ref{app:additional-related-work}.

Hierarchical classification has a long history, including methods that place classifiers at the nodes of a hierarchy, share statistical information between parents and children, or learn a structured classifier over the hierarchy~\citep{KollerSahami1997,McCallum1998,CesaBianchi2006,Rousu2006,SillaFreitas2011}.
Further work has studied classification consistent with a hierarchical distance~\citep{Ramaswamy2015}, methods that return an ancestor when the prediction is uncertain in order to trade specificity for accuracy~\citep{Deng2012,Goren2024}, set-valued predictions represented by internal nodes or multiple nodes~\citep{Mortier2022}, and decoding rules adapted to hierarchical evaluation measures~\citep{Plaud2025}.
Most of these methods use the hierarchy in classifier training, loss design, or output selection, rather than addressing probabilistic calibration.
In particular, the calibrated surrogates of \citet{Ramaswamy2015}. refer to consistency with respect to a hierarchical classification loss and are different from the probabilistic calibration studied here.
Work that directly addresses predictive-probability calibration on a hierarchy includes methods that calibrate the internal binary classifiers of a nested dichotomy or the probability distribution produced by the complete tree~\citep{Leathart2019}.
In contrast, HUC localizes utility errors to relevant nodes without requiring calibration of every internal branch probability.

\section{Experiments}
\label{sec:numerical-experiments}
\label{sec:real-data-support2-v20}
\label{sec:real-data-inaturalist-v16}

\subsection{Toy data}
\paragraph{Setup.}
In a controlled toy data setting, we test whether the cancellation between internal-node contributions discussed in Sections~\ref{sec:single-utility-cancellation} and~\ref{sec:huc-definition} allows UC-Boost to reduce UC while leaving nonzero HUC. We use the medically motivated synthetic eight-leaf tree in Figure~\ref{fig:toy8-sample-size-main}(a), with seven internal nodes, three decision utilities, and whole-population audits. Each utility evaluates a binary action choice based on outcomes distinguished by the tree. For each sample size $n$, UC- and HUC-Boost share the initial prediction and the $n$ calibration labels. Figure~\ref{fig:toy8-sample-size-main}(b) reports population UC (top) and HUC (bottom) under the known label distribution. Curves show means over 200 repetitions; error bars at selected sizes show one sample standard deviation. Appendix~\ref{subsec:toy-main-experiment} gives the full setup.

\begin{figure}[t]
\centering
\begin{minipage}[t]{0.40\textwidth}
\centering
{\small\bfseries (a) Controlled eight-leaf tree}\par
\resizebox{\linewidth}{!}{\begin{tikzpicture}[x=1.08cm,y=0.94cm,
  every node/.style={font=\footnotesize,inner sep=0.8pt,fill=white},
  treeedge/.style={draw=black!72,line width=0.60pt,shorten >=1.6pt,shorten <=1.6pt}]
\path[use as bounding box,draw=none] (-2.55,0.30) rectangle (2.55,-3.480);
\node (r) at (0,0) {$v_{\mathrm{cause}}$};
\node (ni) at (-1.529,-0.800) {$v_{\mathrm{noninf}}$};
\node (sp) at (0.963,-0.800) {$v_{\mathrm{spread}}$};
\node (y1) at (-2.058,-1.600) {$y_1$};
\node (y2) at (-1.001,-1.600) {$y_2$};
\node (si) at (0.142,-1.600) {$v_{\mathrm{site}}$};
\node (sh) at (1.794,-1.600) {$v_{\mathrm{shock}}$};
\node (up) at (-0.453,-2.400) {$v_{\mathrm{upper}}$};
\node (lo) at (0.727,-2.400) {$v_{\mathrm{lower}}$};
\node (y7) at (1.463,-2.400) {$y_7$};
\node (y8) at (2.124,-2.400) {$y_8$};
\node (y3) at (-0.774,-3.200) {$y_3$};
\node (y4) at (-0.132,-3.200) {$y_4$};
\node (y5) at (0.453,-3.200) {$y_5$};
\node (y6) at (1.020,-3.200) {$y_6$};
\foreach \a/\b in {r/ni,r/sp,ni/y1,ni/y2,sp/si,sp/sh,si/up,si/lo,sh/y7,sh/y8,up/y3,up/y4,lo/y5,lo/y6}{\draw[treeedge] (\a)--(\b);}
\end{tikzpicture}}
\end{minipage}\hfill
\begin{minipage}[t]{0.585\textwidth}
\centering
{\small\bfseries (b) UC-Boost and HUC-Boost}\par
\begin{tikzpicture}
\begin{groupplot}[
  group style={group size=1 by 2,vertical sep=1.8mm},
  width=0.965\linewidth,height=2.75cm,
  xmode=log,log basis x=10,xmin=200,xmax=1250000,
  xtick={250,1000,5000,20000,100000,1000000},
  xticklabels={250,1k,5k,20k,100k,1M},
  scaled ticks=false,
  tick label style={font=\scriptsize},label style={font=\scriptsize},
  grid=major,grid style={gray!18},
  axis line style={black!70},tick style={black!70},
  every axis plot/.append style={line width=1.00pt,mark size=1.70pt},clip=false
]
\nextgroupplot[
  ymode=log,log basis y=10,ymin=0.0002,ymax=0.055,
  ytick={0.0005,0.002,0.01,0.05},
  yticklabels={$5\!\times\!10^{-4}$,$2\!\times\!10^{-3}$,$10^{-2}$,$5\!\times\!10^{-2}$},
  ylabel={UC},xticklabels=\empty,
  legend columns=2,
  legend style={font=\scriptsize,draw=none,fill=white,fill opacity=0.94,text opacity=1,
    at={(0.5,1.06)},anchor=south,/tikz/every even column/.append style={column sep=5pt}}
]
\addplot[MainBlue,only marks,mark=none,forget plot,
  error bars/.cd,y dir=both,y explicit,
  error bar style={line width=0.40pt,opacity=0.52},
  error mark options={rotate=90,mark size=1.2pt,line width=0.40pt,opacity=0.52}]
  coordinates {(250,0.034932) +- (0,0.016083) (5000,0.007774) +- (0,0.004054) (1000000,0.000562) +- (0,0.000293)};
\addplot[Purple,only marks,mark=none,forget plot,
  error bars/.cd,y dir=both,y explicit,
  error bar style={line width=0.40pt,opacity=0.52},
  error mark options={rotate=90,mark size=1.2pt,line width=0.40pt,opacity=0.52}]
  coordinates {(500,0.025513) +- (0,0.012754) (20000,0.004056) +- (0,0.002125) (200000,0.001262) +- (0,0.000662)};
\addplot[MainBlue,solid,line width=1.10pt,mark=o,mark options={fill=white,line width=0.70pt}]
 coordinates {(250,0.034932) (500,0.025513) (1000,0.017419) (2000,0.012611) (5000,0.007774) (10000,0.005458) (20000,0.004056) (50000,0.002492) (100000,0.001786) (200000,0.001262) (500000,0.000796) (1000000,0.000562)};
\addlegendentry{UC-Boost}
\addplot[Purple,densely dashed,line width=1.10pt,mark=x,mark options={line width=0.90pt}]
 coordinates {(250,0.034932) (500,0.025513) (1000,0.017419) (2000,0.012611) (5000,0.007774) (10000,0.005458) (20000,0.004056) (50000,0.002492) (100000,0.001786) (200000,0.001262) (500000,0.000796) (1000000,0.000562)};
\addlegendentry{HUC-Boost}

\nextgroupplot[
  ymode=log,log basis y=10,ymin=0.0002,ymax=0.34,
  ytick={0.0005,0.005,0.05,0.2},
  yticklabels={$5\!\times\!10^{-4}$,$5\!\times\!10^{-3}$,$5\!\times\!10^{-2}$,$0.2$},
  ylabel={HUC},xlabel={Calibration sample size $n$}
]
\addplot[MainBlue,only marks,mark=none,forget plot,
  error bars/.cd,y dir=both,y explicit,
  error bar style={line width=0.40pt,opacity=0.50},
  error mark options={rotate=90,mark size=1.2pt,line width=0.40pt,opacity=0.50}]
  coordinates {(250,0.181382) +- (0,0.063502) (5000,0.179586) +- (0,0.062132) (1000000,0.179682) +- (0,0.062164)};
\addplot[Purple,only marks,mark=none,forget plot,
  error bars/.cd,y dir=both,y explicit,
  error bar style={line width=0.40pt,opacity=0.50},
  error mark options={rotate=90,mark size=1.2pt,line width=0.40pt,opacity=0.50}]
  coordinates {(250,0.028765) +- (0,0.009749) (5000,0.006358) +- (0,0.002105) (1000000,0.000471) +- (0,0.000159)};
\addplot[MainBlue,solid,line width=1.10pt,mark=o,mark options={fill=white,line width=0.70pt}]
 coordinates {(250,0.181382) (500,0.179954) (1000,0.180376) (2000,0.180391) (5000,0.179586) (10000,0.179673) (20000,0.179551) (50000,0.179551) (100000,0.179607) (200000,0.179653) (500000,0.179680) (1000000,0.179682)};
\addplot[Purple,densely dashed,line width=1.10pt,mark=diamond*,mark options={fill=Purple}]
 coordinates {(250,0.028765) (500,0.020818) (1000,0.014450) (2000,0.009947) (5000,0.006358) (10000,0.004644) (20000,0.003271) (50000,0.002075) (100000,0.001461) (200000,0.001038) (500000,0.000660) (1000000,0.000471)};
\end{groupplot}
\end{tikzpicture}
\end{minipage}
\caption{Toy tree and population UC/HUC after UC- and HUC-Boost with $n$ calibration labels.}
\label{fig:toy8-sample-size-main}
\end{figure}
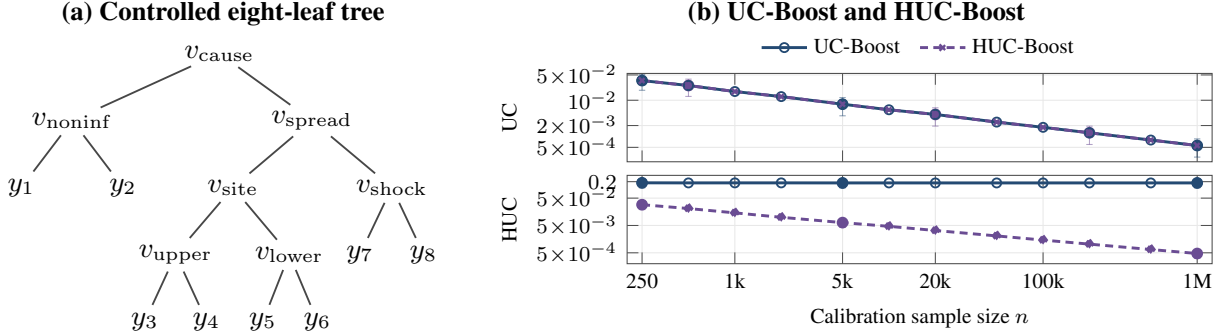

\paragraph{Results.}
The UC errors after UC- and HUC-Boost are similar and decrease as $n$ grows. Their HUC errors differ: after UC-Boost, HUC remains, whereas after HUC-Boost it decreases with sample size. HUC-Boost reduces both errors by addressing node contributions before cancellation. This is consistent with Proposition~\ref{prop:uc-huc-full-main}, which bounds UC by a multiple of HUC. In contrast, UC-Boost reduces the sum without necessarily reducing each contribution, illustrating the decomposition in Eq.~\eqref{eq:uc-huc-main-comparison} and the cancellation discussed in Section~\ref{sec:single-utility-cancellation}.

\subsection{Real data}
\paragraph{Setup.}
We compare HUC-Boost with other post-hoc methods on real datasets. SUPPORT2~\citep{Knaus1995SUPPORT} contains tabular data from seriously ill hospitalized adults. Five outcomes are arranged in a medically motivated binary tree, and nine utilities are obtained from three medical action comparisons evaluated at three parameter values. On iNaturalist~\citep{VanHorn2018}, images are classified into 30 species using a nonuniform biological taxonomy of depth seven, and top-1 correctness at the class, order, family, genus, and species levels is used as the utility. Appendix~\ref{app:experimental-details} provides detailed experimental settings and experiments on additional real datasets.

Both datasets use six base predictors---logistic regression, Gaussian naive Bayes, decision tree, random forest, XGBoost, and multilayer perceptron---using five splits into training, calibration, validation, and test sets. For iNaturalist, each image is converted to a 512-dimensional feature vector using an ImageNet-pretrained ResNet-18, and the base predictors are trained on these features. Base predictors are fitted on the training data; UC-Boost, HUC-Boost, and other recalibration methods are fitted on the calibration data. Validation data select regularization and the checkpoint returned along each boosting path, and test data are used only for final evaluation.

We compare a base predictor (Baseline), UC- and HUC-Boost, temperature and vector scaling~\citep{Guo2017}, and Dirichlet calibration~\citep{Kull2019} (see Appendix~\ref{app:experimental-details} for details). We also apply HUC-Boost after each parametric method, using the same calibration data for both stages; see Appendix~\ref{app:randomized-initialization} for the algorithm and its theoretical properties. The first-stage parametric recalibrator adjusts the predictive distribution globally, and HUC-Boost then targets the remaining node-wise utility errors locally. This global--local combination is expected to improve recalibration efficiency.

\begin{table}[t]
\caption{Independent-test results: mean $\pm$ sample standard deviation over six base predictors and five splits. UC-Boost and HUC-Boost are abbreviated as UC-B and HUC-B, respectively.}
\label{tab:real-data-compact-main}
\centering
\begin{minipage}{\linewidth}
\centering
\begingroup
\scriptsize
\setlength{\tabcolsep}{1.5pt}
\resizebox{\linewidth}{!}{%
\begin{tabular}{@{}l*{5}{c}@{\hspace{5pt}}*{5}{c}@{}}
\toprule
& \multicolumn{5}{c}{\textbf{iNaturalist}} & \multicolumn{5}{c}{\textbf{SUPPORT2}}\\
\cmidrule(lr){2-6}\cmidrule(lr){7-11}
Method & Acc. $\uparrow$ & AUC $\uparrow$ & UC $\downarrow$ & HUC $\downarrow$ & Updates
       & Acc. $\uparrow$ & AUC $\uparrow$ & UC $\downarrow$ & HUC $\downarrow$ & Updates\\
\midrule
Baseline & .489$\pm$.090 & .900$\pm$.108 & .272$\pm$.161 & .125$\pm$.059 & 0.0$\pm$0.0 & .499$\pm$.199 & .676$\pm$.083 & .167$\pm$.178 & .148$\pm$.142 & 0.0$\pm$0.0 \\
UC-B & .462$\pm$.088 & .914$\pm$.066 & \textbf{.045$\pm$.018} & .036$\pm$.011 & 25.4$\pm$24.1 & .575$\pm$.051 & .681$\pm$.079 & .039$\pm$.014 & .050$\pm$.031 & 33.9$\pm$39.4 \\
HUC-B & .463$\pm$.090 & .922$\pm$.060 & .075$\pm$.032 & .034$\pm$.009 & 145.9$\pm$180.5 & .575$\pm$.045 & .698$\pm$.060 & .036$\pm$.010 & .035$\pm$.009 & 45.2$\pm$46.9 \\
\midrule
Temp. & .489$\pm$.090 & .890$\pm$.110 & .130$\pm$.114 & .092$\pm$.076 & 0.0$\pm$0.0 & .499$\pm$.199 & .686$\pm$.072 & .077$\pm$.069 & .076$\pm$.069 & 0.0$\pm$0.0 \\
Temp.$\to$HUC & .470$\pm$.088 & \textbf{.930$\pm$.050} & .047$\pm$.014 & \textbf{.029$\pm$.008} & 43.1$\pm$25.9 & .587$\pm$.032 & .703$\pm$.057 & .032$\pm$.006 & .032$\pm$.006 & 20.4$\pm$22.6 \\
\midrule
Vec. & \textbf{.489$\pm$.095} & .899$\pm$.100 & .114$\pm$.092 & .083$\pm$.073 & 0.0$\pm$0.0 & .575$\pm$.053 & .691$\pm$.065 & .038$\pm$.015 & .038$\pm$.015 & 0.0$\pm$0.0 \\
Vec.$\to$HUC & .478$\pm$.095 & .929$\pm$.053 & .050$\pm$.018 & .031$\pm$.008 & 28.2$\pm$22.4 & .589$\pm$.032 & .701$\pm$.051 & \textbf{.031$\pm$.008} & \textbf{.031$\pm$.008} & 14.3$\pm$12.5 \\
\midrule
Dir. & .436$\pm$.112 & .912$\pm$.063 & .105$\pm$.053 & .055$\pm$.023 & 0.0$\pm$0.0 & .579$\pm$.046 & .712$\pm$.066 & .035$\pm$.013 & .035$\pm$.013 & 0.0$\pm$0.0 \\
Dir.$\to$HUC & .435$\pm$.112 & .912$\pm$.062 & .109$\pm$.052 & .056$\pm$.023 & 4.3$\pm$10.5 & \textbf{.589$\pm$.032} & \textbf{.718$\pm$.050} & .032$\pm$.008 & .032$\pm$.007 & 13.3$\pm$14.1 \\
\bottomrule
\end{tabular}%
}
\endgroup
\end{minipage}
\end{table}

\paragraph{Results.}
Table~\ref{tab:real-data-compact-main} reports independent-test metrics and boosting update counts. Direct HUC-Boost substantially reduces HUC relative to Baseline in both datasets. It uses more updates than UC-Boost, consistent with their update scopes: UC-Boost adjusts all leaf logits globally, whereas HUC-Boost changes branch logits at one internal node per update.

All two-stage variants use substantially fewer HUC-Boost updates than direct HUC-Boost. HUC-Boost after temperature or vector scaling yields lower mean HUC and higher mean AUC in both datasets. The Dirichlet-based variant improves these metrics only on SUPPORT2.

\section{Conclusion and Limitations}
\label{sec:conclusion-limitations}
In this paper, we have shown that, in multiclass probabilistic prediction with hierarchical labels, the utility residual decomposes into internal-node contributions. Based on this decomposition, we proposed HUC to evaluate these contributions before cancellation and HUC-Boost to correct violated branches. We further presented a two-stage recalibration algorithm that combines parametric recalibration methods with HUC-Boost. Numerical experiments demonstrated the effectiveness of the proposed approach in reducing hierarchical utility error.

Our analysis assumes a semantically fixed finite rooted label tree and a utility class specified before analysis. It therefore does not cover decisions outside \(\gU\), relationships absent from the tree, or settings in which the hierarchical structure changes dynamically. Moreover, as with Utility Calibration, our method concerns the agreement between predicted and realized utilities in expectation and does not by itself guarantee the optimality of the resulting decisions. Addressing these limitations will be an important direction for future work toward more reliable decision-making.

\subsection*{Acknowledgments}
This work was supported by JST ASPIRE Grant Number JPMJAP25B1 and JSPS KAKENHI Grant Number JP26K02974, Japan.

\subsection*{AI use statement}
Generative AI tools assisted with theoretical model and conceptual framework development, proof exposition, experimental design and implementation, synthetic-data generation, and data preparation. They also supported literature review and information retrieval, translation and editing, and preparation of code, figures, other artifacts, and references. The authors reviewed all AI-assisted work, checked all mathematical arguments, verified each reference against its original source, and tested the experimental implementation. The authors take full responsibility for this work.

\bibliographystyle{plainnat}
\bibliography{references_verified}

\clearpage
\appendix

% Keep appendix floats close to the discussion with compact, normal spacing.
\setlength{\parskip}{2pt plus 1pt minus 1pt}
\setlength{\abovedisplayskip}{8pt plus 2pt minus 2pt}
\setlength{\belowdisplayskip}{8pt plus 2pt minus 2pt}
\setcounter{topnumber}{3}
\setcounter{bottomnumber}{2}
\setcounter{totalnumber}{5}
\setlength{\floatsep}{10pt plus 2pt minus 2pt}
\setlength{\textfloatsep}{10pt plus 2pt minus 2pt}
\setlength{\intextsep}{8pt plus 2pt minus 2pt}
\renewcommand{\floatpagefraction}{0.75}
\makeatletter
\setlength{\@fptop}{0pt}
\setlength{\@fpsep}{12pt plus 2pt minus 2pt}
\setlength{\@fpbot}{0pt plus 1fil}
\renewcommand{\paragraph}{\@startsection{paragraph}{4}{\z@}%
  {1.75ex plus 0.5ex minus 0.2ex}{-1em}{\normalfont\normalsize\bfseries}}
\makeatother

\section*{Guide to the Appendices}
\phantomsection

\label{app:appendix-guide}

\begingroup
\renewcommand{\theHtable}{appendixguide.\arabic{table}}
\setlength{\tabcolsep}{5pt}
\renewcommand{\arraystretch}{1.04}
\begin{longtable}{@{}>{\raggedright\arraybackslash}p{0.30\textwidth}>{\raggedright\arraybackslash}p{0.65\textwidth}@{}}
\toprule
Appendix & Content and proof locations \\
\midrule

\hyperref[app:formal-setup]
{\textbf{\ref*{app:formal-setup}: Formal definitions and tree properties}}
& Defines the probability space, label-tree quantities, subtree mean utilities, and the path-product representation. \\

\hyperref[app:binary-first]
{\textbf{\ref*{app:binary-first}: Binary classification and trees}}
& Develops the binary-tree calculations and path decomposition. Appendix~\ref{subsec:four-leaf-mean-matching} directly proves Eq.~\eqref{eq:d-full-sec32} and the UC identity in Section~\ref{sec:single-utility-cancellation}. \\

\hyperref[app:general-tree]
{\textbf{\ref*{app:general-tree}: General multiway trees}}
& Appendix~\ref{subsec:general-local-algebra} proves Eq.~\eqref{eq:node-utility-inner-product-main}; Appendix~\ref{subsec:general-tree-decomposition} proves Theorem~\ref{thm:tree-decomp-note}, Eqs.~\eqref{eq:tree-decomp-note} and~\eqref{eq:mean-node-contribution-main}, and the mean-consistency decomposition; Appendix~\ref{subsec:orthogonal-node-decomposition} proves orthogonality, uniqueness, and the Parseval identity. \\

\hyperref[app:uc-huc-full]
{\textbf{\ref*{app:uc-huc-full}: Basic properties of HUC}}
& Appendix~\ref{subsec:uc-huc-full} proves Proposition~\ref{prop:uc-huc-full-main}, Eq.~\eqref{eq:uc-huc-main-bound}, and the implication from canonical calibration. Appendices~\ref{subsec:subtree-counterexample}--\ref{subsec:node-specific-branch-cells} give counterexamples concerning conditional branches, utility impact, and node-specific audit populations. \\

\hyperref[app:huc-boost]
{\textbf{\ref*{app:huc-boost}: Proof of HUC-Boost}}
& Appendices~\ref{subsec:softmax-leaf-distribution} and~\ref{subsec:huc-boost-theorem-proof} analyze Algorithm~\ref{alg:huc-boost} and prove Theorem~\ref{thm:finite-termination-explicit}; Appendix~\ref{subsec:leaf-brier-nonlocality} gives the ordinary-leaf-Brier comparison; Appendix~\ref{subsec:softmax-auxiliary-results} proves the path-product results. \\

\hyperref[app:finite-sample]
{\textbf{\ref*{app:finite-sample}: Finite-sample evaluation}}
& Appendix~\ref{subsec:finite-sample-proof} proves Theorem~\ref{thm:finite-sample-main}; Appendix~\ref{subsec:exact-interval-scan} derives the exact tie-preserving scan; Appendix~\ref{subsec:uc-huc-computational-cost} derives the exact UC--HUC operation counts. \\

\hyperref[app:postprocessing-huc]
{\textbf{\ref*{app:postprocessing-huc}: Post-processing and HUC}}
& Uses the decomposition in Appendix~\ref{subsec:general-tree-decomposition} to prove calibration and one-step refinement guarantees for subtree-mean utilities, including Proposition~\ref{prop:subtree-mean-calibration-refinement}, and relates them to monotone post-processing. \\

\hyperref[app:randomized-initialization]
{\textbf{\ref*{app:randomized-initialization}: Empirical HUC-Boost and fixed-step sample complexity}}
& Algorithm~\ref{alg:empirical-huc-boost} gives the adaptive- and fixed-step implementations. Appendices~\ref{rtf:sec:main-statements}--\ref{rtf:sec:learned-proof} prove fixed-step sample-complexity results for fixed and learned initializers, including Algorithm~\ref{alg:randomized-boost}. \\

\hyperref[app:additional-related-work]
{\textbf{\ref*{app:additional-related-work}: Additional related work}}
& Compares HUC with U-calibration, loss outcome indistinguishability and omniprediction, and utility-directed conformal prediction. \\

\hyperref[app:toy-data]
{\textbf{\ref*{app:toy-data}: Toy data}}
& Specifies the main eight-leaf correction experiment and the decision-utility construction. \\

\hyperref[app:experimental-details]
{\textbf{\ref*{app:experimental-details}: Real data}}
& Gives the common protocol and the data, hierarchies, utilities, subgroups, and results for iNaturalist, SUPPORT2, MASSIVE, NSL-KDD, and N-BaIoT. \\

\bottomrule
\end{longtable}
\addtocounter{table}{-1}% The uncaptioned appendix guide is not a numbered table.
\endgroup

\section{Formal Definitions and Basic Properties of the Tree}
\label{app:formal-setup}

For a node $v$, let $\gY_v$ denote its descendant-leaf set. We use the subtree-membership indicators
$J_v(y):=\mathbf{1}\{y\in\gY_v\}$, $J_{v,j}(y):=J_{v_j}(y)=\mathbf{1}\{y\in\gY_{v_j}\}$,
where $v_j$ is a child of $v$. Each indicator is one when $y$ belongs to the corresponding subtree and zero otherwise.

This notation is used throughout the remaining appendices unless stated otherwise.

\subsection{Utility Functions and Decision-Making}
\label{subsec:tie-breaking-convention}
\paragraph{Utility and its forecast.}
In UC, a utility evaluates the consequence of using $p$ when $y$ is observed,
possibly incorporating a decision \citep[Sections~3 and~3.3]{Hegazy2025}:
\[
u:\Delta^{K-1}\times\gY\to[-1,1],\qquad
\mu_u(p)=\mathbb{E}_{\widetilde Y\sim p}[u(p,\widetilde Y)]
        =\sum\nolimits_y p_yu(p,y).
\]
Its utility classes are built from outcome values, ranks, or action payoffs and selection rules.

\paragraph{Linear utilities.}
A fixed payoff vector gives
$u_b(p,y)=b_y$, $\mu_{u_b}(p)=\langle p,b\rangle$, $b\in[-1,1]^K$.
Varying $b$ gives $\gU_{\mathrm{lin}}$; class-wise indicators use $b=e_j$.
For a tree node $v$, let $\gY_v$ be its descendant-leaf set and write
$J_v(y):=\mathbf{1}\{y\in\gY_v\}$ and $\pi_v(p):=\sum_{z\in\gY_v}p_z$.
Thus $J_v$ is a 0--1 indicator of subtree membership, with $\gY_y=\{y\}$ at a leaf.
Subtree indicators $b_y=J_v(y)$ are instances.

\paragraph{Rank-based utilities.}
Let $r_p(y)$ be the rank of $y$ under decreasing predicted probabilities. Then
$u_\theta(p,y)=\theta_{r_p(y)}$, $\theta\in[-1,1]^K$.
Varying $\theta$ gives $\gU_{\mathrm{rank}}$.
The choice $\theta_r=\mathbf{1}\{r\le k\}$ gives top-$k$ correctness
($1\le k\le K$), with top-class correctness at $k=1$.

\paragraph{Decision-induced utilities.}
For a finite action set $\gA$, a payoff $U:\gA\times\gY\to[-1,1]$ evaluates actions,
whereas a measurable rule $\psi:\Delta^{K-1}\to\gA$ selects them. Set
\begin{equation}
u_\psi(p,y)=U(\psi(p),y),\qquad
\psi_U(p)=\operatorname*{arg\,max}_{\prec;\,a\in\gA}
\sum\nolimits_z p_zU(a,z).
\label{eq:ordered-argmax-convention}
\end{equation}
Varying $U$ with $m=|\gA|$ and $\psi=\psi_U$ gives UC's $\gU_{\mathrm{dec},m}$
(equivalently, $U(a,y)=-\ell(y,a)$). General $u_\psi$ also permits other rules.
\emph{Output correctness} for a reported node $v$ uses
$U(v,y)=J_v(y)=\mathbf{1}\{y\in\gY_v\}$, hence
\[
u_\psi(p,y)=J_{\psi(p)}(y)=\mathbf{1}\{y\in\gY_{\psi(p)}\}, \qquad \mu_{u_\psi}(p)=\sum_{z\in\gY_{\psi(p)}}p_z=\pi_{\psi(p)}(p).
\]
The utility is one when the reported subtree contains the true label and zero otherwise; its predicted mean is the probability assigned to that subtree.
This covers iNaturalist's taxonomic outputs (Appendix~\ref{app:inat-details-v16})
and the status, family, and leaf outputs of NSL-KDD and N-BaIoT
(Appendices~\ref{subsec:nsl-kdd-data} and~\ref{subsec:nbaiot-data}), using their respective reporting rules.
\emph{Binary action comparisons} in Toy and SUPPORT2
(Appendices~\ref{subsec:toy-main-experiment} and~\ref{app:support2-details-v20}) use $\psi_U$ with synthetic payoffs
\[
U(a,y)=(2a-1)\{\mathbf{1}_S(y)-\tau\mathbf{1}_D(y)\},\quad
 a\in\{0,1\},\quad S\subseteq D\subseteq\gY,\quad \tau\in[0,1].
\]
\emph{Abstention and specificity} in MASSIVE (Appendix~\ref{subsec:massive-data})
use the prescribed confidence-threshold rule in Eq.~\eqref{eq:massive-action}
and graded payoffs in Eq.~\eqref{eq:massive-complete-utility}, rather than requiring $\psi=\psi_U$.

UC audits the accuracy of $\mu_u$ as a utility forecast, not decision optimality.
Our experiments specify finite target classes $\gU$. Fix these and total orders
$\prec$ independently of the audit sample; use the first maximizer and ordered
ranks at ties. Measurable scores then give measurable rules.

\subsection{Probability Space, Label Tree, and Predictor}
\label{subsec:probability-tree}
Let $(\Omega,\mathcal F,\mathbb{P})$ be a probability space, with input $X:\Omega\to\gX$ and label $Y:\Omega\to\gY$. Let $\gY=\{1,\ldots,K\}$. Write the predictor and its output as
$f:\gX\to\Delta^{K-1}$, $P:=f(X)$.
Throughout the theory, assume that $P_y>0$ almost surely for every $y\in\gY$. Whenever a deterministic prediction vector $p$ is used, likewise assume that $p_y>0$ for every $y$.

Fix a version of the regular conditional distribution of $Y$ given $P$, and write
\begin{equation}
\eta_y(p):=\mathbb{P}(Y=y\mid P=p),\qquad y\in\gY
\label{eq:regular-conditional-label}
\end{equation}
All subsequent identities conditioned on $P=p$ are understood to hold for almost every $p$ under the law of $P$.
\begin{definition}[Canonical calibration]
\label{def:full-conditional-calibration-app}
A predictor $f$ is canonically calibrated if, for every $y\in\gY$,
\begin{equation}
\mathbb{P}(Y=y\mid P)=P_y\qquad\mathbb{P}\text{-a.s.}
\label{eq:full-conditional-calibration-app}
\end{equation}
In the notation of~Eq.~\eqref{eq:regular-conditional-label}, this is equivalent to $\eta_y(P)=P_y$ almost surely.
\end{definition}
Let $\gT=(V,E_{\gT})$ be a finite rooted tree with root $\rho$, leaf set $\gY$, and internal-node set $V^\circ$. Every non-root node has a unique parent. Write the children of an internal node $v$ as
\begin{equation}
\operatorname{ch}(v)=\{v_1,\ldots,v_{d_v}\},\qquad d_v\ge1
\label{eq:child-set}
\end{equation}
Here $J_a(y)=\mathbf{1}\{y\in\gY_a\}$ indicates membership in the descendant-leaf set $\gY_a$ of node $a$, and $J_{v,j}=J_{v_j}$ for a child $v_j$.
The number of children $d_v$ may vary across nodes, and one-child nodes are allowed. If $d_v=1$, then $q_{v,1}=1$, $J_{v,1}=J_v$, and $\mu_{u,v_1}=\mu_{u,v}$, so $\boldsymbol R_v=\boldsymbol\Delta_{u,v}=\boldsymbol 0$. Such a node contributes to neither the decomposition nor HUC, is never updated, and may be retained or contracted.

For each node $a\in V$, define its descendant leaf set by
\begin{equation}
\gY_a:=\{y\in\gY:\text{the path from $\rho$ to $y$ passes through $a$}\}
\label{eq:descendant-leaf-set}
\end{equation}
Thus $\gY_\rho=\gY$, and $\gY_y=\{y\}$ for a leaf $y$. Since the child subtrees of $v$ have disjoint leaf sets,
\begin{equation}
\gY_v=\dot\bigcup_{j=1}^{d_v}\gY_{v_j}.
\label{eq:child-partition}
\end{equation}

Fix a leaf $y\in\gY$. Write the \emph{ordered sequence} of nodes on the unique path from the root $\rho$ to $y$ as
\begin{equation}
\bigl(w_0(y),w_1(y),\ldots,w_{L_y}(y)\bigr)
\label{eq:root-leaf-path}
\end{equation}
where
\begin{equation}
w_0(y)=\rho,\qquad w_{L_y}(y)=y,\qquad w_{r+1}(y)\in\operatorname{ch}\bigl(w_r(y)\bigr)\quad(r=0,\ldots,L_y-1)
\label{eq:root-leaf-path-endpoints}
\end{equation}
For each $r=0,\ldots,L_y-1$, let
\begin{equation}
w_{r+1}(y)=\bigl(w_r(y)\bigr)_{j_r(y)}
\label{eq:path-child-index}
\end{equation}
define the unique index $j_r(y)\in\{1,\ldots,d_{w_r(y)}\}$ of the child selected by the path at $w_r(y)$. Write the set of internal nodes on the path as
\begin{equation}
\operatorname{Path}(y):=\{w_0(y),\ldots,w_{L_y-1}(y)\}
\label{eq:path-node-set}
\end{equation}
If $v=w_r(y)\in\operatorname{Path}(y)$, also write $j_v(y):=j_r(y)$. When computing for a fixed $y$, we sometimes abbreviate $w_r:=w_r(y)$ and $j_r:=j_r(y)$ for readability.
\subsection{Tree Probabilities and Subtree Mean Utilities}
\label{subsec:subtree-quantities}
For a prediction vector $p\in\Delta^{K-1}$ and a node $a\in V$, define the \emph{predicted reach probability} of $a$ by
\begin{equation}
\pi_a(p):=\sum_{y\in\gY_a}p_y
\label{eq:subtree-probability}
\end{equation}
All leaf probabilities are positive, so $\pi_a(p)>0$ for every node $a$. Moreover, $\gY_\rho=\gY$, $\gY_y=\{y\}$, and $p\in\Delta^{K-1}$ imply
\begin{equation}
\pi_\rho(p)=\sum_{z\in\gY_\rho}p_z=\sum_{z\in\gY}p_z=1, \qquad \pi_y(p)=\sum_{z\in\gY_y}p_z=\sum_{z\in\{y\}}p_z=p_y\qquad(y\in\gY).
\label{eq:root-leaf-probability}
\end{equation}
Define the predicted conditional probability of proceeding to child $v_j$, conditional on reaching internal node $v$ under the prediction, by
\begin{equation}
q_{v,j}(p):=\frac{\pi_{v_j}(p)}{\pi_v(p)},\qquad j=1,\ldots,d_v
\label{eq:branch-conditional}
\end{equation}
and write $\boldsymbol q_v(p):=(q_{v,1}(p),\ldots,q_{v,d_v}(p))^\top$. By~Eq.~\eqref{eq:child-partition},
\[
\sum_{j=1}^{d_v}q_{v,j}(p)
=\frac{\sum_{j=1}^{d_v}\pi_{v_j}(p)}{\pi_v(p)}
=\frac{\sum_{y\in\gY_v}p_y}{\pi_v(p)}
=1.
\]
Therefore $\boldsymbol q_v(p)\in\Delta^{d_v-1}$.

Rearranging~Eq.~\eqref{eq:branch-conditional} gives, for every leaf $y$ and each $r=0,\ldots,L_y-1$,
\begin{equation}
\pi_{w_{r+1}(y)}(p)=\pi_{w_r(y)}(p)q_{w_r(y),j_r(y)}(p)
\label{eq:path-reach-recursion}
\end{equation}
Thus the next node's predicted reach probability equals the current node's predicted reach probability times the selected branch probability. Taking the product in Eq.~\eqref{eq:branch-conditional} along the path gives
\begin{align}
\prod_{r=0}^{L_y-1}q_{w_r(y),j_r(y)}(p)
&=\prod_{r=0}^{L_y-1}\frac{\pi_{w_{r+1}(y)}(p)}{\pi_{w_r(y)}(p)}\notag\\
&=\frac{\pi_{w_1(y)}(p)}{\pi_{w_0(y)}(p)}\frac{\pi_{w_2(y)}(p)}{\pi_{w_1(y)}(p)}\cdots
  \frac{\pi_{w_{L_y}(y)}(p)}{\pi_{w_{L_y-1}(y)}(p)}\notag\\
&=\frac{\pi_{w_{L_y}(y)}(p)}{\pi_{w_0(y)}(p)}
 =\frac{\pi_y(p)}{\pi_\rho(p)}=p_y.
\label{eq:leaf-factorization}
\end{align}
The intermediate reach probabilities $\pi_{w_1(y)}(p),\ldots,\pi_{w_{L_y-1}(y)}(p)$ each appear once in a numerator and once in a denominator and hence cancel. We then use $w_0(y)=\rho$ and $w_{L_y}(y)=y$ from~Eq.~\eqref{eq:root-leaf-path-endpoints}. Finally, substituting $\pi_\rho(p)=1$ and $\pi_y(p)=p_y$ from~Eq.~\eqref{eq:root-leaf-probability} gives the last equality. Thus each predicted leaf probability is the product of the branch probabilities along its path.

Let $u:\Delta^{K-1}\times\gY\to[-1,1]$ be a utility. For a node $a$, define the numerator of its subtree mean utility and the subtree mean itself by
\begin{align}
N_{u,a}(p)&:=\sum_{y\in\gY_a}p_yu(p,y),\label{eq:subtree-utility-sum}\\
\mu_{u,a}(p)&:=\frac{N_{u,a}(p)}{\pi_a(p)}\label{eq:subtree-utility-mean}
\end{align}
At the root and at a leaf,
\begin{equation}
\mu_{u,\rho}(p)=\sum_{y\in\gY}p_yu(p,y)=:\mu_u(p),\qquad\mu_{u,y}(p)=u(p,y).
\label{eq:root-leaf-utility-mean}
\end{equation}

Define the true conditional reach probability directly from the conditional distribution $\eta(p)$ by
\begin{equation}
\pi_a^\star(p):=\sum_{y\in\gY_a}\eta_y(p)=\mathbb{P}(Y\in\gY_a\mid P=p)
\label{eq:true-subtree-probability}
\end{equation}
This is always defined for every node $a$. Only when $\pi_v^\star(p)>0$ do we define the true conditional probability of selecting child $v_j$, conditional on truly reaching $v$, by
\begin{equation}
q_{v,j}^\star(p):=\frac{\pi_{v_j}^\star(p)}{\pi_v^\star(p)}
\label{eq:true-branch-conditional}
\end{equation}
If $\pi_v^\star(p)=0$, the conditioning event $\{Y\in\gY_v\}$ has conditional probability zero, so we do not define $q_{v,j}^\star(p)$. Setting all components to zero would give a vector whose components sum to zero, not a probability vector, so it would also be inappropriate to call that vector the true conditional branch probabilities. Even in this case, $\pi_{v_j}^\star(p)=0$ for every child $v_j$. Subsequent proofs use $\pi_{v_j}^\star(p)-\pi_v^\star(p)q_{v,j}(p)$, which remains defined without conditioning on a null event.

With the subtree indicators introduced above, define the local branch residual and local utility contrast componentwise by
\begin{align}
R_{v,j}(p,y)&:=J_{v,j}(y)-J_v(y)q_{v,j}(p),\label{eq:local-residual-component}\\
\Delta_{u,v,j}(p)&:=\mu_{u,v_j}(p)-\mu_{u,v}(p)\label{eq:local-contrast-component}
\end{align}
and write
$\boldsymbol R_v(p,y):=(R_{v,1}(p,y),\ldots,R_{v,d_v}(p,y))^\top$ and
$\boldsymbol\Delta_{u,v}(p):=(\Delta_{u,v,1}(p),\ldots,\Delta_{u,v,d_v}(p))^\top$, respectively.

Let $\gI:=\{I\subseteq[-1,1]:I\text{ is an interval}\}\cup\{\varnothing\}$. Let $\gC$ be a family of measurable functions $c:\gX\to[-1,1]$, and let $\gU$ be a family of measurable utilities $u:\Delta^{K-1}\times\gY\to[-1,1]$. For a fixed $(c,u,I,v)$, define the node-wise audit moment by
\begin{equation}
\Gamma_{c,u,I,v}(f):=\mathbb{E}\!\left[c(X)\mathbf{1}\{\mu_u(P)\in I\}\left\langle \boldsymbol\Delta_{u,v}(P),\boldsymbol R_v(P,Y)\right\rangle\right]
\label{eq:gamma-definition}
\end{equation}
Also set
\begin{equation}
\mathcal V_{\gT}(u):=\{v\in V^\circ:\boldsymbol\Delta_{u,v}\not\equiv\boldsymbol 0\}
\label{eq:relevant-node-set}
\end{equation}
Here $\boldsymbol\Delta_{u,v}\not\equiv\boldsymbol 0$ means that, for at least one prediction vector $p$ with positive components, the subtree mean utilities $\mu_{u,v_1}(p),\ldots,\mu_{u,v_{d_v}}(p)$ are not all equal. Define
\begin{equation}
\HUC(f;\gC,\gU):=\sup_{\substack{c\in\gC,\,u\in\gU,\,I\in\gI\\v\in\mathcal V_{\gT}(u)}}|\Gamma_{c,u,I,v}(f)|
\label{eq:huc-definition}
\end{equation}
If the candidate set is empty, define the supremum to be zero. For comparison, set
\begin{equation}
\UC(f;\gC,\gU):=\sup_{\substack{c\in\gC,\,u\in\gU,\,I\in\gI}}\left|\mathbb{E}\!\left[c(X)\mathbf{1}\{\mu_u(P)\in I\}\{u(P,Y)-\mu_u(P)\}\right]\right|
\label{eq:uc-definition}
\end{equation}
as the corresponding UC error. Whenever $\gC$ and $\gU$ are finite, write the number of audited triples $(c,u,v)$ as
\begin{equation}
N_{\mathrm{cand}}
:=|\gC|\sum_{u\in\gU}|\mathcal V_{\gT}(u)|.
\label{eq:audit-candidate-count}
\end{equation}

\section{Binary Classification and Binary Trees}
\label{app:binary-first}
Before proceeding to general multiclass classification, we first build intuition through a concrete example of binary classification with a binary tree.
\subsection{Binary Classification: A Single Branching Node}
\label{subsec:binary-two-class}
Let $\gY=\{L,R\}$, let $p_L=1-p_R$, and write $q_L(p):=p_L$ and $q_R(p):=p_R$. For utility values $u_L:=u(p,L)$ and $u_R:=u(p,R)$, the predicted mean utility is
\begin{equation}
\mu_u(p)=q_L(p)u_L+q_R(p)u_R
\label{eq:binary-two-class-mean}
\end{equation}
Define the local utility contrast and binary residual by
\begin{equation}
\Delta_u(p):=u_R-u_L,\qquad R(p,y):=\mathbf{1}\{y=R\}-q_R(p)
\label{eq:binary-two-class-residual}
\end{equation}
respectively.
\begin{lemma}[Pointwise identity for binary classification]
\label{lem:binary-two-class}
For every $y\in\{L,R\}$,
\begin{equation}
u(p,y)-\mu_u(p)=\Delta_u(p)R(p,y).
\label{eq:binary-two-class-identity}
\end{equation}

\end{lemma}
\begin{proof}
If $y=R$, then $R(p,R)=1-q_R(p)=q_L(p)$. Hence
\[
\Delta_u(p)R(p,R)
=q_L(p)(u_R-u_L)
=u_R-\{q_L(p)u_L+q_R(p)u_R\}
=u(p,R)-\mu_u(p).
\]
The third equality uses $u_R=\{q_L(p)+q_R(p)\}u_R$.

If $y=L$, then $R(p,L)=-q_R(p)$. Hence
\[
\Delta_u(p)R(p,L)
=-q_R(p)(u_R-u_L)
=u_L-\{q_L(p)u_L+q_R(p)u_R\}
=u(p,L)-\mu_u(p).
\]
The third equality uses $u_L=\{q_L(p)+q_R(p)\}u_L$.
\end{proof}
Writing the true right-label probability as $q_R^\star(p):=\mathbb{P}(Y=R\mid P=p)$ gives
\begin{equation}
\mathbb{E}[R(p,Y)\mid P=p]
=\mathbb{P}(Y=R\mid P=p)-q_R(p)
=q_R^\star(p)-q_R(p).
\label{eq:binary-two-class-conditional-residual}
\end{equation}
Consequently,
\begin{equation}
\mathbb{E}[u(p,Y)-\mu_u(p)\mid P=p]=\{q_R^\star(p)-q_R(p)\}\Delta_u(p).
\label{eq:binary-two-class-conditional-error}
\end{equation}

\subsection{A Single Internal Node of a Binary Tree}
\label{subsec:binary-one-node}
Write the left and right children of an internal node $v$ in a binary tree as $v_L,v_R$. Expanding the definition of subtree mean utility gives
\begin{equation}
\mu_{u,v}(p) =\frac{\pi_{v_L}(p)\mu_{u,v_L}(p)+\pi_{v_R}(p)\mu_{u,v_R}(p)}{\pi_v(p)} =q_{v,L}(p)\mu_{u,v_L}(p)+q_{v,R}(p)\mu_{u,v_R}(p).
\label{eq:binary-parent-mean}
\end{equation}
Define the local utility contrast and binary branch residual by
\begin{equation}
\Delta_{u,v}(p):=\mu_{u,v_R}(p)-\mu_{u,v_L}(p),\qquad R_v(p,y):=J_{v,R}(y)-q_{v,R}(p)J_v(y)
\label{eq:binary-node-residual}
\end{equation}
respectively.
\begin{lemma}[Local identity at a binary branching node]
\label{lem:binary-node}
For every leaf $y$,
\begin{equation}
\Delta_{u,v}(p)R_v(p,y)=\begin{cases}\mu_{u,v_R}(p)-\mu_{u,v}(p),&y\in\gY_{v_R},\\\mu_{u,v_L}(p)-\mu_{u,v}(p),&y\in\gY_{v_L},\\0,&y\notin\gY_v\end{cases}.
\label{eq:binary-node-identity}
\end{equation}

\end{lemma}
\begin{proof}
If $y\in\gY_{v_R}$, then $J_v(y)=J_{v,R}(y)=1$, so $R_v(p,y)=1-q_{v,R}(p)=q_{v,L}(p)$. Hence
\[
\Delta_{u,v}(p)R_v(p,y)
=q_{v,L}(p)\{\mu_{u,v_R}(p)-\mu_{u,v_L}(p)\}
=\mu_{u,v_R}(p)-\mu_{u,v}(p).
\]
The last equality uses~Eq.~\eqref{eq:binary-parent-mean}.

If $y\in\gY_{v_L}$, then $J_v(y)=1$ and $J_{v,R}(y)=0$, so $R_v(p,y)=-q_{v,R}(p)$. Hence
\[
\Delta_{u,v}(p)R_v(p,y)
=-q_{v,R}(p)\{\mu_{u,v_R}(p)-\mu_{u,v_L}(p)\}
=\mu_{u,v_L}(p)-\mu_{u,v}(p).
\]
If $y\notin\gY_v$, then $J_v(y)=J_{v,R}(y)=0$ and therefore $R_v(p,y)=0$.
\end{proof}
\subsection{Path Decomposition and Conditional Means on a Binary Tree}
\label{subsec:binary-tree-decomposition}\label{subsec:binary-special-case}
\begin{proposition}[Utility-error decomposition on a binary tree]
\label{prop:binary-tree-decomposition}
Suppose every internal node has two children. For every $p$ and leaf $y$,
\begin{equation}
u(p,y)-\mu_u(p)=\sum_{v\in V^\circ}\Delta_{u,v}(p)R_v(p,y).
\label{eq:binary-tree-decomposition}
\end{equation}

\end{proposition}
\begin{proof}
Write the path to leaf $y$ as $\rho=w_0,\ldots,w_{L_y}=y$. For every $r$, $w_{r+1}$ is either the left or the right child of $w_r$. By Lemma~\ref{lem:binary-node},
$\Delta_{u,w_r}(p)R_{w_r}(p,y)=\mu_{u,w_{r+1}}(p)-\mu_{u,w_r}(p)$.
Summing along the path gives
\begin{align*}
\sum_{r=0}^{L_y-1}\Delta_{u,w_r}(p)R_{w_r}(p,y)
&=\sum_{r=0}^{L_y-1}\{\mu_{u,w_{r+1}}(p)-\mu_{u,w_r}(p)\}\\[-1mm]
&=\mu_{u,y}(p)-\mu_{u,\rho}(p)=u(p,y)-\mu_u(p).
\end{align*}
Each intermediate subtree mean utility occurs once with a positive sign and once with a negative sign and hence cancels. For a node $v$ outside the path, $y\notin\gY_v$, so $R_v(p,y)=0$. The sum along the path can therefore be extended to all internal nodes.
\end{proof}
For an internal node $v$ of a binary tree, the true conditional reach probability of its right subtree is $\pi_{v_R}^\star(p)=\mathbb{P}(Y\in\gY_{v_R}\mid P=p)$. Thus
\begin{equation}
\begin{aligned}
\mathbb{E}[R_v(p,Y)\mid P=p]
&=\mathbb{P}(Y\in\gY_{v_R}\mid P=p)-q_{v,R}(p)\mathbb{P}(Y\in\gY_v\mid P=p)\\
&=\pi_{v_R}^\star(p)-q_{v,R}(p)\pi_v^\star(p).
\end{aligned}
\label{eq:binary-conditional-node-residual}
\end{equation}
Since $\Delta_{u,v}(p)$ is constant conditional on $P=p$,
\begin{equation}
\mathbb{E}[\Delta_{u,v}(p)R_v(p,Y)\mid P=p]
=\{\pi_{v_R}^\star(p)-q_{v,R}(p)\pi_v^\star(p)\}\Delta_{u,v}(p).
\label{eq:binary-mean-node-contribution}
\end{equation}
When $\pi_v^\star(p)>0$, we can define $q_v^\star(p):=\pi_{v_R}^\star(p)/\pi_v^\star(p)$ and rewrite the right-hand side as
\[
\pi_v^\star(p)\{q_v^\star(p)-q_{v,R}(p)\}\Delta_{u,v}(p).
\]
When $\pi_v^\star(p)=0$, we also have $\pi_{v_R}^\star(p)=0$. The right-hand side of~Eq.~\eqref{eq:binary-mean-node-contribution} is then already zero, without defining $q_v^\star(p)$.
\subsection{Four-Leaf Mean-Utility Decomposition in Section~\ref{sec:single-utility-cancellation}}
\label{subsec:four-leaf-mean-matching}

This subsection verifies the child-utility notation and proves Eq.~\eqref{eq:d-full-sec32} directly, without using the general-tree decomposition. Fix the four-leaf tree, a positive prediction $p$, and the target utility $u$ of Section~\ref{sec:single-utility-cancellation}. Use the conditional distribution $\eta_y(p)$ from Appendix~\ref{subsec:probability-tree}; conditional identities hold for almost every $p$ under the law of $P$. Assume $0<q_{0,R}^\star<1$ as in that section.

\paragraph{Child means and their predicted averages.}
For $t\in\{0,L,R\}$ and $j\in\{L,R\}$, the child mean is
\begin{equation}
\bar u_{t,j}
:=\sum_{y\in\gY_{v_{t,j}}}p_yu(p,y)\,/\!\sum_{y\in\gY_{v_{t,j}}}p_y
=\mu_{u,v_{t,j}}(p).
\label{eq:four-leaf-child-means}
\end{equation}
The denominator is positive, and the normalized weights are nonnegative and sum to one. Thus $\bar u_{t,j}\in[-1,1]$ and, for $\widetilde Y\sim p$,
$\bar u_{t,j}=\mathbb E_{\widetilde Y\sim p} [u(p,\widetilde Y)\mid\widetilde Y\in\gY_{v_{t,j}}]$.
These means are measurable functions of $p$, since they are ratios of finite sums of measurable functions with positive denominators. As in Section~\ref{sec:single-utility-cancellation}, dependence on the fixed $u,p$ is suppressed. The lower children are leaves: $(v_{L,L},v_{L,R},v_{R,L},v_{R,R})=(y_1,y_2,y_3,y_4)$. Hence $\bar u_{t,j}=u(p,v_{t,j})$ for $t\in\{L,R\}$.

For $t,j\in\{L,R\}$, at the root $q_{0,t}=\sum_{y\in\gY_{v_t}}p_y$, and at a lower node $q_{t,j}=p_{v_{t,j}}/q_{0,t}$. Therefore
\begin{equation}
\begin{aligned}
\sum_{j=L,R}q_{t,j}\bar u_{t,j}
&=\sum_{j=L,R}(p_{v_{t,j}}/q_{0,t})u(p,v_{t,j})
=\bar u_{0,t},\\
\sum_{t=L,R}q_{0,t}\bar u_{0,t}
&=\sum_{t=L,R}\sum_{y\in\gY_{v_t}}p_yu(p,y)=\mu_u(p).
\end{aligned}
\label{eq:four-leaf-mean-recursions}
\end{equation}
The first identity holds for $t\in\{L,R\}$. It expresses the predicted mean at a parent as the predicted-branch average of its child means.

\begin{proposition}[Four-leaf mean-utility decomposition]
\label{prop:four-leaf-mean-matching}
With the notation of Section~\ref{sec:single-utility-cancellation},
$D^{\mathrm{full}}(p)=D_{\mathrm{top}}(p)+D_L(p)+D_R(p)$ as in Eq.~\eqref{eq:d-full-sec32}. Each term is the conditional mean of the utility increment at its corresponding node.
\end{proposition}
\begin{proof}
For $t,j\in\{L,R\}$, the multiplication rule for conditional probabilities gives
$\mathbb P(Y=v_{t,j}\mid P=p)=q_{0,t}^\star q_{t,j}^\star$.
Because these children are leaves, the true conditional mean is
\begin{equation}
\mathbb E[u(p,Y)\mid P=p]
=\sum_{t=L,R}q_{0,t}^\star\sum_{j=L,R}q_{t,j}^\star\bar u_{t,j}.
\label{eq:four-leaf-true-mean}
\end{equation}
Subtract the predicted root mean in Eq.~\eqref{eq:four-leaf-mean-recursions} and add and subtract $\sum_{t=L,R}q_{0,t}^\star\bar u_{0,t}$:
\begin{align*}
D^{\mathrm{full}}(p)
&=\sum_{t=L,R}q_{0,t}^\star\sum_{j=L,R}q_{t,j}^\star\bar u_{t,j}
-\sum_{t=L,R}q_{0,t}\bar u_{0,t}\\
&=[\sum_{t=L,R}q_{0,t}^\star\bar u_{0,t}
-\sum_{t=L,R}q_{0,t}\bar u_{0,t}]
+\sum_{t=L,R}q_{0,t}^\star
[\sum_{j=L,R}q_{t,j}^\star\bar u_{t,j}-\bar u_{0,t}]\\
&=D_{\mathrm{top}}(p)
+\sum_{t=L,R}q_{0,t}^\star
[\sum_{j=L,R}q_{t,j}^\star\bar u_{t,j}
-\sum_{j=L,R}q_{t,j}\bar u_{t,j}]\\
&=D_{\mathrm{top}}(p)+D_L(p)+D_R(p).
\end{align*}
The third equality uses the first identity in Eq.~\eqref{eq:four-leaf-mean-recursions}.

For the increment interpretation, consider an observed leaf $y=v_{t,j}$. Along its two-step path,
\[
u(p,y)-\mu_u(p)
=\{\bar u_{0,t}-\mu_u(p)\}
+\{\bar u_{t,j}-\bar u_{0,t}\}.
\]
Equivalently, using the subtree indicators of Appendix~\ref{app:binary-first},
\begin{equation}
u(p,y)-\mu_u(p) ={}\sum_{t=L,R}J_{v_t}(y)\{\bar u_{0,t}-\mu_u(p)\} +\sum_{t=L,R}\sum_{j=L,R}J_{v_{t,j}}(y) \{\bar u_{t,j}-\bar u_{0,t}\}.
\label{eq:four-leaf-pointwise-increments}
\end{equation}
The conditional mean of the root term is
\[
\sum_{t=L,R}q_{0,t}^\star\{\bar u_{0,t}-\mu_u(p)\}
=\sum_{t=L,R}q_{0,t}^\star\bar u_{0,t}
-\sum_{t=L,R}q_{0,t}\bar u_{0,t}
=D_{\mathrm{top}}(p).
\]
For a lower node $v_t$, its conditional mean is
\[
\sum_{j=L,R}q_{0,t}^\star q_{t,j}^\star
\{\bar u_{t,j}-\bar u_{0,t}\}
=q_{0,t}^\star[\sum_{j=L,R}q_{t,j}^\star\bar u_{t,j}
-\sum_{j=L,R}q_{t,j}\bar u_{t,j}]
=D_t(p).
\]
Here $\sum_jq_{t,j}^\star=1$ and Eq.~\eqref{eq:four-leaf-mean-recursions} give the equality. Thus the lower-node contributions retain their true reach probabilities.
\end{proof}

\paragraph{Correspondence with the binary-node calculation.}
Since $q_{t,L}^\star-q_{t,L}=-(q_{t,R}^\star-q_{t,R})$, the same terms can also be written as
\begin{equation}
\begin{aligned}
D_{\mathrm{top}}(p)&=(q_{0,R}^\star-q_{0,R})(\bar u_{0,R}-\bar u_{0,L}),\\
D_t(p)&=q_{0,t}^\star(q_{t,R}^\star-q_{t,R})(\bar u_{t,R}-\bar u_{t,L}),
\qquad t\in\{L,R\}.
\end{aligned}
\label{eq:four-leaf-binary-factorization}
\end{equation}
These are precisely the conditional means in Appendix~\ref{subsec:binary-tree-decomposition}: the binary contrast is the right-child mean minus the left-child mean. By Lemma~\ref{lem:one-step-increment}, the corresponding summands in Eq.~\eqref{eq:four-leaf-pointwise-increments} also equal the vector inner products used in Section~\ref{sec:huc-definition}. This verifies the correspondence between the three quantities in Eq.~\eqref{eq:d-full-sec32} and the general-node contributions.

\paragraph{The UC identity in Section~\ref{sec:single-utility-cancellation}.}
For a predictor satisfying the four-leaf assumptions almost surely, write $D^{\mathrm{full}}(P)=\mathbb E[u(P,Y)-\mu_u(P)\mid P]$ and restore the dependence of the $D_t$ on $P$. For any $I\in\gI$, the interval indicator is measurable with respect to $P$, so
\begin{equation}
\begin{aligned}
&\mathbb E[\mathbf1\{\mu_u(P)\in I\}\{u(P,Y)-\mu_u(P)\}] =\mathbb E[\mathbf1\{\mu_u(P)\in I\}D^{\mathrm{full}}(P)]\\
&=\sum_{t\in\{\mathrm{top},L,R\}}
\mathbb E[\mathbf1\{\mu_u(P)\in I\}D_t(P)].
\end{aligned}
\label{eq:four-leaf-uc-moment}
\end{equation}
Taking the absolute value and then the supremum over $I$ gives the displayed UC identity in Section~\ref{sec:single-utility-cancellation}. Thus the sum is taken before the absolute value, allowing opposite node contributions to cancel within the same predicted-utility interval.

\section{Extension to General Trees}
\label{app:general-tree}

For binary trees, each internal node was associated with one scalar residual and one scalar utility contrast. At a general internal node $v$ with $d_v$ children, both the local branch residual and the local utility contrast are $d_v$-dimensional vectors. We use the predicted reach probability $\pi_a(p)$, predicted conditional branch probability $q_{v,j}(p)$, and path-product representation already defined and derived in~Eq.~\eqref{eq:subtree-probability}, Eq.~\eqref{eq:branch-conditional}, and~Eq.~\eqref{eq:leaf-factorization}. No additional auxiliary stochastic process is introduced; the argument uses only these quantities determined by the leaf distribution $p$.
\subsection{Local Algebra on a General Tree}
\label{subsec:general-local-algebra}
\begin{lemma}[Recursions for subtree probabilities and subtree mean utilities]
\label{lem:subtree-recursions}
For every internal node $v$,
\begin{align}
\pi_v(p)&=\sum_{j=1}^{d_v}\pi_{v_j}(p),\label{eq:subtree-probability-recursion}\\
N_{u,v}(p)&=\sum_{j=1}^{d_v}N_{u,v_j}(p),\label{eq:subtree-utility-sum-recursion}\\
\mu_{u,v}(p)&=\sum_{j=1}^{d_v}q_{v,j}(p)\mu_{u,v_j}(p).\label{eq:subtree-mean-recursion}
\end{align}
Moreover,
\begin{equation}
\left\langle \boldsymbol q_v(p),\boldsymbol\Delta_{u,v}(p)\right\rangle=0.
\label{eq:q-delta-orthogonality}
\end{equation}
\end{lemma}
\begin{proof}
By~Eq.~\eqref{eq:child-partition},
\begin{align*}
\pi_v(p)&=\sum_{y\in\gY_v}p_y=\sum_{j=1}^{d_v}\sum_{y\in\gY_{v_j}}p_y=\sum_{j=1}^{d_v}\pi_{v_j}(p),\\
N_{u,v}(p)&=\sum_{y\in\gY_v}p_yu(p,y)=\sum_{j=1}^{d_v}\sum_{y\in\gY_{v_j}}p_yu(p,y)=\sum_{j=1}^{d_v}N_{u,v_j}(p).
\end{align*}
Since $N_{u,v_j}(p)=\pi_{v_j}(p)\mu_{u,v_j}(p)$,
\[
\mu_{u,v}(p)
=\frac{\sum_{j=1}^{d_v}\pi_{v_j}(p)\mu_{u,v_j}(p)}{\pi_v(p)}
=\sum_{j=1}^{d_v}q_{v,j}(p)\mu_{u,v_j}(p).
\]
Finally,
\[
\left\langle \boldsymbol q_v(p),\boldsymbol\Delta_{u,v}(p)\right\rangle
=\sum_{j=1}^{d_v}q_{v,j}(p)\mu_{u,v_j}(p)-\mu_{u,v}(p)\sum_{j=1}^{d_v}q_{v,j}(p)
=0.
\]
\end{proof}
\begin{lemma}[Local contribution of one node]
\label{lem:one-step-increment}
If $y\in\gY_{v_k}$, then
\begin{equation}
\left\langle \boldsymbol\Delta_{u,v}(p),\boldsymbol R_v(p,y)\right\rangle=\mu_{u,v_k}(p)-\mu_{u,v}(p).
\label{eq:one-step-increment}
\end{equation}
If $y\notin\gY_v$, the left-hand side is zero. Moreover, for every $p,y$,
\begin{equation}
\sum_{j=1}^{d_v}R_{v,j}(p,y)=0.
\label{eq:residual-sum-zero}
\end{equation}
\end{lemma}
\begin{proof}
Suppose $y\in\gY_{v_k}$. Then $J_v(y)=1$, $J_{v,k}(y)=1$, and $J_{v,j}(y)=0$ for $j\ne k$. Thus
\begin{equation}
R_{v,j}(p,y)=\mathbf{1}\{j=k\}-q_{v,j}(p).
\label{eq:residual-on-path}
\end{equation}
Consequently,
\[
\left\langle \boldsymbol\Delta_{u,v}(p),\boldsymbol R_v(p,y)\right\rangle
=\Delta_{u,v,k}(p)-\sum_{j=1}^{d_v}q_{v,j}(p)\Delta_{u,v,j}(p)
=\mu_{u,v_k}(p)-\mu_{u,v}(p).
\]
The last equality uses~Eq.~\eqref{eq:q-delta-orthogonality}.

If $y\notin\gY_v$, then $J_v(y)=J_{v,j}(y)=0$ for every $j$, so $\boldsymbol R_v(p,y)=\boldsymbol 0$. Furthermore,
\[
\sum_{j=1}^{d_v}R_{v,j}(p,y)
=\sum_{j=1}^{d_v}J_{v,j}(y)-J_v(y)\sum_{j=1}^{d_v}q_{v,j}(p)
=0.
\]
\end{proof}
\subsection{Utility-Error Decomposition and Conditional Means on a General Tree}
\label{subsec:general-tree-decomposition}\label{subsec:tree-decomposition-proof}
\begin{theorem}[Utility-error decomposition on a general tree]
\label{thm:tree-decomposition}
For every positive prediction vector $p$ and every leaf $y$,
\begin{equation}
u(p,y)-\mu_u(p)=\sum_{v\in V^\circ}\left\langle \boldsymbol\Delta_{u,v}(p),\boldsymbol R_v(p,y)\right\rangle.
\label{eq:tree-decomposition}
\end{equation}
\end{theorem}
\begin{proof}
For a leaf $y$, use the path sequence defined in~Eq.~\eqref{eq:root-leaf-path}, abbreviating $w_r:=w_r(y)$ in this proof. By Lemma~\ref{lem:one-step-increment}, each internal node $w_r$ along the path satisfies
$\left\langle \boldsymbol\Delta_{u,w_r}(p),\boldsymbol R_{w_r}(p,y)\right\rangle=\mu_{u,w_{r+1}}(p)-\mu_{u,w_r}(p)$.
Therefore,
\begin{align*}
\sum_{r=0}^{L_y-1}\left\langle \boldsymbol\Delta_{u,w_r}(p),\boldsymbol R_{w_r}(p,y)\right\rangle
&=\sum_{r=0}^{L_y-1}\{\mu_{u,w_{r+1}}(p)-\mu_{u,w_r}(p)\}\\[-1mm]
&=\mu_{u,y}(p)-\mu_{u,\rho}(p)=u(p,y)-\mu_u(p).
\end{align*}
At an internal node $v$ outside the path, $y\notin\gY_v$, so its local contribution is zero by~Lemma~\ref{lem:one-step-increment}. Hence the sum along the path can be extended to all of $V^\circ$.
\end{proof}
\begin{lemma}[Conditional mean of a branch residual and mean-utility matching]
\label{lem:conditional-residual}
For every internal node $v$ and child $v_j$,
\begin{equation}
\mathbb{E}[R_{v,j}(p,Y)\mid P=p]=\pi_{v_j}^\star(p)-\pi_v^\star(p)q_{v,j}(p).
\label{eq:conditional-residual-component}
\end{equation}
Consequently, the conditional mean of the local contribution of $v$ is
\begin{equation}
\begin{aligned}
&\mathbb E[\langle\boldsymbol\Delta_{u,v}(p),\boldsymbol R_v(p,Y)\rangle\mid P=p] =\sum_{j=1}^{d_v}\Delta_{u,v,j}(p)
\{\pi_{v_j}^\star(p)-\pi_v^\star(p)q_{v,j}(p)\}\\
&=\sum_{j=1}^{d_v}\pi_{v_j}^\star(p)\mu_{u,v_j}(p)
-\pi_v^\star(p)\mu_{u,v}(p).
\end{aligned}
\label{eq:mean-node-contribution}
\end{equation}
If $\pi_v^\star(p)>0$, this is equivalently
\begin{equation}
\mathbb E[\langle\boldsymbol\Delta_{u,v}(p),\boldsymbol R_v(p,Y)\rangle\mid P=p] =\pi_v^\star(p)[\sum_{j=1}^{d_v}q_{v,j}^\star(p)\mu_{u,v_j}(p)-\sum_{j=1}^{d_v}q_{v,j}(p)\mu_{u,v_j}(p)].
\label{eq:node-mean-matching}
\end{equation}
If $\pi_v^\star(p)=0$, the quantities in Eq.~\eqref{eq:mean-node-contribution} are zero without defining $q_{v,j}^\star(p)$.
\end{lemma}
\begin{proof}
Taking conditional expectations in Eq.~\eqref{eq:local-residual-component} gives
\begin{align*}
\mathbb E[R_{v,j}(p,Y)\mid P=p]
&=\mathbb E[J_{v,j}(Y)\mid P=p]
-q_{v,j}(p)\mathbb E[J_v(Y)\mid P=p]\\
&=\pi_{v_j}^\star(p)-\pi_v^\star(p)q_{v,j}(p).
\end{align*}
This proves Eq.~\eqref{eq:conditional-residual-component}. The utility contrast is fixed conditional on $P=p$; expanding the inner product and interchanging the finite sum with conditional expectation gives the first equality in Eq.~\eqref{eq:mean-node-contribution}.

The true child subtrees partition the parent subtree, so
$\sum_j\pi_{v_j}^\star(p)=\pi_v^\star(p)$, and $\sum_jq_{v,j}(p)=1$. Substituting
$\Delta_{u,v,j}(p)=\mu_{u,v_j}(p)-\mu_{u,v}(p)$ and using these identities, the parent-mean terms cancel:
\begin{align*}
\sum_{j=1}^{d_v}\Delta_{u,v,j}(p)
 \{\pi_{v_j}^\star(p)-\pi_v^\star(p)q_{v,j}(p)\}
&=\sum_{j=1}^{d_v}\mu_{u,v_j}(p)
 \{\pi_{v_j}^\star(p)-\pi_v^\star(p)q_{v,j}(p)\}\\
&=\sum_{j=1}^{d_v}\pi_{v_j}^\star(p)\mu_{u,v_j}(p)
 -\pi_v^\star(p)\sum_{j=1}^{d_v}q_{v,j}(p)\mu_{u,v_j}(p)\\
&=\sum_{j=1}^{d_v}\pi_{v_j}^\star(p)\mu_{u,v_j}(p)
 -\pi_v^\star(p)\mu_{u,v}(p).
\end{align*}
The last line uses Eq.~\eqref{eq:subtree-mean-recursion}. This proves the second equality in Eq.~\eqref{eq:mean-node-contribution} for every true reach probability.

If $\pi_v^\star(p)>0$, substitute $\pi_{v_j}^\star(p)=\pi_v^\star(p)q_{v,j}^\star(p)$ in the penultimate line and factor out $\pi_v^\star(p)$. This gives Eq.~\eqref{eq:node-mean-matching} and Eq.~\eqref{eq:mean-node-contribution-main}. If $\pi_v^\star(p)=0$, then $0\le\pi_{v_j}^\star(p)\le\pi_v^\star(p)=0$ for every child, so all expressions in Eq.~\eqref{eq:mean-node-contribution} vanish without conditioning on a null event.
\end{proof}

\paragraph{Meaning of the two averages.}
For $\widetilde Y\sim p$, each value $\mu_{u,v_j}(p)$ is the conditional average of $u(p,\widetilde Y)$ given $\widetilde Y\in\gY_{v_j}$. Thus the finer outcomes within a child are marginalized under the predictive distribution. When $\pi_v^\star(p)>0$, the actual child index $j_v(Y)$, conditional on $P=p$ and $Y\in\gY_v$, has probabilities $q_{v,j}^\star(p)$. Consequently,
\begin{equation}
\mathbb E[\mu_{u,v_{j_v(Y)}}(p)-\mu_{u,v}(p) \mid P=p,\ Y\in\gY_v] =\sum_{j=1}^{d_v}q_{v,j}^\star(p)\mu_{u,v_j}(p) -\sum_{j=1}^{d_v}q_{v,j}(p)\mu_{u,v_j}(p).
\label{eq:conditional-parent-child-mean-difference}
\end{equation}
Only the outer branch probabilities differ between the two sums; both use the same predictively marginalized child values. Equation~\eqref{eq:node-mean-matching} multiplies this difference by the actual probability of reaching the parent. Under the predictive distribution instead, the child average equals $\mu_{u,v}(p)$ by Eq.~\eqref{eq:subtree-mean-recursion}.

\begin{corollary}[Node decomposition of the cell-wise UC residual]
\label{cor:uc-cell-decomposition}
For every $c,u,I$,
\begin{equation}
\mathbb{E}\!\left[c(X)\mathbf{1}\{\mu_u(P)\in I\}\{u(P,Y)-\mu_u(P)\}\right]=\sum_{v\in V^\circ}\Gamma_{c,u,I,v}(f).
\label{eq:uc-cell-node-sum}
\end{equation}
\end{corollary}
\begin{proof}
Substitute $p=P$ and $y=Y$ into Theorem~\ref{thm:tree-decomposition}, multiply both sides by $c(X)\mathbf{1}\{\mu_u(P)\in I\}$, and take expectations. Since the tree is finite, the expectation and the finite sum can be interchanged:
\begin{align*}
&\mathbb{E}\!\left[c(X)\mathbf{1}\{\mu_u(P)\in I\}\{u(P,Y)-\mu_u(P)\}\right]\\
&=\sum_{v\in V^\circ}\mathbb{E}\!\left[c(X)\mathbf{1}\{\mu_u(P)\in I\}\left\langle \boldsymbol\Delta_{u,v}(P),\boldsymbol R_v(P,Y)\right\rangle\right]\\
&=\sum_{v\in V^\circ}\Gamma_{c,u,I,v}(f).\qedhere
\end{align*}
\end{proof}

\paragraph{Mean consistency under common audit conditions.}
Lemma~\ref{lem:one-step-increment} and $\sum_jJ_{v,j}(y)=J_v(y)$ give the pointwise identity
\[
\langle\boldsymbol\Delta_{u,v}(p),\boldsymbol R_v(p,y)\rangle
=\sum_{j=1}^{d_v}J_{v,j}(y)\mu_{u,v_j}(p)-J_v(y)\mu_{u,v}(p).
\]
Multiplying by the common audit weight and taking expectation under the actual joint law yields
\begin{equation}
\Gamma_{c,u,I,v}(f)
=\mathbb E\!\left[c(X)\mathbf1\{\mu_u(P)\in I\}
\{\sum_{j=1}^{d_v}J_{v,j}(Y)\mu_{u,v_j}(P)
-J_v(Y)\mu_{u,v}(P)\}\right].
\label{eq:node-audit-mean-consistency}
\end{equation}
The expression in braces is the child's utility assessment minus the parent's on reaching $v$, and zero otherwise. Hence $|\Gamma_{c,u,I,v}(f)|\le\HUC(f;\gC,\gU)$ for every $c\in\gC$, $u\in\gU$, $I\in\gI$, and relevant $v$; irrelevant nodes have zero contribution. This is the reach-weighted, audit-moment sense in which HUC evaluates mean consistency under hierarchical refinement uniformly across nodes. The identity retains $c(X)$ inside the expectation, so it applies to the full subgroup class in the definition of HUC.

\subsection{Orthogonality and Uniqueness of the Node Decomposition}
\label{subsec:orthogonal-node-decomposition}

The telescoping argument proves that the node contributions sum to the utility residual. We now show that, for a fixed tree and predictive distribution, the same decomposition is an orthogonal expansion and is therefore unique and nonredundant.

Fix a prediction vector $p$ with strictly positive coordinates. For functions $h,g:\gY\to\R$, define
\begin{equation}
\langle h,g\rangle_p:=\sum_{y\in\gY}p_yh(y)g(y),
\qquad
\|h\|_p^2:=\langle h,h\rangle_p.
\label{eq:weighted-leaf-inner-product}
\end{equation}
For each internal node $v$, choose vectors
$a_{v,1},\ldots,a_{v,d_v-1}\in\R^{d_v}$ satisfying
\begin{equation}
\sum_{j=1}^{d_v}q_{v,j}(p)a_{v,\ell,j}=0,
\qquad
\sum_{j=1}^{d_v}q_{v,j}(p)a_{v,\ell,j}a_{v,m,j}
=\mathbf{1}\{\ell=m\}.
\label{eq:local-detail-basis}
\end{equation}
Thus these vectors form an orthonormal basis, under the $\boldsymbol q_v(p)$-weighted inner product, for the $(d_v-1)$-dimensional centered child space at $v$. Lift each vector to the leaves by
\begin{equation}
\phi_{v,\ell,p}(y)
:=
\begin{cases}
\frac{a_{v,\ell,j}}{\sqrt{\pi_v(p)}},
& y\in\gY_{v_j},\\
0,&y\notin\gY_v,
\end{cases}
\qquad \ell=1,\ldots,d_v-1.
\label{eq:multiway-tree-haar-function}
\end{equation}

\begin{proposition}[Probability-weighted orthogonal representation]
\label{prop:orthogonal-utility-decomposition}
The collection
$\{\phi_{v,\ell,p}:v\in V^\circ,\ \ell=1,\ldots,d_v-1\}$
is an orthonormal basis, under the inner product in Eq.~\eqref{eq:weighted-leaf-inner-product}, for the leaf functions having zero $p$-mean. Its size is
\begin{equation}
\sum_{v\in V^\circ}(d_v-1)=K-1.
\label{eq:tree-haar-basis-count}
\end{equation}
For a target utility $u$, define
\begin{equation}
\theta_{u,v,\ell}(p)
:=
\sqrt{\pi_v(p)}
\sum_{j=1}^{d_v}q_{v,j}(p)
\Delta_{u,v,j}(p)a_{v,\ell,j}.
\label{eq:tree-haar-utility-coefficient}
\end{equation}
Then the contribution of node $v$ and the complete centered utility satisfy
\begin{align}
\left\langle \boldsymbol\Delta_{u,v}(p),\boldsymbol R_v(p,y)\right\rangle
&=
\sum_{\ell=1}^{d_v-1}
\theta_{u,v,\ell}(p)\phi_{v,\ell,p}(y),
\label{eq:node-contribution-haar-expansion}\\
u(p,y)-\mu_u(p)
&=
\sum_{v\in V^\circ}\sum_{\ell=1}^{d_v-1}
\theta_{u,v,\ell}(p)\phi_{v,\ell,p}(y).
\label{eq:utility-tree-haar-expansion}
\end{align}
This expansion is unique, and the contribution assigned to each node is the unique least-squares component that is zero outside $\gY_v$, constant on each child subtree, and centered under the predictive distribution within $\gY_v$. In particular,
\begin{equation}
\operatorname{Var}_{\widetilde Y\sim p}(u(p,\widetilde Y)) = \sum_{v\in V^\circ}\sum_{\ell=1}^{d_v-1} \theta_{u,v,\ell}(p)^2 = \sum_{v\in V^\circ}\pi_v(p) \sum_{j=1}^{d_v}q_{v,j}(p) \{\mu_{u,v_j}(p)-\mu_{u,v}(p)\}^2.
\label{eq:tree-parseval-utility}
\end{equation}
\end{proposition}

\begin{proof}
Eq.~\eqref{eq:local-detail-basis} and the definition of $\phi_{v,\ell,p}$ imply that every $\phi_{v,\ell,p}$ has zero $p$-mean and unit norm, and that functions attached to the same node are mutually orthogonal. For distinct nodes, their descendant-leaf sets are either disjoint or one is contained in a child subtree of the other. In the first case their supports are disjoint. In the second case the function attached to the ancestor is constant on the descendant subtree, whereas the function attached to the descendant has zero $p$-weighted sum there. Hence the functions are again orthogonal.

Their number is given by Eq.~\eqref{eq:tree-haar-basis-count}, because
$\sum_{v\in V^\circ}d_v=|V|-1=|V^\circ|+K-1$.
The zero-$p$-mean functions on $K$ leaves form a $(K-1)$-dimensional space, so this orthonormal collection is a basis.

By~Eq.~\eqref{eq:q-delta-orthogonality}, the child-value vector
$\boldsymbol\Delta_{u,v}(p)$ has zero $\boldsymbol q_v(p)$-weighted mean. Expansion in the basis in Eq.~\eqref{eq:local-detail-basis} and lifting via Eq.~\eqref{eq:multiway-tree-haar-function} give Eq.~\eqref{eq:node-contribution-haar-expansion}. Summing and using Theorem~\ref{thm:tree-decomposition} give Eq.~\eqref{eq:utility-tree-haar-expansion}.

Orthogonality makes this representation unique. More explicitly, let
$g_u^p(y):=u(p,y)-\mu_u(p)$ and
$r_{u,v}^p(y):=\left\langle \boldsymbol\Delta_{u,v}(p),\boldsymbol R_v(p,y)\right\rangle$.
For any other branch-local centered function $h$ associated with node $v$, all orthogonal components of $g_u^p-r_{u,v}^p$ belong to other nodes and are orthogonal to $h-r_{u,v}^p$. Therefore
\[
\|g_u^p-h\|_p^2
=
\|g_u^p-r_{u,v}^p\|_p^2
+
\|r_{u,v}^p-h\|_p^2,
\]
which proves the stated least-squares property and its uniqueness. Finally, Parseval's identity gives the first equality in~Eq.~\eqref{eq:tree-parseval-utility}; the second follows from
\[
\|r_{u,v}^p\|_p^2
=
\pi_v(p)\sum_{j=1}^{d_v}q_{v,j}(p)
\{\mu_{u,v_j}(p)-\mu_{u,v}(p)\}^2.
\]
\end{proof}

For a binary tree, $d_v-1=1$, so each internal node contributes one detail coordinate. With the scalar binary residual $R_v$ of Appendix~\ref{subsec:binary-one-node}, choosing the standard normalization gives
\begin{equation}
\phi_{v,p}(y)
=\frac{R_v(p,y)}{\sqrt{\pi_v(p)q_{v,L}(p)q_{v,R}(p)}}.
\label{eq:binary-tree-haar-coordinate}
\end{equation}
Thus a binary tree has one probability-weighted orthogonal coordinate per internal node, and~Eq.~\eqref{eq:utility-tree-haar-expansion} is the corresponding expansion of the centered utility.

\begin{remark}[Reference measure and scope of the orthogonality]
\label{rem:tree-haar-reference-measure}
The inner product in~Eq.~\eqref{eq:weighted-leaf-inner-product} uses the model distribution $\widetilde Y\sim p$.  Orthogonality, uniqueness, and the Parseval identity are therefore statements about how the model decomposes a centered leaf utility.  HUC moments, in contrast, average the same node coordinates under the actual joint law of $(X,Y)$.  Consequently, this orthogonality does not imply that population HUC moments vanish or cannot cancel after expectations are taken under the true distribution.  It gives a nonredundant node decomposition of this pointwise residual.
\end{remark}

This characterization is conditional on the chosen label tree and predictive reference distribution $p$. It establishes uniqueness and least-squares optimality of the node contributions for that tree and $p$; it does not assert that the tree itself, or the outer $\ell_\infty$ aggregation used by HUC, is optimal.

\section{Relationships among UC, HUC, and Canonical Calibration}
\label{app:uc-huc-full}

\subsection{UC--HUC Comparison and the Implication of Canonical Calibration}
\label{subsec:uc-huc-full}\label{subsec:full-huc-uc-proof}
\begin{proposition}[UC--HUC comparison]
\label{prop:uc-huc-full}
For every subgroup class $\gC$ and utility $u$,
\begin{equation}
\UC(f;\gC,\{u\})\le|\mathcal V_{\gT}(u)|\HUC(f;\gC,\{u\}).
\label{eq:uc-huc-bound}
\end{equation}
Moreover, if $f$ is canonically calibrated, then
\begin{equation}
\HUC(f;\{1\},\gU)=0.
\label{eq:full-implies-huc}
\end{equation}
\end{proposition}
\begin{proof}
Fix $c\in\gC$ and $I\in\gI$. By~Eq.~\eqref{eq:uc-cell-node-sum} and the triangle inequality,
\begin{align*}
&\left|\mathbb{E}\!\left[c(X)\mathbf{1}\{\mu_u(P)\in I\}\{u(P,Y)-\mu_u(P)\}\right]\right|\\
&=\left|\sum_{v\in V^\circ}\Gamma_{c,u,I,v}(f)\right|=\left|\sum_{v\in\mathcal V_{\gT}(u)}\Gamma_{c,u,I,v}(f)\right|
\le\sum_{v\in\mathcal V_{\gT}(u)}|\Gamma_{c,u,I,v}(f)|\le|\mathcal V_{\gT}(u)|\HUC(f;\gC,\{u\}).
\end{align*}
The restriction uses $\boldsymbol\Delta_{u,v}\equiv\boldsymbol 0$ for $v\notin\mathcal V_{\gT}(u)$; taking the supremum over $c,I$ proves Eq.~\eqref{eq:uc-huc-bound}.

Recall that $J_a(y)=\mathbf{1}\{y\in\gY_a\}$ indicates membership in the descendant-leaf set of node $a$, and $J_{v,j}(y)=\mathbf{1}\{y\in\gY_{v_j}\}$ for a child $v_j$.
Next assume canonical calibration: $\mathbb{P}(Y=y\mid P)=P_y$ for every $y$. At each node $a$,
\begin{align*}
\mathbb{E}[J_a(Y)\mid P]&=\mathbb{P}(Y\in\gY_a\mid P)=\sum_{y\in\gY_a}\mathbb{P}(Y=y\mid P)=\sum_{y\in\gY_a}P_y=\pi_a(P).
\end{align*}
Thus, for every internal node $v$ and child $v_j$,
\[
\begin{aligned}
&\mathbb{E}[R_{v,j}(P,Y)\mid P]=\mathbb{E}[J_{v,j}(Y)\mid P]-q_{v,j}(P)\mathbb{E}[J_v(Y)\mid P]\\
&=\pi_{v_j}(P)-q_{v,j}(P)\pi_v(P) =\pi_{v_j}(P)-\frac{\pi_{v_j}(P)}{\pi_v(P)}\pi_v(P)=0.
\end{aligned}
\]
Hence $\mathbb{E}[\boldsymbol R_v(P,Y)\mid P]=\boldsymbol 0$. The tower property gives
\begin{align*}
\Gamma_{1,u,I,v}(f)&=\mathbb{E}\!\left[\mathbf{1}\{\mu_u(P)\in I\}\left\langle \boldsymbol\Delta_{u,v}(P),\boldsymbol R_v(P,Y)\right\rangle\right]\\
&=\mathbb{E}\!\left[\mathbf{1}\{\mu_u(P)\in I\}\left\langle \boldsymbol\Delta_{u,v}(P),\mathbb{E}[\boldsymbol R_v(P,Y)\mid P]\right\rangle\right]=0.
\end{align*}
Since this holds for every $u,I,v$, it proves~Eq.~\eqref{eq:full-implies-huc}.
\end{proof}
\begin{example}[Exact cancellation with zero UC and nonzero HUC]
\label{ex:exact-uc-zero-huc-positive}
On the four-leaf tree of Section~\ref{sec:single-utility-cancellation}, consider the constant predictor $P\equiv p=(1/4,1/4,1/4,1/4)$. Let $(u(p,y_1),u(p,y_2),u(p,y_3),u(p,y_4))=(-1,1,1,1)$ and let the true conditional leaf distribution be $\eta(p)=(1/4,0,3/8,3/8)$. Extend the utility away from this prediction vector to any bounded measurable function. The predicted branch probabilities are $q_{0,R}=q_{L,R}=q_{R,R}=1/2$, whereas the true branch probabilities are $q_{0,R}^\star=3/4,q_{L,R}^\star=0,q_{R,R}^\star=1/2$, and $\bar u_{0,L}=0,\bar u_{0,R}=1$. Consequently,
$D_{\mathrm{top}}(p)=1/4$, $D_L(p)=-1/4$, $D_R(p)=0$.
The predicted and actual mean utilities are both $1/2$, so the total residual vanishes. An interval either contains the sole score $1/2$ or does not. Thus $\UC(f;\{1\},\{u\})=0$ and $\HUC(f;\{1\},\{u\})=1/4$. This realizes exact cancellation under a valid joint distribution.
\end{example}
\begin{lemma}[Characterization of canonical calibration by branch residuals]
\label{lem:full-branch-equivalence}
Assume that every component of the prediction vector is positive. The following are equivalent.
\begin{enumerate}
\item[(i)] $f$ is canonically calibrated.
\item[(ii)] $\mathbb{E}[\boldsymbol R_v(P,Y)\mid P]=\boldsymbol 0$ for every internal node $v$.
\end{enumerate}
\end{lemma}
\begin{proof}
The implication from (i) to (ii) was shown in the preceding proof. We now prove the converse. For every node $a$, write $\zeta_a(P):=\mathbb{E}[J_a(Y)\mid P]$. At the root, $J_\rho(Y)=1$, so $\zeta_\rho(P)=1=\pi_\rho(P)$.

Suppose $\zeta_v(P)=\pi_v(P)$ holds at an internal node $v$. The $j$th component of (ii) gives
\begin{align*}
0&=\mathbb{E}[R_{v,j}(P,Y)\mid P]=\zeta_{v_j}(P)-q_{v,j}(P)\zeta_v(P),\\
\zeta_{v_j}(P)&=q_{v,j}(P)\zeta_v(P)=q_{v,j}(P)\pi_v(P)
=\frac{\pi_{v_j}(P)}{\pi_v(P)}\pi_v(P)=\pi_{v_j}(P).
\end{align*}
Induction from the root to the leaves gives $\zeta_y(P)=\pi_y(P)$ for every leaf $y$. Since $\zeta_y(P)=\mathbb{P}(Y=y\mid P)$ and $\pi_y(P)=P_y$ at a leaf, canonical calibration follows.
\end{proof}

\subsection{Marginal Calibration and Conditional Branches}
\label{subsec:subtree-counterexample}\label{subsec:subtree-utility-counterexample}
This example supplements the distinction in Section~\ref{sec:single-utility-cancellation} between calibration of individual leaf or subtree probabilities and agreement of conditional branch probabilities. We first verify that the former calibration errors are zero, and then calculate the conditional branch-probability errors that remain.

Consider a tree with leaves $\{e,f,o_1,o_2\}$, in which $D:=\{e,f\}$ branches into $E:=\{e\}$ and $F:=\{f\}$. Let the two input groups $X\in\{a,b\}$ each have probability $1/2$. In the coordinate order $(e,f,o_1,o_2)$, let the predicted and true conditional distributions be, respectively,
\begin{align}
p^{(a)}&=(1/4,1/4,2/5,1/10),\\
p^{(b)}&=(1/4,1/4,1/10,2/5),\label{eq:counterexample-predictions}\\
\eta^{(a)}&=(3/8,1/8,2/5,1/10),\\
\eta^{(b)}&=(1/8,3/8,1/10,2/5).\label{eq:counterexample-truth}
\end{align}
Thus $P=p^{(g)}$ when $X=g$.

\paragraph{Zero UC for the subtree indicator utilities.}
For a leaf subset $A$, define $u_A(p,y):=\mathbf{1}\{y\in A\}$, $p(A):=\sum_{y\in A}p_y$, and $\mu_{u_A}(p)=p(A)$.
Both groups have $P(D)=1/2$ and $P(E)=1/4$. Conditioning on either scalar score therefore retains both groups with equal probability, and
\begin{align*}
\mathbb{P}(Y\in D\mid P(D)=1/2)
&=\frac12(\frac38+\frac18)
 +\frac12(\frac18+\frac38)=\frac12,\\
\mathbb{P}(Y\in E\mid P(E)=1/4)
&=\frac12\cdot\frac38+\frac12\cdot\frac18=\frac14.
\end{align*}
Hence $\mathbb{E}[u_A(P,Y)\mid P(A)]=P(A)$ for $A=D,E$. Substituting this identity into the UC definition gives, for every interval $I$,
\begin{align*}
&\mathbb{E}\!\left[\mathbf{1}\{P(A)\in I\}
\bigl(u_A(P,Y)-P(A)\bigr)\right]\\
&\qquad=\mathbb{E}\!\left[\mathbf{1}\{P(A)\in I\}
\bigl(\mathbb{E}[u_A(P,Y)\mid P(A)]-P(A)\bigr)\right]=0.
\end{align*}
Therefore $\UC(f;\{1\},\{u_D,u_E\})=0$.

\paragraph{Zero class-wise calibration error.}
Class-wise calibration requires $\mathbb{P}(Y=k\mid P_k)=P_k$ for each leaf $k$. For $e$ and $f$, the scalar predictions are constant at $1/4$, and
\begin{align*}
\mathbb{P}(Y=e\mid P_e=1/4)
&=\frac12\cdot\frac38+\frac12\cdot\frac18=\frac14,\\
\mathbb{P}(Y=f\mid P_f=1/4)
&=\frac12\cdot\frac18+\frac12\cdot\frac38=\frac14.
\end{align*}
For $o_1$ and $o_2$, each scalar prediction identifies one group, whose true probability agrees with that prediction:
\[
\begin{aligned}
&\mathbb{P}(Y=o_1\mid P_{o_1}=2/5)=\eta^{(a)}_{o_1}=\frac25, \qquad  \mathbb{P}(Y=o_1\mid P_{o_1}=1/10)=\eta^{(b)}_{o_1}=\frac1{10},\\
&\mathbb{P}(Y=o_2\mid P_{o_2}=1/10)=\eta^{(a)}_{o_2}=\frac1{10}, \qquad  \mathbb{P}(Y=o_2\mid P_{o_2}=2/5)=\eta^{(b)}_{o_2}=\frac25.
\end{aligned}
\]
These calculations cover every possible scalar prediction for all four leaves, so class-wise calibration holds exactly.

\paragraph{Nonzero conditional branch-probability errors.}
The conditional branch agreement discussed in Section~\ref{sec:single-utility-cancellation} would require $\mathbb{P}(Y\in E\mid Y\in D,P)=P(E)/P(D)$.
The right-hand side is the predicted probability of taking the $D\to E$ branch. In both groups it equals
$\frac{p^{(a)}(E)}{p^{(a)}(D)} =\frac{p^{(b)}(E)}{p^{(b)}(D)} =\frac{1/4}{1/2}=\frac12$.
However, $p^{(a)}\ne p^{(b)}$, so the full prediction vector identifies the group. The true conditional branch probabilities are
\begin{align*}
\mathbb{P}(Y\in E\mid Y\in D,P=p^{(a)})
&=\frac{\eta^{(a)}(E)}{\eta^{(a)}(D)}
=\frac{3/8}{3/8+1/8}=\frac34,\\
\mathbb{P}(Y\in E\mid Y\in D,P=p^{(b)})
&=\frac{\eta^{(b)}(E)}{\eta^{(b)}(D)}
=\frac{1/8}{1/8+3/8}=\frac14.
\end{align*}
Thus the errors are
\begin{align*}
\mathbb{P}(Y\in E\mid Y\in D,P=p^{(a)})
-\frac{p^{(a)}(E)}{p^{(a)}(D)}
&=\frac34-\frac12=\frac14,\\
\mathbb{P}(Y\in E\mid Y\in D,P=p^{(b)})
-\frac{p^{(b)}(E)}{p^{(b)}(D)}
&=\frac14-\frac12=-\frac14.
\end{align*}
The $D\to E$ probability is underestimated by $1/4$ in group $a$ and overestimated by $1/4$ in group $b$.

The reason these errors do not appear in the indicator-utility UC is also explicit. The conditional mean utility residuals are
\begin{align*}
\mathbb{E}[u_E(P,Y)\mid P=p^{(a)}]-\mu_{u_E}(p^{(a)})
&=\frac38-\frac14=\frac18,\\
\mathbb{E}[u_E(P,Y)\mid P=p^{(b)}]-\mu_{u_E}(p^{(b)})
&=\frac18-\frac14=-\frac18.
\end{align*}
Both groups have predicted mean utility $1/4$, so a predicted-utility interval cannot separate them, and their residuals cancel:
$\frac12\cdot\frac18+\frac12(-\frac18)=0$.
The same cancellation occurs in class-wise calibration of $e$, since both groups have $P_e=1/4$. Consequently, zero indicator-utility UC and zero class-wise calibration error do not guarantee agreement of branch probabilities conditional on the full prediction vector. In this example the conditional branch-probability error has absolute value $1/4$ in each group.

\subsection{Branch Errors versus Utility Impact}
\label{subsec:branch-discrepancy-utility-ranking}
Section~\ref{sec:single-utility-cancellation} distinguishes a branch-probability error from its effect on the specified utility. We verify this distinction by substituting into the three contributions $D_{\mathrm{top}},D_L,D_R$ defined in Eq.~\eqref{eq:d-full-sec32}.

Use the four-leaf tree of that section, with $y_1,y_2$ below $v_L$ and $y_3,y_4$ below $v_R$. Let the predictor be constant at $p=(1/4,1/4,1/4,1/4)$, let the true conditional distribution be $\eta=(1/20,9/20,1/5,3/10)$, and take the target utility to be
$u(p,y):=\mathbf{1}\{y=y_4\}$, $(u_1,u_2,u_3,u_4)=(0,0,0,1)$.

\paragraph{Branch-probability errors.}
The predicted right-branch probabilities are
$q_{0,R}=p_3+p_4=\frac12$, $q_{L,R}=\frac{p_2}{p_1+p_2}=\frac12$, $q_{R,R}=\frac{p_4}{p_3+p_4}=\frac12$.
The corresponding true probabilities are
\[
\begin{aligned}
&q_{0,R}^\star=\eta_3+\eta_4=\frac15+\frac3{10}=\frac12, \qquad  q_{L,R}^\star=\frac{\eta_2}{\eta_1+\eta_2}
=\frac{9/20}{1/20+9/20}=\frac9{10},\\
&q_{R,R}^\star=\frac{\eta_4}{\eta_3+\eta_4}
=\frac{3/10}{1/5+3/10}=\frac35.
\end{aligned}
\]
Hence
$|q_{0,R}^\star-q_{0,R}|=0$, $|q_{L,R}^\star-q_{L,R}|=\frac25$, $|q_{R,R}^\star-q_{R,R}|=\frac1{10}$.
The left node has the largest branch-probability error.

\paragraph{Contributions to the target-utility error.}
The predicted subtree mean utilities from Section~\ref{sec:single-utility-cancellation} are $\bar u_{0,L}=q_{L,L}u_1+q_{L,R}u_2=0$ and $\bar u_{0,R}=q_{R,L}u_3+q_{R,R}u_4=1/2$.
Direct substitution into Eq.~\eqref{eq:d-full-sec32} gives
\begin{align*}
D_{\mathrm{top}}(p)
&=(q_{0,R}^\star-q_{0,R})(\bar u_{0,R}-\bar u_{0,L})
=(\frac12-\frac12)(\frac12-0)=0,\\
D_L(p)
&=q_{0,L}^\star(q_{L,R}^\star-q_{L,R})(u_2-u_1)
=(1-\frac12)(\frac9{10}-\frac12)(0-0)=0,\\
D_R(p)
&=q_{0,R}^\star(q_{R,R}^\star-q_{R,R})(u_4-u_3)
=\frac12(\frac35-\frac12)(1-0)=\frac1{20}.
\end{align*}
At the left node, the branch-probability error is $2/5$, but the utility difference $u_2-u_1$ multiplying it is zero. At the right node, the smaller branch-probability error $1/10$ is multiplied by the nonzero utility difference $u_4-u_3=1$, giving contribution $1/20$. Indeed,
\[
D^{\mathrm{full}}(p) =\sum_{i=1}^4\eta_i u_i-\sum_{i=1}^4p_i u_i =\frac3{10}-\frac14=\frac1{20} =D_{\mathrm{top}}(p)+D_L(p)+D_R(p).
\]
Thus the target-utility error in this example comes from the right node, whereas ranking nodes only by branch-probability error places the left node first, despite its zero utility contribution. Branch-probability error alone therefore does not identify the source of the specified target-utility error.

\subsection{Node-Specific Branch Cells Can Audit Different Populations}
\label{subsec:node-specific-branch-cells}
Section~\ref{sec:single-utility-cancellation} seeks to compare the sources of utility error within the same population selected by predicted mean utility. This example shows that selecting a population by each node's own branch score need not preserve that common population.

Use the same four-leaf tree and two input groups $a,b$, each with probability $1/2$. Let the predicted distributions be
$p^{(a)}=(1/2,1/10,3/20,1/4)$, $p^{(b)}=(1/5,3/10,1/4,1/4)$,
and the true conditional distributions be
$\eta^{(a)}=(5/12,1/12,3/16,5/16)$, $\eta^{(b)}=(1/5,3/10,3/10,1/5)$.
Again take $u(p,y):=\mathbf{1}\{y=y_4\}$.

\paragraph{The same predicted-utility cell.}
By the definition of predicted mean utility,
$\mu_u(p^{(a)})=p_4^{(a)}=\frac14$, $\mu_u(p^{(b)})=p_4^{(b)}=\frac14$.
Both groups therefore belong to the same cell $\{\mu_u(P)=1/4\}$.

\paragraph{Different sources of utility error in the two groups.}
The root and right-node predicted branch probabilities are
\begin{align*}
q_{0,R}(p^{(a)})&=\frac3{20}+\frac14=\frac25,&
q_{R,R}(p^{(a)})&=\frac{1/4}{3/20+1/4}=\frac58,\\
q_{0,R}(p^{(b)})&=\frac14+\frac14=\frac12,&
q_{R,R}(p^{(b)})&=\frac{1/4}{1/4+1/4}=\frac12.
\end{align*}
The corresponding true probabilities are
\begin{align*}
q_{0,R}^\star(p^{(a)})&=\frac3{16}+\frac5{16}=\frac12,&
q_{R,R}^\star(p^{(a)})&=\frac{5/16}{3/16+5/16}=\frac58,\\
q_{0,R}^\star(p^{(b)})&=\frac3{10}+\frac15=\frac12,&
q_{R,R}^\star(p^{(b)})&=\frac{1/5}{3/10+1/5}=\frac25.
\end{align*}
For this utility, $\bar u_{0,L}=0$, $\bar u_{0,R}=q_{R,R}$, $u_2-u_1=0$, and $u_4-u_3=1$. Substituting into Eq.~\eqref{eq:d-full-sec32} yields
\[
\begin{aligned}
&D_{\mathrm{top}}(p^{(a)})
=(\frac12-\frac25)(\frac58-0)=\frac1{16}, \qquad  D_R(p^{(a)})
=\frac12(\frac58-\frac58)(1-0)=0,\\
&D_{\mathrm{top}}(p^{(b)})
=(\frac12-\frac12)(\frac12-0)=0, \qquad  D_R(p^{(b)})
=\frac12(\frac25-\frac12)(1-0)=-\frac1{20}.
\end{aligned}
\]
The left-node contribution is zero in both groups because $u_2-u_1=0$. Thus, among observations with the same predicted mean utility $1/4$, group $a$ has a root contribution $1/16$, whereas group $b$ has a right-node contribution $-1/20$.

\paragraph{Branch-score cells do not preserve the common population.}
For the root score, take $I_0=[3/8,7/16]$. Then
$q_{0,R}(p^{(a)})=\frac25\in I_0$, $q_{0,R}(p^{(b)})=\frac12\notin I_0$,
so this cell selects only group $a$. For the right-node score, take $I_R=[7/16,9/16]$. Then
$q_{R,R}(p^{(a)})=\frac58\notin I_R$, $q_{R,R}(p^{(b)})=\frac12\in I_R$,
so this cell selects only group $b$. Equivalently,
$\mathbf{1}\{q_{0,R}(P)\in I_0\}=\mathbf{1}\{X=a\}$, $\mathbf{1}\{q_{R,R}(P)\in I_R\}=\mathbf{1}\{X=b\}$.
The two cells select disjoint input populations. Their results therefore do not compare the root and right-node contributions within the same population.

By contrast, retaining the common predicted-utility interval $I=[1/4,1/4]$ from Section~\ref{sec:single-utility-cancellation} gives
\begin{align*}
\mathbb{E}\!\left[\mathbf{1}\{\mu_u(P)\in I\}D_{\mathrm{top}}(P)\right]
&=\frac12\cdot\frac1{16}+\frac12\cdot0=\frac1{32},\\
\mathbb{E}\!\left[\mathbf{1}\{\mu_u(P)\in I\}D_L(P)\right]&=0,\\
\mathbb{E}\!\left[\mathbf{1}\{\mu_u(P)\in I\}D_R(P)\right]
&=\frac12\cdot0+\frac12(-\frac1{20})=-\frac1{40}.
\end{align*}
All three expectations concern the same population containing both groups. Comparing the hierarchical sources of the target-utility error requires retaining this common predicted-utility cell rather than changing the selected population from node to node.

\paragraph{Auditing every branch also includes utility-irrelevant fluctuations.}
At the left node, the predicted and true branch probabilities are
\begin{align*}
q_{L,R}(p^{(a)})&=\frac{1/10}{1/2+1/10}=\frac16,&
q_{L,R}^\star(p^{(a)})&=\frac{1/12}{5/12+1/12}=\frac16,\\
q_{L,R}(p^{(b)})&=\frac{3/10}{1/5+3/10}=\frac35,&
q_{L,R}^\star(p^{(b)})&=\frac{3/10}{1/5+3/10}=\frac35.
\end{align*}
Thus the population branch-probability error is zero in each group. Nevertheless, estimating these probabilities from a finite sample introduces fluctuations in the observed leaf frequencies. For example, conditional on group $a$ and on reaching the left subtree, the difference from the predicted probability $1/6$ is
\[
\mathbf{1}\{Y=y_2\}-\frac16
=\begin{cases}
-\frac16,&Y=y_1,\\
\frac56,&Y=y_2.
\end{cases}
\]
Its population mean is
$\frac56(-\frac16)+\frac16\cdot\frac56=0$,
but its finite-sample mean need not be zero. A full branch audit therefore includes fluctuating empirical discrepancies in the cells at scores $1/6$ and $3/5$.

For the specified target utility, however, the left-node term in Eq.~\eqref{eq:d-full-sec32} is multiplied by $u_2-u_1=0$, so $D_L(p)=q_{0,L}^\star(q_{L,R}^\star-q_{L,R})\underbrace{(u_2-u_1)}_{=0}=0$.
This remains zero regardless of the branch-probability error or its empirical estimate. The fluctuating allocation between $y_1$ and $y_2$ does not affect this utility. Auditing every branch would include these utility-irrelevant fluctuations. To locate the target-utility error in this example, it suffices to compare the root and right-node contributions, which can be nonzero, on the same predicted-utility cell.

\section{Proofs for HUC-Boost}
\label{app:huc-boost}

This appendix proves the loss decrease and finite termination stated in Theorem~\ref{thm:finite-termination-explicit}. Appendix~\ref{subsec:softmax-leaf-distribution} prepares the softmax branch probabilities, the induced leaf distribution, and log loss. Appendix~\ref{subsec:huc-boost-theorem-proof} then proves the theorem in four steps. We then compare branch-local updates with ordinary leaf Brier loss and prove the auxiliary path-product results.

\subsection{Softmax Branches, Leaf Distribution, and Log Loss}
\label{subsec:softmax-leaf-distribution}
Before the main proof, we verify that node-wise softmax branch probabilities define a positive leaf distribution, that the original conditional branch probabilities are recovered from this leaf distribution, and that the population log loss used as the progress potential is well defined.

For a predictor $g:\gX\to\Delta^{K-1}$, define the population log loss by
\begin{equation}
\mathcal L_{\log}(g):=\mathbb{E}[-\log g(X)_Y].
\label{eq:population-log-loss}
\end{equation}
Assign logits $\boldsymbol z_v(x)=(z_{v,1}(x),\ldots,z_{v,d_v}(x))^\top$ to each internal node $v$, and set
\begin{equation}
q_{v,j}^{\boldsymbol z}(x):=\frac{\exp(z_{v,j}(x))}{\sum_{k=1}^{d_v}\exp(z_{v,k}(x))},\qquad j=1,\ldots,d_v.
\label{eq:softmax-branch}
\end{equation}
For every leaf $y$, use the path sequence in~Eq.~\eqref{eq:root-leaf-path} to define
\begin{equation}
f_{\boldsymbol z}(x)_y:=\prod_{r=0}^{L_y-1}q_{w_r(y),j_r(y)}^{\boldsymbol z}(x).
\label{eq:softmax-leaf-product}
\end{equation}
Lemma~\ref{lem:softmax-leaf-distribution} shows that this path product defines a positive probability distribution over the leaves for every input. Lemma~\ref{lem:softmax-branch-subtree-consistency} further gives, for every internal node $v$ and child $v_j$,
\begin{equation}
\pi_{v_j}(f_{\boldsymbol z}(x))
=
\pi_v(f_{\boldsymbol z}(x))q_{v,j}^{\boldsymbol z}(x).
\label{eq:softmax-path-basic-properties}
\end{equation}
Hence the conditional branch probability recovered from the leaf distribution by~Eq.~\eqref{eq:branch-conditional} equals the original softmax branch probability. Appendix~\ref{subsec:softmax-auxiliary-results} proves both lemmas.

\subsection{Proof of Theorem~\ref{thm:finite-termination-explicit}}
\label{subsec:huc-boost-theorem-proof}
The proof proceeds in four steps. Step 1 derives how a one-node logit update changes branch and leaf probabilities and identifies the probability structure it preserves. Step 2 evaluates the resulting change in ordinary leaf log loss. Step 3 computes the first and second derivatives needed for the Taylor expansion in Step 4. Step 4 establishes the one-step loss decrease and finite termination.

\paragraph{Step 1: probability changes under a one-node logit update.}
At iteration $t$, let $(c_t,u_t,I_t,v_t)$ be the selected candidate, and write $v:=v_t$ and $d:=d_v$. Define the update direction at the selected node by
\begin{equation}
\boldsymbol h_t(x)
:=
c_t(x)\mathbf{1}\{\mu_{u_t}(f_t(x))\in I_t\}
\boldsymbol\Delta_{u_t,v}(f_t(x)).
\label{eq:huc-boost-frozen-direction}
\end{equation}
For any $\alpha\in\R$, set
\begin{equation}
\boldsymbol z_{t,v}^{(\alpha)}(x)
:=\boldsymbol z_{t,v}(x)+\alpha\boldsymbol h_t(x),
\qquad
\boldsymbol z_{t,w}^{(\alpha)}(x)
:=\boldsymbol z_{t,w}(x)
\quad(w\ne v),
\label{eq:one-node-trial-logit-update}
\end{equation}
and denote the corresponding branch probabilities and leaf distribution by $\boldsymbol q_{t,w}^{(\alpha)}$ and $f_t^{(\alpha)}$.

\begin{proposition}[Branch and leaf probabilities after a one-node update]
\label{prop:boost-leaf-rescaling}
Let
\begin{equation}
Z_t^{(\alpha)}(x)
:=
\sum_{k=1}^{d}q_{t,v,k}(x)e^{\alpha h_{t,k}(x)}.
\label{eq:one-node-logit-normalizer}
\end{equation}
At the selected node,
\begin{equation}
q_{t,v,j}^{(\alpha)}(x)
=
\frac{q_{t,v,j}(x)e^{\alpha h_{t,j}(x)}}{Z_t^{(\alpha)}(x)},
\qquad j=1,\ldots,d,
\label{eq:one-node-multiplicative-branch-update}
\end{equation}
whereas the branch probabilities at every $w\ne v$ remain unchanged. Moreover, for every leaf $y$,
\begin{equation}
f_t^{(\alpha)}(x)_y
=
\begin{cases}
f_t(x)_y,&y\notin\gY_v,\\
f_t(x)_y\frac{e^{\alpha h_{t,j}(x)}}{Z_t^{(\alpha)}(x)},&y\in\gY_{v_j}.
\end{cases}
\label{eq:boost-leaf-rescale}
\end{equation}
Consequently, the update preserves every leaf probability outside the selected subtree, the reach probability of the selected node, and the conditional leaf distribution within each child subtree.
\end{proposition}

\begin{proof}
By the softmax definition,
\[
q_{t,v,j}^{(\alpha)}(x)
=
\frac{e^{z_{t,v,j}(x)+\alpha h_{t,j}(x)}}{\sum_{k=1}^{d}e^{z_{t,v,k}(x)+\alpha h_{t,k}(x)}}
=
\frac{q_{t,v,j}(x)e^{\alpha h_{t,j}(x)}}{\sum_{k=1}^{d}q_{t,v,k}(x)e^{\alpha h_{t,k}(x)}}.
\]
This proves~Eq.~\eqref{eq:one-node-multiplicative-branch-update}. If the path to $y$ does not pass through $v$, none of its path-product factors changes. If $y\in\gY_{v_j}$, the only changed factor is $q_{t,v,j}$, yielding~Eq.~\eqref{eq:boost-leaf-rescale}.

Summing~Eq.~\eqref{eq:boost-leaf-rescale} over $\gY_{v_j}$ gives
\[
\pi_{v_j}(f_t^{(\alpha)}(x))=\pi_v(f_t(x))q_{t,v,j}^{(\alpha)}(x).
\]
Summing further over $j$ yields $\pi_v(f_t^{(\alpha)}(x))=\pi_v(f_t(x))$. Hence, for $y\in\gY_{v_j}$,
\[
f_t^{(\alpha)}(x)_y/\pi_{v_j}(f_t^{(\alpha)}(x))
=f_t(x)_y/\pi_{v_j}(f_t(x)),
\]
which proves preservation of the conditional leaf distribution within every child subtree.
\end{proof}

\paragraph{Step 2: change in the potential under a one-node update.}
Recall the subtree-membership indicators $J_v(y)=\mathbf{1}\{y\in\gY_v\}$ and $J_{v,j}(y)=\mathbf{1}\{y\in\gY_{v_j}\}$, where $\gY_v$ and $\gY_{v_j}$ are the descendant-leaf sets of $v$ and its child $v_j$.
We now evaluate how the one-node logit update changes the population log loss defined in Appendix~\ref{subsec:softmax-leaf-distribution}.

\begin{lemma}[Exact change in log loss under a one-node update]
\label{lem:nodewise-logloss}
For every $(x,y)$,
\begin{equation}
-\log f_t^{(\alpha)}(x)_y+\log f_t(x)_y
=
J_v(y)\log Z_t^{(\alpha)}(x)
-
\alpha\sum_{j=1}^{d}J_{v,j}(y)h_{t,j}(x).
\label{eq:one-node-pointwise-logloss-change}
\end{equation}
Consequently,
\begin{equation}
\mathcal L_{\log}(f_t^{(\alpha)})
-
\mathcal L_{\log}(f_t)
=
\mathbb{E}\left[
J_v(Y)\log Z_t^{(\alpha)}(X)
-
\alpha\sum_{j=1}^{d}J_{v,j}(Y)h_{t,j}(X)
\right].
\label{eq:one-node-global-logloss-change}
\end{equation}
\end{lemma}

\begin{proof}
Proposition~\ref{prop:boost-leaf-rescaling} gives, for every leaf $y$,
\[
\frac{f_t^{(\alpha)}(x)_y}{f_t(x)_y}
=
\exp\left(
\alpha\sum_{j=1}^{d}J_{v,j}(y)h_{t,j}(x)
-
J_v(y)\log Z_t^{(\alpha)}(x)
\right).
\]
If $y\notin\gY_v$, all indicators on the right are zero. If $y\in\gY_{v_j}$, the right-hand side equals $e^{\alpha h_{t,j}(x)}/Z_t^{(\alpha)}(x)$. Negative logarithms give Eq.~\eqref{eq:one-node-pointwise-logloss-change}, and expectations give Eq.~\eqref{eq:one-node-global-logloss-change}.
\end{proof}

\paragraph{Step 3: first and second derivatives needed for the one-step decrease.}
For the Taylor expansion in Step 4, define
$\phi_t(\alpha):=\mathcal L_{\log}(f_t^{(\alpha)})$
and evaluate its first derivative at $\alpha=0$ and its second derivative at arbitrary $\alpha$. Set
\begin{align}
\Gamma_t
&:=
\Gamma_{c_t,u_t,I_t,v_t}(f_t)
=
\mathbb{E}\left[
\left\langle
\boldsymbol h_t(X),
\boldsymbol R_v(f_t(X),Y)
\right\rangle
\right],
\label{eq:general-direction-gamma}\\
\Lambda_t
&:=
\frac14
\mathbb{E}\left[
J_v(Y)
\left\{
\max_j h_{t,j}(X)-\min_j h_{t,j}(X)
\right\}^2
\right].
\label{eq:general-direction-lambda}
\end{align}

\begin{lemma}[Directional derivative and curvature along a one-node update]
\label{lem:boost-curvature}
Assume $\mathcal L_{\log}(f_t)<\infty$. Then $\phi_t(\alpha)<\infty$ for every finite $\alpha\in\R$, and
\begin{align}
\phi_t'(0)&=-\Gamma_t,
\label{eq:boost-first-derivative}\\
\phi_t''(\alpha)
&=
\mathbb{E}\left[
J_v(Y)
\operatorname{Var}_{J\sim\boldsymbol q_{t,v}^{(\alpha)}(X)}
(h_{t,J}(X))
\right]
\le\Lambda_t\le1.
\label{eq:boost-second-derivative}
\end{align}
\end{lemma}

\begin{proof}
From~Eq.~\eqref{eq:one-node-global-logloss-change},
\[
\frac{d}{d\alpha}\log Z_t^{(\alpha)}(x)
=
\sum_{j=1}^{d}q_{t,v,j}^{(\alpha)}(x)h_{t,j}(x),
\]
and therefore
\[
\phi_t'(\alpha)
=
-\mathbb{E}\left[
\sum_{j=1}^{d}h_{t,j}(X)
\left\{
J_{v,j}(Y)-J_v(Y)q_{t,v,j}^{(\alpha)}(X)
\right\}
\right].
\]
Setting $\alpha=0$ gives $\phi_t'(0)=-\Gamma_t$. Moreover,
\[
\frac{d^2}{d\alpha^2}\log Z_t^{(\alpha)}(x)
=
\operatorname{Var}_{J\sim\boldsymbol q_{t,v}^{(\alpha)}(x)}(h_{t,J}(x)),
\]
which proves the equality in~Eq.~\eqref{eq:boost-second-derivative}.

A random variable supported on $[a,b]$ has variance at most $(b-a)^2/4$. Applying this pointwise in $x$ gives $\phi_t''(\alpha)\le\Lambda_t$. Since $\max_jh_{t,j}(x)-\min_jh_{t,j}(x)\le2$, we also have $\Lambda_t\le1$. The absolute value of~Eq.~\eqref{eq:one-node-pointwise-logloss-change} is at most $2|\alpha|$, and the first and second derivatives are uniformly bounded; hence differentiation and expectation may be interchanged.
\end{proof}

\paragraph{Step 4: one-step loss decrease and finite termination.}
\begin{lemma}[Positive curvature for a nonzero audit moment]
\label{lem:positive-curvature}
If $|\Gamma_t|>0$, then $\Lambda_t>0$.
\end{lemma}

\begin{proof}
We prove the contrapositive. If $\Lambda_t=0$, then on every point with $J_v(Y)=1$, all components of $\boldsymbol h_t(X)$ are equal. Since $\sum_jR_{v,j}(f_t(X),Y)=0$,
$\left\langle \boldsymbol h_t(X), \boldsymbol R_v(f_t(X),Y) \right\rangle=0$.
When $J_v(Y)=0$, $\boldsymbol R_v(f_t(X),Y)=\boldsymbol 0$. Therefore $\Gamma_t=0$.
\end{proof}

\begin{proof}[Proof of Theorem~\ref{thm:finite-termination-explicit}]
Lemma~\ref{lem:positive-curvature} gives $\Lambda_t>0$ whenever an update is made. Lemma~\ref{lem:boost-curvature} and the integral form of Taylor's theorem imply, for every $\alpha\in\R$,
\[
\phi_t(\alpha)
\le
\phi_t(0)-\alpha\Gamma_t+\frac{\alpha^2}{2}\Lambda_t.
\]
Substituting $\alpha=\alpha_t=\Gamma_t/\Lambda_t$ gives the one-step loss decrease used in Theorem~\ref{thm:finite-termination-explicit}. Proposition~\ref{prop:boost-leaf-rescaling} ensures that all branch and leaf probabilities remain positive after every finite number of updates.

Before termination, $|\Gamma_t|>\varepsilon$ and $\Lambda_t\le1$, so each update decreases log loss by more than $\varepsilon^2/2$. If $m$ updates are made, nonnegativity of log loss gives
\[
\mathcal L_{\log}(f_0)
\ge
\sum_{t=0}^{m-1}
\left\{\mathcal L_{\log}(f_t)-\mathcal L_{\log}(f_{t+1})\right\}
>
m\frac{\varepsilon^2}{2}.
\]
Thus the update bound in Theorem~\ref{thm:finite-termination-explicit} follows. At termination, $|\Gamma_{c,u,I,v}(f_T)|\le\varepsilon$ for every candidate, and hence $\HUC(f_T;\gC,\gU)\le\varepsilon$.
\end{proof}

\paragraph{Fixed-step empirical update.}
On a fixed sample $S=((X_i,Y_i))_{i=1}^n$, define
\begin{equation}
\widehat{\mathcal L}_{\log,S}(g)
:=-\frac1n\sum_{i=1}^n\log g(X_i)_{Y_i}.
\label{eq:empirical-log-loss}
\end{equation}
Given the current predictor $f$ and a candidate $(c,u,I,v)$, set $\boldsymbol h(x):=c(x)\mathbf 1\{\mu_u(f(x))\in I\}\boldsymbol\Delta_{u,v}(f(x))$, and denote the corresponding empirical audit moment by $\widehat{\Gamma}$.

\begin{corollary}[Empirical loss decrease under a signed fixed-step update]
\label{cor:empirical-fixed-step-descent}
For any $\alpha>0$ and $\sigma\in\{-1,1\}$, let $f^+$ be the predictor obtained by updating the logits at the selected node as $\boldsymbol z_v^+=\boldsymbol z_v+\sigma\alpha\boldsymbol h$. Then
\begin{equation}
\widehat{\mathcal L}_{\log,S}(f^+)
-
\widehat{\mathcal L}_{\log,S}(f)
\le
-\alpha\sigma\widehat{\Gamma}
+
\frac{\alpha^2}{2}.
\label{eq:empirical-fixed-step-descent}
\end{equation}
In particular, if $\sigma=\operatorname{sgn}(\widehat{\Gamma})$, $|\widehat{\Gamma}|>3\varepsilon/4$, and $\alpha=\varepsilon/2$, then the empirical log loss decreases by more than $\varepsilon^2/4$.
\end{corollary}

\begin{proof}
Apply the Taylor bound of Lemma~\ref{lem:boost-curvature} to the empirical average and use the curvature bound 1.
\end{proof}

\subsection{Ordinary Leaf Brier Loss Is Not Branch-Local}
\label{subsec:leaf-brier-nonlocality}
Define the ordinary multiclass Brier loss by
\begin{equation}
\mathcal B_{\mathrm{leaf}}(p,y):=\|e_y-p\|_2^2
\label{eq:ordinary-leaf-brier}
\end{equation}
For UC-Boost, which updates the full leaf-probability vector, the first-order term of this loss aligns with the leaf residual $e_y-p$. A branch-local HUC-Boost update preserves the parent reach probability and the conditional leaf distributions within child subtrees, so its direction in leaf space differs.

Fix a positive prediction vector $p$, an internal node $v$, and its children $v_1,\ldots,v_{d_v}$. Define the conditional leaf distribution within each child subtree, embedded in the full leaf space, and its squared norm by
\begin{equation}
[\boldsymbol r_{v,j}(p)]_z:=\frac{p_z}{\pi_{v_j}(p)}\mathbf{1}\{z\in\gY_{v_j}\},\qquad C_{v,j}(p):=\|\boldsymbol r_{v,j}(p)\|_2^2
\label{eq:child-conditional-leaf-vector}
\end{equation}
Hold all branches other than $v$ fixed, and let $\boldsymbol q'_v=\boldsymbol q_v(p)+\boldsymbol\delta$, where $\sum_j\delta_j=0$ and every component of $\boldsymbol q'_v$ is positive. The same path-product calculation as in Step 1 gives
\begin{equation}
p'-p=\pi_v(p)\sum_{j=1}^{d_v}\delta_j\boldsymbol r_{v,j}(p).
\label{eq:branch-local-leaf-change}
\end{equation}
The supports of the $\boldsymbol r_{v,j}$ are disjoint, and $\left\langle \boldsymbol r_{v,j}(p),p\right\rangle=\pi_v(p)q_{v,j}(p)C_{v,j}(p)$. Hence
\begin{align}
&\mathcal B_{\mathrm{leaf}}(p',y)-\mathcal B_{\mathrm{leaf}}(p,y)\notag\\
&\quad=-2\left\langle e_y-p,p'-p\right\rangle+\|p'-p\|_2^2\notag\\
&\quad=-2\pi_v(p)\sum_{j=1}^{d_v}\delta_j\left\{[\boldsymbol r_{v,j}(p)]_y-\pi_v(p)q_{v,j}(p)C_{v,j}(p)\right\}\notag\\
&\qquad\quad+\pi_v(p)^2\sum_{j=1}^{d_v}\delta_j^2C_{v,j}(p).
\label{eq:ordinary-brier-branch-change}
\end{align}

The first-order term in~Eq.~\eqref{eq:ordinary-brier-branch-change} does not generally align with the local HUC residual. To see this, consider an observed label outside the selected subtree, $y\notin\gY_v$. Then $J_v(y)=J_{v,j}(y)=0$ and $\boldsymbol R_v(p,y)=\boldsymbol 0$. Since $[\boldsymbol r_{v,j}(p)]_y=0$ as well,~Eq.~\eqref{eq:ordinary-brier-branch-change} becomes
\[
\mathcal B_{\mathrm{leaf}}(p',y)-\mathcal B_{\mathrm{leaf}}(p,y)
=2\pi_v(p)^2\sum_j\delta_jq_{v,j}(p)C_{v,j}(p)
+\pi_v(p)^2\sum_j\delta_j^2C_{v,j}(p).
\]
Although $\sum_j\delta_j=0$, the coefficients $q_{v,j}(p)C_{v,j}(p)$ generally differ across children, so the linear term need not vanish. For example, setting $\delta_j=t$, $\delta_k=-t$, and all other components to zero gives the linear term
\[
2\pi_v(p)^2t\{q_{v,j}(p)C_{v,j}(p)-q_{v,k}(p)C_{v,k}(p)\}.
\]
If the expression in braces is positive, the leaf Brier loss increases for sufficiently small $t>0$. Thus the first-order change in leaf Brier loss need not vanish even when the local HUC residual is zero, and a decrease guarantee cannot be obtained by directly identifying the HUC moment with this first-order term.

\paragraph{Lifting to leaf space versus branch locality.}
Lift the HUC node contribution to leaf space by setting
\begin{equation}
g_{u,v}(p)_z:=\left\langle \boldsymbol\Delta_{u,v}(p),\boldsymbol R_v(p,z)\right\rangle
\label{eq:huc-leaf-lift}
\end{equation}
Since $\sum_zp_z\boldsymbol R_v(p,z)=\boldsymbol 0$,
\[\left\langle \boldsymbol g_{u,v}(p),e_Y-p\right\rangle=\left\langle \boldsymbol\Delta_{u,v}(p),\boldsymbol R_v(p,Y)\right\rangle.\]
Thus a projected update on the full leaf simplex can use the HUC moment as the first-order Brier term. However, for $z\in\gY_{v_j}$, $g_{u,v}(p)_z=\mu_{u,v_j}(p)-\mu_{u,v}(p)$, so the same amount is added to all leaves in a child subtree. Unlike~Eq.~\eqref{eq:branch-local-leaf-change}, this change is not proportional to conditional leaf probabilities and generally changes relative probabilities within a child subtree. Simplex projection can also alter leaves outside the selected subtree, so the update is not branch-local.

The logit update in Appendix~\ref{subsec:huc-boost-theorem-proof} is paired with ordinary global leaf log loss.
\subsection{Auxiliary Results for Softmax Path Products}
\label{subsec:softmax-auxiliary-results}
This subsection proves the normalization of the path product and its consistency with the branch probabilities used in Appendix~\ref{subsec:softmax-leaf-distribution}.

\begin{lemma}[Path products define a probability distribution on the leaves]
\label{lem:softmax-leaf-distribution}
For every input $x$,
\begin{equation}
f_{\boldsymbol z}(x)_y>0\quad(y\in\gY),\qquad\sum_{y\in\gY}f_{\boldsymbol z}(x)_y=1.
\label{eq:softmax-leaf-normalization}
\end{equation}
Therefore $f_{\boldsymbol z}:\gX\to\Delta^{K-1}$.
\end{lemma}
\begin{proof}
Fix $x$. Each softmax vector in~Eq.~\eqref{eq:softmax-branch} has positive coordinates summing to one. Starting at the root, whenever an internal node $v$ is reached, choose its child $v_j$ with probability $q_{v,j}^{\boldsymbol z}(x)$ and continue until a leaf is reached. This specifies a valid probability distribution for the next branch at every step. Since the tree is finite and each move goes to a child, the procedure terminates at exactly one leaf.

For each leaf $y$, the multiplication rule for conditional probabilities along its unique root-to-leaf path gives its probability of being reached as
$\prod_{r=0}^{L_y-1}q_{w_r(y),j_r(y)}^{\boldsymbol z}(x) =f_{\boldsymbol z}(x)_y>0$.
The events of terminating at the respective leaves are disjoint and exhaustive, so $\sum_{y\in\gY}f_{\boldsymbol z}(x)_y=1$.
\end{proof}
\begin{lemma}[Consistency of the generated leaf distribution and the node branch probabilities]
\label{lem:softmax-branch-subtree-consistency}
Fix an input $x$ and set $p:=f_{\boldsymbol z}(x)$. For every internal node $v$ and child $v_j$,
\begin{equation}
\pi_{v_j}(p)=\pi_v(p)q_{v,j}^{\boldsymbol z}(x).
\label{eq:softmax-branch-subtree-consistency}
\end{equation}
Thus the predicted conditional branch probabilities computed from $p$ by~Eq.~\eqref{eq:branch-conditional} coincide with the original softmax probabilities $q_{v,j}^{\boldsymbol z}(x)$.
\end{lemma}
\begin{proof}
Use the root-to-leaf procedure in the proof of Lemma~\ref{lem:softmax-leaf-distribution}, whose terminal-leaf distribution is $p$. Terminating in $\gY_v$ is equivalent to reaching $v$, so $\pi_v(p)>0$ is the probability of reaching $v$. Conditional on reaching $v$, the next step selects $v_j$ with probability $q_{v,j}^{\boldsymbol z}(x)$. The terminal leaf lies in $\gY_{v_j}$ if and only if this child is selected. Since $\gY_{v_j}\subseteq\gY_v$, the definition of conditional probability gives
$\frac{\pi_{v_j}(p)}{\pi_v(p)}=q_{v,j}^{\boldsymbol z}(x)$.
Multiplying by $\pi_v(p)$ proves~Eq.~\eqref{eq:softmax-branch-subtree-consistency}, and the ratio is exactly the recovered branch probability in~Eq.~\eqref{eq:branch-conditional}.
\end{proof}

\section{Finite-Sample Evaluation and Exact Auditing}
\label{app:finite-sample}

Section~\ref{sec:huc-definition} ends with the finite-sample result. This appendix proves Theorem~\ref{thm:finite-sample-main}, tracks the numerical constants in the uniform deviation bound, and explains exact empirical maximization over all intervals on a finite sample. The proof follows the same outline as the finite-sample analysis of UC. For a fixed triple $(c,u,v)$, we reduce intervals to half-lines, evaluate score-ordered Rademacher partial sums, and then apply symmetrization, concentration, and a union bound over the $N_{\mathrm{cand}}$ triples. The predictor is fixed independently of the audit sample; sample-adaptive prediction is treated in Appendix~\ref{app:randomized-initialization}.

\subsection{Proof of Theorem~\ref{thm:finite-sample-main}}
\label{subsec:finite-sample-proof}\label{subsec:finite-sample-vc-proof}

\begin{proof}
We divide the proof into four steps. The first three fix one triple $(c,u,v)$ and vary only the interval $I$. The final step treats all triples simultaneously.

\paragraph{Step 1: bounding the contribution of one observation.}
Fix $c\in\gC$, $u\in\gU$, and $v\in\mathcal V_{\gT}(u)$, and set
\begin{equation}
A(X,Y):=c(X)\left\langle \boldsymbol\Delta_{u,v}(P),\boldsymbol R_v(P,Y)\right\rangle,\qquad s(X):=\mu_u(P).
\label{eq:finite-sample-weight-score}
\end{equation}
If $Y\notin\gY_v$, Lemma~\ref{lem:one-step-increment} gives $A(X,Y)=0$. If $Y\in\gY_{v_j}$, the same lemma gives
\begin{equation}
A(X,Y)=c(X)\{\mu_{u,v_j}(P)-\mu_{u,v}(P)\}.
\label{eq:finite-sample-local-value}
\end{equation}
Subtree mean utilities are convex combinations of values in $[-1,1]$, so both $\mu_{u,v_j}(P)$ and $\mu_{u,v}(P)$ belong to $[-1,1]$. Hence
$|A(X,Y)|\le |c(X)|\,|\mu_{u,v_j}(P)-\mu_{u,v}(P)|\le2$.
Thus, in all cases,
\begin{equation}
|A(X,Y)|\le2.
\label{eq:finite-sample-envelope}
\end{equation}
For $I\in\gI$, define
\begin{equation}
h_I(x,y):=A(x,y)\mathbf{1}\{s(x)\in I\}.
\label{eq:interval-indexed-local-function}
\end{equation}
The population and empirical audit moments are
\begin{equation}
\Gamma_{c,u,I,v}(f)=\mathbb{E}[h_I(X,Y)],\qquad \widehat{\Gamma}_{c,u,I,v}(f):=\frac1n\sum_{i=1}^n h_I(X_i,Y_i).
\label{eq:gamma-as-empirical-process}
\end{equation}

\paragraph{Step 2: reducing intervals to half-lines.}
Let
$\mathcal H:=\{\varnothing,\R\}\cup\{(-\infty,t),(-\infty,t]:t\in\R\}$
be the class of left half-lines. Every real interval $I$ can be represented, for suitable $H_1,H_0\in\mathcal H$, as
\begin{equation}
\mathbf{1}\{s\in I\}=\mathbf{1}\{s\in H_1\}-\mathbf{1}\{s\in H_0\}.
\label{eq:interval-as-two-halflines}
\end{equation}
For example, if $I=(a,b]$, take $H_1=(-\infty,b]$ and $H_0=(-\infty,a]$. Other endpoint conventions are handled by choosing $<$ or $\le$, and unbounded intervals and the empty set are covered using $\R$ or $\varnothing$.

Fix the sample and write $A_i:=A(X_i,Y_i)$ and $s_i:=s(X_i)$. Introduce independent Rademacher signs $\xi_1,\ldots,\xi_n$. By~Eq.~\eqref{eq:interval-as-two-halflines} and the triangle inequality,
\begin{align}
&\sup_{I\in\gI}\left|\frac1n\sum_{i=1}^n\xi_iA_i\mathbf{1}\{s_i\in I\}\right|\notag\\
&\quad\le2\sup_{H\in\mathcal H}\left|\frac1n\sum_{i=1}^n\xi_iA_i\mathbf{1}\{s_i\in H\}\right|.
\label{eq:interval-rademacher-to-halfline}
\end{align}
Fix a permutation ordering the scores nondecreasingly, and write the reordered values as $s_{(1)}\le\cdots\le s_{(n)}$, with corresponding $A_{(i)}$ and $\xi_{(i)}$. A half-line never splits a tied score group, and the selected indices form a prefix $\{1,\ldots,k\}$. Therefore
\begin{equation}
\sup_{H\in\mathcal H}\left|\sum_{i=1}^n\xi_iA_i\mathbf{1}\{s_i\in H\}\right|\le\max_{0\le k\le n}\left|\sum_{i=1}^k\xi_{(i)}A_{(i)}\right|.
\label{eq:halfline-prefix-bound}
\end{equation}
Let $Z_0^\xi:=0$, $Z_k^\xi:=\sum_{i=1}^k\xi_{(i)}A_{(i)}$, and $\mathcal F_k^\xi:=\sigma(\xi_{(1)},\ldots,\xi_{(k)})$. Conditional on the sample, $A_{(1)},\ldots,A_{(n)}$ are constants, while $\xi_{(k+1)}$ is independent of $\mathcal F_k^\xi$ and has mean zero. Hence
\[
\mathbb{E}_\xi[Z_{k+1}^\xi\mid\mathcal F_k^\xi]
=Z_k^\xi+A_{(k+1)}\mathbb{E}_\xi[\xi_{(k+1)}\mid\mathcal F_k^\xi]
=Z_k^\xi.
\]
Moreover, $|A_{(i)}|\le2$ makes every $Z_k^\xi$ square-integrable. Thus $(Z_k^\xi)_{k=0}^n$ is a square-integrable martingale with respect to $(\mathcal F_k^\xi)_{k=0}^n$. Doob's $L^2$ maximal inequality gives
\begin{equation}
\mathbb{E}_\xi\left[\max_{0\le k\le n}|Z_k^\xi|^2\right]\le4\mathbb{E}_\xi[(Z_n^\xi)^2].
\label{eq:doob-maximal-inequality}
\end{equation}
Expanding the right-hand side and using independence, $\mathbb{E}_\xi[\xi_{(i)}]=0$, and $\mathbb{E}_\xi[\xi_{(i)}^2]=1$ yields
\[
\mathbb{E}_\xi[(Z_n^\xi)^2]
=\mathbb{E}_\xi\!\left[\left\{\sum_{i=1}^n\xi_{(i)}A_{(i)}\right\}^2\right]
=\sum_{i=1}^nA_{(i)}^2.
\]
By~Eq.~\eqref{eq:doob-maximal-inequality}, Jensen's inequality, and $|A_{(i)}|\le2$,
\[
\mathbb{E}_\xi\!\left[\max_{0\le k\le n}|Z_k^\xi|\right]
\le2\left(\sum_{i=1}^nA_{(i)}^2\right)^{1/2}
\le4\sqrt n.
\]
Combining~Eq.~\eqref{eq:interval-rademacher-to-halfline} and~Eq.~\eqref{eq:halfline-prefix-bound},
\begin{equation}
\mathbb{E}_\xi\left[\sup_{I\in\gI}\left|\frac1n\sum_{i=1}^n\xi_iA_i\mathbf{1}\{s_i\in I\}\right|\right]\le\frac8{\sqrt n}.
\label{eq:interval-rademacher-bound}
\end{equation}

\paragraph{Step 3: uniform deviation for one triple $(c,u,v)$.}
We derive the required high-probability bound from~Eq.~\eqref{eq:interval-rademacher-bound}. Let $(X_i',Y_i')_{i=1}^n$ be an independent copy of the sample. Conditional Jensen's inequality yields
\begin{align*}
&\mathbb{E}\left[\sup_{I\in\gI}\left|\frac1n\sum_{i=1}^nh_I(X_i,Y_i)-\mathbb{E}[h_I(X,Y)]\right|\right]\\
&=\mathbb{E}\left[\sup_{I\in\gI}\left|\frac1n\sum_{i=1}^nh_I(X_i,Y_i)-\mathbb{E}\left[\frac1n\sum_{i=1}^nh_I(X_i',Y_i')\middle|(X_i,Y_i)_{i=1}^n\right]\right|\right]\\
&\le\mathbb{E}\left[\sup_{I\in\gI}\left|\frac1n\sum_{i=1}^n\{h_I(X_i,Y_i)-h_I(X_i',Y_i')\}\right|\right].
\end{align*}
By exchangeability of each pair $((X_i,Y_i),(X_i',Y_i'))$ and independent Rademacher signs,
\begin{align*}
&\mathbb{E}\left[\sup_{I\in\gI}\left|\frac1n\sum_{i=1}^n\{h_I(X_i,Y_i)-h_I(X_i',Y_i')\}\right|\right]\\
&=\mathbb{E}\left[\sup_{I\in\gI}\left|\frac1n\sum_{i=1}^n\xi_i\{h_I(X_i,Y_i)-h_I(X_i',Y_i')\}\right|\right]\\
&\le\mathbb{E}\left[\sup_{I\in\gI}\left|\frac1n\sum_{i=1}^n\xi_i h_I(X_i,Y_i)\right|\right]
+\mathbb{E}\left[\sup_{I\in\gI}\left|\frac1n\sum_{i=1}^n\xi_i h_I(X_i',Y_i')\right|\right]\\
&\le\frac{16}{\sqrt n},
\end{align*}
where the last inequality applies~Eq.~\eqref{eq:interval-rademacher-bound} twice. Consequently,
\begin{equation}
\mathbb{E}\left[\sup_{I\in\gI}\left|\frac1n\sum_{i=1}^nh_I(X_i,Y_i)-\mathbb{E}[h_I(X,Y)]\right|\right]\le\frac{16}{\sqrt n}.
\label{eq:expected-interval-deviation}
\end{equation}
Set
\[
\Phi((X_i,Y_i)_{i=1}^n):=\sup_{I\in\gI}\left|\frac1n\sum_{i=1}^nh_I(X_i,Y_i)-\mathbb{E}[h_I(X,Y)]\right|.
\]
Replacing one observation changes each empirical average by at most $4/n$, because $h_I$ takes values in $[-2,2]$. Hence it changes $\Phi$ by at most $4/n$. McDiarmid's inequality implies, for every $t>0$,
\begin{equation}
\mathbb{P}(\Phi-\mathbb{E}[\Phi]\ge t)\le\exp\left(-\frac{nt^2}{8}\right).
\label{eq:mcdiarmid-interval-deviation}
\end{equation}
Taking $t=4\sqrt{\log(1/\delta')/(2n)}$ yields, with probability at least $1-\delta'$,
\begin{equation}
\sup_{I\in\gI}|\widehat{\Gamma}_{c,u,I,v}(f)-\Gamma_{c,u,I,v}(f)|\le\frac{16}{\sqrt n}+4\sqrt{\frac{\log(1/\delta')}{2n}}.
\label{eq:one-triple-uniform}
\end{equation}

\paragraph{Step 4: simultaneous control for all triples $(c,u,v)$.}
If $N_{\mathrm{cand}}=0$, both empirical and population HUC are zero and the conclusion is immediate. Otherwise set $\delta':=\delta/N_{\mathrm{cand}}$ in~Eq.~\eqref{eq:one-triple-uniform}. A union bound gives, with probability at least $1-\delta$, simultaneously for all triples,
\begin{equation}
\sup_{I\in\gI}|\widehat{\Gamma}_{c,u,I,v}(f)-\Gamma_{c,u,I,v}(f)|\le\frac{16}{\sqrt n}+4\sqrt{\frac{\log(N_{\mathrm{cand}}/\delta)}{2n}}.
\label{eq:all-triples-uniform}
\end{equation}
Since $N_{\mathrm{cand}}\ge1$ in this case, $\max(1,N_{\mathrm{cand}})=N_{\mathrm{cand}}$.

Finally, for any bounded real families $(x_\lambda)_\lambda$ and $(y_\lambda)_\lambda$,
\begin{equation}
\left|\sup_\lambda|x_\lambda|-\sup_\lambda|y_\lambda|\right|\le\sup_\lambda|x_\lambda-y_\lambda|.
\label{eq:difference-of-suprema}
\end{equation}
Indeed, the triangle inequality gives $|x_\lambda|\le|y_\lambda|+|x_\lambda-y_\lambda|$, hence
$\sup_\lambda|x_\lambda|-\sup_\lambda|y_\lambda|\le\sup_\lambda|x_\lambda-y_\lambda|$.
Interchanging $x$ and $y$ gives the reverse direction. Applying~Eq.~\eqref{eq:difference-of-suprema} with $x_\lambda=\widehat{\Gamma}_{c,u,I,v}(f)$ and $y_\lambda=\Gamma_{c,u,I,v}(f)$, and then using~Eq.~\eqref{eq:all-triples-uniform}, proves~Eq.~\eqref{eq:finite-sample-main}.
\end{proof}

\subsection{Reduction to UC with At Most One Relevant Node}
\label{subsec:finite-sample-uc-reduction}
Assume $|\mathcal V_{\gT}(u)|\le1$ for every $u\in\gU$. If $\mathcal V_{\gT}(u)=\{v_u\}$, the tree decomposition contains only the contribution of $v_u$; if the relevant-node set is empty, the utility residual is identically zero. Therefore $\UC(f;\gC,\gU)=\HUC(f;\gC,\gU)$. Likewise, $\widehat{\UC}(f;\gC,\gU)=\widehat{\HUC}(f;\gC,\gU)$.
Here $\widehat{\UC}$ denotes UC with the expectation replaced by its empirical average. The weight $A(X,Y)$ in the proof above then equals $c(X)\{u(P,Y)-\mu_u(P)\}$. The half-line reduction, score-ordered Rademacher partial sums, Doob's inequality, symmetrization, and concentration specialize without change to the UC proof. A flat tree recovers multiclass UC, and a two-leaf tree recovers binary UC.

\subsection{Exact Empirical Maximization over All Intervals}
\label{subsec:exact-interval-scan}\label{subsec:huc-audit-computation}
On a fixed sample of size $n\ge1$, fix $(c,u,v)$ and set
\begin{equation}
w_i:=c(X_i)\left\langle \boldsymbol\Delta_{u,v}(P_i),\boldsymbol R_v(P_i,Y_i)\right\rangle,\qquad s_i:=\mu_u(P_i).
\label{eq:interval-scan-values}
\end{equation}
Then the empirical moment in~Eq.~\eqref{eq:gamma-as-empirical-process} is
$\widehat\Gamma_{c,u,I,v}(f)=\frac1n\sum_{i=1}^n w_i\mathbf1\{s_i\in I\}$.
We derive an exact maximization over $I\in\gI$ and give pseudocode that returns a maximizing interval together with its signed empirical moment.

\paragraph{From intervals to tied score blocks.}
Sort the observations by $s_i$. Let $m_{\mathrm{sc}}\le n$ be the number of distinct scores, ordered as $s_{(1)}<\cdots<s_{(m_{\mathrm{sc}})}$, and define
\begin{equation}
G_r:=\{i:s_i=s_{(r)}\},\qquad B_r:=\frac1n\sum_{i\in G_r}w_i.
\label{eq:tie-groups}
\end{equation}
All observations in $G_r$ have the same interval membership, so their contributions must be grouped before maximization.

If an interval contains $s_{(a)}$ and $s_{(b)}$, it also contains every score between them. Thus its selected distinct scores form either the empty set or a contiguous block $s_{(a)},\ldots,s_{(b)}$. In the latter case,
\[
\widehat\Gamma_{c,u,I,v}(f)
=\sum_{r:s_{(r)}\in I}\frac1n\sum_{i\in G_r}w_i
=\sum_{r=a}^{b}B_r.
\]
Conversely, every such block is selected by the closed interval $[s_{(a)},s_{(b)}]$, including a singleton when $a=b$. The empty interval gives zero. Therefore
\begin{equation}
\sup_{I\in\gI}|\widehat\Gamma_{c,u,I,v}(f)|
=\max\left(0,\max_{1\le a\le b\le m_{\mathrm{sc}}}\left|\sum_{k=a}^{b}B_k\right|\right).
\label{eq:exact-block-maximum}
\end{equation}
This reduction preserves every possible interval membership set on the sample, including ties.

\paragraph{Prefix-sum representation.}
Set $C_0:=0$ and $C_b:=\sum_{k=1}^{b}B_k$. For $1\le a\le b\le m_{\mathrm{sc}}$,
$\sum_{k=a}^{b}B_k =\sum_{k=1}^{b}B_k-\sum_{k=1}^{a-1}B_k =C_b-C_{a-1}$.
Writing $r=a-1$ gives $0\le r<b$. In particular, $r$ is the position immediately before the selected block, and
$\widehat\Gamma_{c,u,[s_{(r+1)},s_{(b)}],v}(f)=C_b-C_r$.
Since $m_{\mathrm{sc}}\ge1$ and absolute values are nonnegative, the zero in~Eq.~\eqref{eq:exact-block-maximum} may be omitted when computing the maximum value. Hence
\begin{equation}
\sup_{I\in\gI}|\widehat\Gamma_{c,u,I,v}(f)|
=\max_{1\le b\le m_{\mathrm{sc}}}\max_{0\le r<b}|C_b-C_r|.
\label{eq:exact-prefix-maximum}
\end{equation}
The outer maximum selects the right endpoint; the inner maximum selects the left endpoint for that fixed right endpoint.

\paragraph{Two candidates for each right endpoint.}
For fixed $b$, the value $C_b$ does not depend on $r$. Using $|t|=\max(t,-t)$ gives
\begin{equation}
\begin{aligned}
\max_{0\le r<b}|C_b-C_r|
&=\max_{0\le r<b}\max(C_b-C_r,\ C_r-C_b)\\
&=\max\left(\max_{0\le r<b}(C_b-C_r),\ \max_{0\le r<b}(C_r-C_b)\right)\\
&=\max\left(C_b-\min_{0\le r<b}C_r,\ \max_{0\le r<b}C_r-C_b\right).
\end{aligned}
\label{eq:interval-prefix-extrema}
\end{equation}
The first term maximizes the signed interval sum by subtracting the smallest previous prefix sum. The second maximizes its negative by subtracting $C_b$ from the largest previous prefix sum. Thus only these two previous prefix sums, rather than all possible left endpoints, are needed for each $b$.

\paragraph{Pseudocode and correspondence with the formulas.}
Algorithm~\ref{alg:exact-interval-scan} implements~Eq.~\eqref{eq:interval-prefix-extrema}. The indices $r_{\min}$ and $r_{\max}$ store positions attaining the smallest and largest previous prefix sums. The variables $I_{\mathrm{best}}$ and $\widehat\Gamma_{\mathrm{best}}$ store the best interval found so far and its signed empirical moment.

\begin{algorithm}[!htb]
\caption{Exact tie-preserving interval maximization for a fixed $(c,u,v)$}
\label{alg:exact-interval-scan}
\begin{algorithmic}[1]
\Statex \textbf{Input:} $(s_i,w_i)_{i=1}^n$ from~Eq.~\eqref{eq:interval-scan-values}, with $n\ge1$
\Statex \textbf{Output:} A maximizing interval, its signed moment, and its absolute value
\State Sort the pairs $(s_i,w_i)$ by nondecreasing $s_i$ \label{line:interval-sort}
\State Scan the sorted pairs once to form $s_{(r)},G_r,B_r$ in~Eq.~\eqref{eq:tie-groups} \label{line:interval-group}
\State $C_0\gets0$; $r_{\min}\gets0$; $r_{\max}\gets0$
\State $I_{\mathrm{best}}\gets\varnothing$; $\widehat\Gamma_{\mathrm{best}}\gets0$
\For{$b=1,\ldots,m_{\mathrm{sc}}$}
    \State $C_b\gets C_{b-1}+B_b$ \label{line:interval-prefix}
    \If{$C_b-C_{r_{\min}}> |\widehat\Gamma_{\mathrm{best}}|$} \label{line:interval-positive}
        \State $\widehat\Gamma_{\mathrm{best}}\gets C_b-C_{r_{\min}}$
        \State $I_{\mathrm{best}}\gets[s_{(r_{\min}+1)},s_{(b)}]$
    \EndIf
    \If{$C_{r_{\max}}-C_b> |\widehat\Gamma_{\mathrm{best}}|$} \label{line:interval-negative}
        \State $\widehat\Gamma_{\mathrm{best}}\gets C_b-C_{r_{\max}}$ \label{line:interval-negative-signed}
        \State $I_{\mathrm{best}}\gets[s_{(r_{\max}+1)},s_{(b)}]$
    \EndIf
    \If{$C_b<C_{r_{\min}}$} \label{line:interval-minimum}
        \State $r_{\min}\gets b$
    \EndIf
    \If{$C_b>C_{r_{\max}}$} \label{line:interval-maximum}
        \State $r_{\max}\gets b$
    \EndIf
\EndFor
\State \Return $I_{\mathrm{best}},\widehat\Gamma_{\mathrm{best}},|\widehat\Gamma_{\mathrm{best}}|$
\end{algorithmic}
\end{algorithm}

Before the two candidate comparisons at right endpoint $b$, the stored indices satisfy
\[
C_{r_{\min}}=\min_{0\le r<b}C_r,
\qquad
C_{r_{\max}}=\max_{0\le r<b}C_r.
\]
This holds initially because the only previous prefix is $C_0=0$.
Lines~\ref{line:interval-positive} and~\ref{line:interval-negative} compare the two terms of~Eq.~\eqref{eq:interval-prefix-extrema} with the best absolute value so far. In the second comparison, $C_{r_{\max}}-C_b$ is the magnitude of a negative candidate when that candidate improves the best value; line~\ref{line:interval-negative-signed} therefore stores $C_b-C_{r_{\max}}$, retaining its actual sign. The corresponding interval starts at the group after the subtracted prefix, explaining the index $r_{\min}+1$ or $r_{\max}+1$.

Only after these comparisons do lines~\ref{line:interval-minimum} and~\ref{line:interval-maximum} include $C_b$ among the possible previous prefixes. Using the old indices, the updated extreme values are
$\min_{0\le r\le b}C_r=\min(C_{r_{\min}},C_b)$, $\max_{0\le r\le b}C_r=\max(C_{r_{\max}},C_b)$.
The invariant therefore holds for the next right endpoint $b+1$. Updating the extrema after the comparisons ensures that the current candidates always satisfy $r<b$.

After processing right endpoint $b$, the saved absolute value is
\[
|\widehat\Gamma_{\mathrm{best}}|
=\max\left(0,\max_{1\le j\le b}\max_{0\le r<j}|C_j-C_r|\right),
\]
and its saved signed value equals $\widehat\Gamma_{c,u,I_{\mathrm{best}},v}(f)$. At $b=m_{\mathrm{sc}}$, Eq.~\eqref{eq:exact-prefix-maximum} proves that the returned interval is a global empirical maximizer. Strict comparisons retain the first encountered candidate at ties; if the maximum is zero, the algorithm returns the empty interval and zero. Grouping in line~\ref{line:interval-group} ensures that equal scores are never separated.

\paragraph{Operation count.}
Sorting the $n$ pairs costs $O(n\log n)$. Line~\ref{line:interval-group} groups tied scores in $O(n)$, and the prefix-extrema scan performs $O(1)$ work for each of the $m_{\mathrm{sc}}\le n$ distinct scores. Thus Algorithm~\ref{alg:exact-interval-scan} costs $O(n\log n)$ from unsorted inputs and $O(n)$ once the score order is available, without enumerating the $m_{\mathrm{sc}}(m_{\mathrm{sc}}+1)/2$ nonempty blocks. Input preparation and reuse of the score order are treated in Appendix~\ref{subsec:uc-huc-computational-cost}.

\subsection{Computational Comparison of Exact UC and HUC Auditing}
\label{subsec:uc-huc-computational-cost}

We compare arithmetic operation counts for exact empirical auditing on the same hierarchical dataset. Model inference is excluded, and evaluating one value $u(P_i,y)$ or one subgroup weight $c(X_i)$ is treated as constant time. Let $n$ be the audit-sample size, $K=|\gY|$ the number of leaves, $d_v$ the number of children of internal node $v$, and
\begin{equation}
|E_{\gT}|=\sum_{v\in V^\circ}d_v=|V|-1
\label{eq:tree-edge-count-computation}
\end{equation}
be the number of tree edges. For $u\in\gU$, define the maximum number of relevant nodes on a root-to-leaf path by
\begin{equation}
D_u:=\max_{y\in\gY}|\operatorname{Path}(y)\cap\mathcal V_{\gT}(u)|,
\qquad
D:=\max_{u\in\gU}D_u.
\label{eq:relevant-path-depth-computation}
\end{equation}
We divide the computation into three stages: \textbf{1) precomputation of the required statistics, 2) sorting the predicted utility scores, and 3) maximization over all intervals.} We derive the cost of each stage for UC and HUC below. All interval maxima are computed by the exact tie-preserving scan of Appendix~\ref{subsec:exact-interval-scan}.

\subsubsection*{Computational complexity of UC}

\paragraph{1) Precomputation of the required statistics.}
For each observation $i$ and utility $u$, UC uses the predicted utility score and realized utility residual
\begin{equation}
s_{i,u}:=\mu_u(P_i)=\sum_{y\in\gY}P_{i,y}u(P_i,y),
\qquad
a_{i,u}:=u(P_i,Y_i)-s_{i,u}.
\label{eq:uc-computation-score-residual}
\end{equation}
For a general dense utility, computing the score requires evaluating the $K$ leaf utility values and their weighted sum, at a cost of $O(K)$ per $(i,u)$. The residual is then obtained in constant time from these values. Thus precomputation over all observations and utilities costs
$O(n|\gU|K)$.

\paragraph{2) Sorting the predicted utility scores.}
For each utility $u$, sort the $n$ scores $s_{1,u},\ldots,s_{n,u}$ once. This costs $O(n\log n)$ per utility and $O(n|\gU|\log n)$ overall. Because $s_{i,u}=\mu_u(P_i)$ does not depend on $c$, the order is shared across all $c\in\gC$.

\paragraph{3) Maximization over all intervals.}
For fixed $(c,u)$, run Algorithm~\ref{alg:exact-interval-scan} on $w_i=c(X_i)a_{i,u}$ in the shared score order. The tie grouping and scan cost $O(n)$ per $(c,u)$, hence $O(n|\gC||\gU|)$ over all subgroup weights and utilities.

Combining the three stages gives
\begin{equation}
T_{\mathrm{UC}}
=
O\!\left(
 {n|\gU|K}
 +{n|\gU|\log n}
 +{n|\gC||\gU|}
\right).
\label{eq:uc-audit-computational-cost}
\end{equation}
If the scores and realized utility residuals in~Eq.~\eqref{eq:uc-computation-score-residual} are already available, the cost of Step~1 is omitted.

\subsubsection*{Computational complexity of HUC}

\paragraph{1) Precomputation of the required statistics.}
HUC uses the same leaf-level utility values as UC, together with the predicted mass and mean utility of each subtree. For each $(i,u)$, initialize the leaf values by
$\pi_y(P_i)=P_{i,y}$, $N_{u,y}(P_i)=P_{i,y}u(P_i,y)$.
Then traverse the tree from the leaves to the root, using the recursions in Lemma~\ref{lem:subtree-recursions}:
\[
\pi_v(P_i)=\sum_{j=1}^{d_v}\pi_{v_j}(P_i),
\qquad
N_{u,v}(P_i)=\sum_{j=1}^{d_v}N_{u,v_j}(P_i).
\]
These values give the subtree means, branch probabilities, and local utility contrasts:
$\mu_{u,v}(P_i)=\frac{N_{u,v}(P_i)}{\pi_v(P_i)}$, $q_{v,j}(P_i)=\frac{\pi_{v_j}(P_i)}{\pi_v(P_i)}$, $\Delta_{u,v,j}(P_i)=\mu_{u,v_j}(P_i)-\mu_{u,v}(P_i)$.
At each internal node, we reuse the already computed child aggregates rather than summing its descendant-leaf values from scratch. The work at node $v$ is $O(d_v)$, so the total tree computation costs
$O\!\left(\sum_{v\in V^\circ}d_v\right)=O(|E_{\gT}|)$.
Including the leaf utility evaluation, precomputation costs $O(K+|E_{\gT}|)$ per $(i,u)$ and
$O\!\left(n|\gU|(K+|E_{\gT}|)\right)$
over all observations and utilities.

\paragraph{2) Sorting the predicted utility scores.}
HUC uses the same score $s_{i,u}=\mu_u(P_i)$ as UC, so one sort per utility costs $O(n|\gU|\log n)$. The order depends on neither $c$ nor $v$ and is therefore shared across all subgroup weights and relevant nodes.

\paragraph{3) Maximization over all intervals.}
For fixed $(c,u)$, make one pass in this shared order while maintaining the Algorithm~\ref{alg:exact-interval-scan} state for the relevant nodes. By Lemma~\ref{lem:one-step-increment}, observation $i$ can contribute only at nodes in $\operatorname{Path}(Y_i)\cap\mathcal V_{\gT}(u)$, of which there are at most $D_u$; each such contribution is available in constant time from the subtree means computed in Step~1.

Equal-score observations are accumulated as one block, and the node-wise prefix states are advanced only after the block is complete, so tied scores are never split. Nodes not touched by the block retain the same state, while the global maximum is updated from the touched nodes. The scan therefore processes only pairs consisting of an observation and a relevant node on its realized path. Their total number is
\[
\sum_{i=1}^n
\bigl|\operatorname{Path}(Y_i)\cap\mathcal V_{\gT}(u)\bigr|
\le nD_u.
\]
Adding each contribution and performing each node-wise update at the end of a group take constant time, and the number of the latter updates is no larger than the number of these pairs. Exact interval maximization thus costs $O(nD_u)$ per $(c,u)$ and
$O\!\left(n|\gC|\sum_{u\in\gU}D_u\right)$
over all subgroup weights and utilities. Nodes not reached by any observation have empirical moment zero and need not be scanned.

Combining the three stages gives
\begin{equation}
T_{\mathrm{HUC}}
=
O\!\left(
n|\gU|(K+|E_{\gT}|)
 +n|\gU|\log n
 +n|\gC|\sum_{u\in\gU}D_u
\right).
\label{eq:huc-audit-computational-cost-general}
\end{equation}

\paragraph{Contracting unary nodes.}
A node with only one child has identically zero local contribution, so contracting such nodes leaves the audit value unchanged. After contraction every internal node has at least two children. Writing $M:=|V^\circ|$, we have
$2M\le\sum_{v\in V^\circ}d_v=|E_{\gT}|=K+M-1$.
Therefore
$M\le K-1$, $|E_{\gT}|\le2K-2$,
and the tree precomputation in Step~1 is also linear in $K$. Thus~Eq.~\eqref{eq:huc-audit-computational-cost-general} becomes
\begin{equation}
\begin{aligned}
T_{\mathrm{HUC}}
&=
O\!\left(
 n|\gU|K
 +n|\gU|\log n
 +n|\gC|\sum_{u\in\gU}D_u
\right)\\
&=
O\!\left(
 n|\gU|K
 +n|\gU|\log n
 +n|\gC||\gU|D
\right),
\end{aligned}
\label{eq:huc-audit-computational-cost-reduced}
\end{equation}
where the last bound uses $\sum_{u\in\gU}D_u\le|\gU|D$.

\paragraph{Comparison.}
Table~\ref{tab:uc-huc-computational-comparison} summarizes the three stages.
UC does not traverse internal nodes: its $O(K)$ term comes from computing the predicted mean of a dense utility over the leaves. HUC uses the same leaf-level quantities and adds a tree traversal that is also linear in $K$ on a reduced tree, together with path-sparse interval accumulation. Consequently, for a fixed finite subgroup class, both exact audits have leading order
\begin{equation}
O\!\left(n|\gU|K+n|\gU|\log n\right).
\label{eq:uc-huc-same-leading-order}
\end{equation}
They are not identical in every computational regime. If the predicted utility scores and all local subtree quantities have already been computed, the remaining scan costs are
\begin{equation}
T_{\mathrm{UC}}^{\mathrm{scan}}
=O\!\left(n|\gU|\log n+n|\gC||\gU|\right),
\qquad
T_{\mathrm{HUC}}^{\mathrm{scan}}
=O\!\left(n|\gU|\log n+n|\gC|\sum_{u\in\gU}D_u\right).
\label{eq:uc-huc-precomputed-scan-cost}
\end{equation}
Thus, the additional cost of localization is controlled by the number of relevant nodes on the realized root-to-leaf paths, rather than by enumerating all intervals separately for every node.

\begin{table}[htbp]
\caption{Stage-wise operation counts for exact empirical UC and HUC auditing on the same sample, for general dense utilities. Model inference is excluded; each total is the sum of the three stages.}
\label{tab:uc-huc-computational-comparison}
\centering
\begin{tabularx}{\textwidth}{@{}>{\raggedright\arraybackslash}Xcc@{}}
\toprule
Computational stage & UC & HUC \\
\midrule
1) Precomputation of the required statistics
& $O(n|\gU|K)$
& $O\!\left(n|\gU|(K+|E_{\gT}|)\right)$ \\
2) Sorting the predicted utility scores
& $O(n|\gU|\log n)$
& $O(n|\gU|\log n)$ \\
3) Maximization over all intervals
& $O(n|\gC||\gU|)$
& $O\!\left(n|\gC|\sum_{u\in\gU}D_u\right)$ \\
\bottomrule
\end{tabularx}
\end{table}

\paragraph{Sparse utilities and working storage.}
Let $s_u:=|\mathcal V_{\gT}(u)|$.  Eq.~\eqref{eq:huc-audit-computational-cost-general} is an end-to-end bound for a general dense utility, for which the leaf values must first be aggregated through the tree.  If a utility is supplied directly by a cut-level representation or its required subtree means have already been cached, the $K+|E_{\gT}|$ preprocessing term can be replaced by the size of that active representation.  Processing one utility at a time requires $O(n+|\gC|s_u)$ working scalar storage after the score order has been formed; retaining all utilities simultaneously gives $O(n|\gU|+|\gC|\sum_us_u)$.  These storage choices do not change the operation counts above.

\section{Post-processing and HUC}
\label{app:postprocessing-huc}
\label{subsec:decision-guarantee}\label{subsec:hierarchical-decision-proof}

The post-processing guarantee for UC bounds the risk improvement achievable by a nondecreasing transformation of predicted mean utility~\citep{Hegazy2025,Rossellini2025}. We apply the same comparison to hierarchically aggregated utilities derived from the original target utility. Specifically, we stop refinement at selected positions in the label tree and average the utility over the leaves below each stopping position. Small HUC for the original utility then controls, for each such aggregation, the excess risk over the best monotone post-processing of the common predicted-utility score.

\paragraph{Branch selection and hierarchical utility aggregation.}
Let $\mathcal S\subseteq V^\circ$ be the set of internal nodes at which branching continues. Whenever a node belongs to $\mathcal S$, all its strict internal ancestors must also belong to $\mathcal S$. Thus choosing $\mathcal S$ determines where the tree is truncated for each leaf. Follow the path from the root to a leaf $y$: continue to the child on that path while the current node belongs to $\mathcal S$, and stop at the first node outside $\mathcal S$. Denote this stopping node by $b_{\mathcal S}(y)$.

If the path contains any node in $\mathcal S$, its boundary $b_{\mathcal S}(y)$ is the child toward $y$ of the last such node on that path. The boundary itself is not in $\mathcal S$. When $\mathcal S=\varnothing$, the procedure stops at the root; when $\mathcal S=V^\circ$, it reaches every original leaf. Different paths may stop at different depths.

Below the boundary, individual leaves are no longer distinguished, and the original utility is averaged using the predictive distribution within that subtree. Define the resulting \emph{hierarchically aggregated utility} by
\begin{equation}
u_{\mathcal S}(p,y)
:=\mu_{u,b_{\mathcal S}(y)}(p)
=\frac{\sum_{z\in\gY_{b_{\mathcal S}(y)}}p_z u(p,z)}{\pi_{b_{\mathcal S}(y)}(p)}.
\label{eq:internal-subtree-boundary-value}
\end{equation}
Leaves with the same boundary receive the same utility value. This is not an arbitrarily specified new payoff: it is determined by the original utility $u$, the prediction $p$, and the branch set $\mathcal S$. In particular,
$u_{\varnothing}(p,y)=\mu_u(p)$, $u_{V^\circ}(p,y)=u(p,y)$.
Using no branches gives the overall predicted mean, whereas retaining every branch recovers the original leaf utility.

\begin{lemma}[Decomposition and predicted mean of the aggregated utility]
\label{lem:internal-subtree-utility}
For every positive prediction vector $p$ and leaf $y$, we have $u_{\mathcal S}(p,y)\in[-1,1]$ and
\begin{align}
u_{\mathcal S}(p,y)-\mu_u(p)
&=\sum_{v\in\mathcal S}\left\langle\boldsymbol\Delta_{u,v}(p),\boldsymbol R_v(p,y)\right\rangle,
\label{eq:internal-subtree-utility}\\
\mu_{u_{\mathcal S}}(p)&=\mu_u(p).
\label{eq:internal-subtree-same-mean}
\end{align}
\end{lemma}

\begin{proof}
If the path to $y$ proceeds from $v$ to $v_j$, Lemma~\ref{lem:one-step-increment} gives
$\left\langle\boldsymbol\Delta_{u,v}(p),\boldsymbol R_v(p,y)\right\rangle =\mu_{u,v_j}(p)-\mu_{u,v}(p)$.
By the definition of the boundary, the nodes where branching occurs before reaching $b_{\mathcal S}(y)$ are exactly the nodes of $\mathcal S$ on the path to $y$. Summing their contributions cancels every intermediate subtree mean, leaving the difference between the boundary and root means. Off-path contributions are zero, so
\[
\sum_{v\in\mathcal S}\left\langle\boldsymbol\Delta_{u,v}(p),\boldsymbol R_v(p,y)\right\rangle =\mu_{u,b_{\mathcal S}(y)}(p)-\mu_{u,\rho}(p) =u_{\mathcal S}(p,y)-\mu_u(p).
\]
The subtree mean is a weighted average of utility values in $[-1,1]$, which also proves boundedness.

Under the predictive distribution $\widetilde Y\sim p$, each residual component satisfies
$\mathbb E_{\widetilde Y\sim p}[R_{v,j}(p,\widetilde Y)] =\pi_{v_j}(p)-q_{v,j}(p)\pi_v(p)=0$.
Taking predictive expectation in~Eq.~\eqref{eq:internal-subtree-utility} therefore gives
\[
\mu_{u_{\mathcal S}}(p)-\mu_u(p)
=\sum_{v\in\mathcal S}\left\langle\boldsymbol\Delta_{u,v}(p),
\mathbb E_{\widetilde Y\sim p}[\boldsymbol R_v(p,\widetilde Y)]\right\rangle=0.\qedhere
\]
\end{proof}
The equality of means is under the predictive distribution $p$; the original and aggregated utilities need not have equal means under the actual label distribution.

\paragraph{Monotone post-processing of the common score.}
For a threshold $t\in[-1,1]$ and a binary rule $\varphi:\gX\to\{0,1\}$, define
\begin{equation}
\mathcal R_{u,\mathcal S}(\varphi;t)
:=\mathbb E\!\left[|u_{\mathcal S}(P,Y)-t|\,
\mathbf1\!\left\{\varphi(X)\ne\mathbf1\{u_{\mathcal S}(P,Y)\ge t\}\right\}\right].
\label{eq:internal-subtree-risk}
\end{equation}
A wrong decision about whether the aggregated utility is at least the threshold incurs its distance from that threshold. The expectation is under the actual joint distribution of $(X,Y)$.

Lemma~\ref{lem:internal-subtree-utility} shows that the predicted mean of the aggregated utility is still $\mu_u(P)$. Accordingly, define the baseline rule and its post-processed counterpart by
\begin{equation}
\varphi^0_{u,t}(X):=\mathbf1\{\mu_u(P)\ge t\},\qquad
\varphi^h_{u,t}(X):=\mathbf1\{h(\mu_u(P))\ge t\},
\label{eq:internal-subtree-rules}
\end{equation}
where $h:[-1,1]\to[-1,1]$ is nondecreasing. The post-processing gap is
\begin{equation}
\operatorname{Gap}^{\mathrm{mono}}_{t,\mathcal S}(f,u)
:=\mathcal R_{u,\mathcal S}(\varphi^0_{u,t};t)
-\inf_{h\ \mathrm{nondecreasing}}\mathcal R_{u,\mathcal S}(\varphi^h_{u,t};t).
\label{eq:internal-subtree-gap}
\end{equation}
Here $f,u,\mathcal S,t$ remain fixed, and only $h$ is varied.

The post-processing operation and the class of decision rules are the same as for UC applied to the original utility. For the same $t,h$, the binary decisions do not depend on $\mathcal S$. What changes is whether those decisions are evaluated using the original leaf utility $u(P,Y)$ or the aggregated utility $u_{\mathcal S}(P,Y)$. Tree truncation specifies the granularity of the evaluation utility, not a replacement of the prediction rule by node-specific threshold decisions.

\paragraph{Post-processing guarantee for aggregated utilities.}
\begin{theorem}[Monotone post-processing guarantee for hierarchical utility aggregation]
\label{thm:decision-guarantee}
For every target utility $u$, threshold $t\in[-1,1]$, and set $\mathcal S\subseteq V^\circ$ satisfying the ancestor condition above,
\begin{equation}
\operatorname{Gap}^{\mathrm{mono}}_{t,\mathcal S}(f,u) \le 2\,\UC(f;\{1\},\{u_{\mathcal S}\}) \le 2|\mathcal S\cap\mathcal V_{\gT}(u)|\,\HUC(f;\{1\},\{u\}).
\label{eq:decision-guarantee}
\end{equation}
\end{theorem}
The first inequality is the UC post-processing guarantee applied to $u_{\mathcal S}$. The second bounds its UC error using the node-wise errors of the original utility $u$. We make these two steps explicit below.

\begin{proof}
Let $Z:=\mu_u(P)$ and $U:=u_{\mathcal S}(P,Y)$. Lemma~\ref{lem:internal-subtree-utility} gives $Z=\mu_{u_{\mathcal S}}(P)$, $U,Z\in[-1,1]$, and
$U-Z=\sum_{v\in\mathcal S}\left\langle\boldsymbol\Delta_{u,v}(P),\boldsymbol R_v(P,Y)\right\rangle$.
For every $I\in\gI$, multiplication by $\mathbf1\{Z\in I\}$ and expectation yield
\begin{equation}
\mathbb E[\mathbf1\{Z\in I\}(U-Z)]
=\sum_{v\in\mathcal S\cap\mathcal V_{\gT}(u)}\Gamma_{1,u,I,v}(f).
\label{eq:internal-subtree-interval-residual}
\end{equation}
We omit nodes whose local utility contrast is identically zero. The triangle inequality and the definition of HUC imply
\begin{equation}
\UC(f;\{1\},\{u_{\mathcal S}\}) =\sup_{I\in\gI}\left|\sum_{v\in\mathcal S\cap\mathcal V_{\gT}(u)}\Gamma_{1,u,I,v}(f)\right| \le |\mathcal S\cap\mathcal V_{\gT}(u)|\,\HUC(f;\{1\},\{u\}),
\label{eq:internal-subtree-by-huc}
\end{equation}
which proves the second inequality.

For the first inequality, every $a\in\{0,1\}$ satisfies the pointwise identity
\begin{equation}
|U-t|\mathbf1\{a\ne\mathbf1\{U\ge t\}\}
=\frac12|U-t|+\frac12(1-2a)(U-t).
\label{eq:threshold-loss-identity}
\end{equation}
For $a=0$, the right-hand side is $U-t$ when $U\ge t$ and zero otherwise; for $a=1$, it is $t-U$ when $U<t$ and zero otherwise. Thus
\[
\mathcal R_{u,\mathcal S}(\varphi;t)
=\frac12\mathbb E|U-t|+\frac12\mathbb E[(1-2\varphi(X))(U-t)].
\]
The first term does not depend on the rule and cancels in the risk difference, giving
\begin{equation}
\mathcal R_{u,\mathcal S}(\varphi^0_{u,t};t) -\mathcal R_{u,\mathcal S}(\varphi^h_{u,t};t) =\mathbb E[(U-t)\{\varphi^h_{u,t}(X)-\varphi^0_{u,t}(X)\}].
\label{eq:risk-difference}
\end{equation}
Let $E_+:=\{Z<t,\ h(Z)\ge t\}$ and $E_-:=\{Z\ge t,\ h(Z)<t\}$ be the events on which the decision changes from 0 to 1 and from 1 to 0, respectively. Since $\varphi^h_{u,t}(X)-\varphi^0_{u,t}(X)=\mathbf1(E_+)-\mathbf1(E_-)$,
\[
\mathcal R_{u,\mathcal S}(\varphi^0_{u,t};t) -\mathcal R_{u,\mathcal S}(\varphi^h_{u,t};t) =\mathbb E[(U-t)\mathbf1(E_+)]+\mathbb E[(t-U)\mathbf1(E_-)].
\]
On $E_+$, $Z<t$ gives $U-t=(U-Z)+(Z-t)\le U-Z$. On $E_-$, $Z\ge t$ gives $t-U=(Z-U)+(t-Z)\le-(U-Z)$. Therefore,
\[
\mathcal R_{u,\mathcal S}(\varphi^0_{u,t};t) -\mathcal R_{u,\mathcal S}(\varphi^h_{u,t};t) \le\left|\mathbb E[(U-Z)\mathbf1(E_+)]\right| +\left|\mathbb E[(U-Z)\mathbf1(E_-)]\right|.
\]
Because $h$ is nondecreasing, each of the score sets $\{z\in[-1,1]:z<t,\ h(z)\ge t\}$ and $\{z\in[-1,1]:z\ge t,\ h(z)<t\}$ is an interval, possibly empty. Each term is consequently at most
$\sup_{I\in\gI}\left|\mathbb E[\mathbf1\{Z\in I\}(U-Z)]\right| =\UC(f;\{1\},\{u_{\mathcal S}\})$.
Taking the supremum of the improvement over all nondecreasing $h$ proves the first inequality. Combining it with~Eq.~\eqref{eq:internal-subtree-by-huc} proves~Eq.~\eqref{eq:decision-guarantee}.
\end{proof}

\paragraph{Subtree-mean calibration and refinement.}
The same decomposition also yields the subgroup-weighted calibration guarantee used in the remark following Definition~\ref{def:huc-main}.

\begin{proposition}[Calibration under fixed tree boundaries]
\label{prop:subtree-mean-calibration-refinement}
For every subgroup class $\gC$, target utility $u$, and set $\mathcal S\subseteq V^\circ$ satisfying the ancestor condition above,
\begin{equation}
\UC(f;\gC,\{u_{\mathcal S}\})
\le |\mathcal S\cap\mathcal V_{\gT}(u)|\,\HUC(f;\gC,\{u\}).
\label{eq:subtree-uc-by-huc-subgroups}
\end{equation}
These bounds hold simultaneously for all such fixed tree boundaries. Moreover, if $v\notin\mathcal S$ is an internal boundary node, meaning that all its strict internal ancestors are in $\mathcal S$, then $\mathcal S^+:=\mathcal S\cup\{v\}$ also satisfies this condition and
\begin{equation}
\begin{aligned}
&\{u_{\mathcal S^+}(p,y)-\mu_{u_{\mathcal S^+}}(p)\}
-\{u_{\mathcal S}(p,y)-\mu_{u_{\mathcal S}}(p)\}\\
&\quad=u_{\mathcal S^+}(p,y)-u_{\mathcal S}(p,y)
=\langle\boldsymbol\Delta_{u,v}(p),\boldsymbol R_v(p,y)\rangle.
\end{aligned}
\label{eq:subtree-refinement-residual-increment}
\end{equation}
Its moment under the same audit conditions is
\begin{equation}
\mathbb E\!\left[c(X)\mathbf1\{\mu_u(P)\in I\}
\{u_{\mathcal S^+}(P,Y)-u_{\mathcal S}(P,Y)\}\right]
=\Gamma_{c,u,I,v}(f).
\label{eq:subtree-refinement-audit-increment}
\end{equation}
\end{proposition}
\begin{proof}
Lemma~\ref{lem:internal-subtree-utility} gives $u_{\mathcal S}(p,y)\in[-1,1]$ and $\mu_{u_{\mathcal S}}(p)=\mu_u(p)$. Since the tree boundary is fixed, $u_{\mathcal S}$ is a measurable utility of the same permitted form as $u$. Its predicted-utility intervals therefore select exactly the same observations. By Eq.~\eqref{eq:internal-subtree-utility}, for every $c\in\gC$ and $I\in\gI$,
\[
\mathbb E[c(X)\mathbf1\{\mu_{u_{\mathcal S}}(P)\in I\} \{u_{\mathcal S}(P,Y)-\mu_{u_{\mathcal S}}(P)\}] =\sum_{v\in\mathcal S\cap\mathcal V_{\gT}(u)}\Gamma_{c,u,I,v}(f).
\]
The finite sum may be interchanged with expectation, and nodes outside $\mathcal V_{\gT}(u)$ contribute zero. Thus its absolute value is at most
$\sum_{v\in\mathcal S\cap\mathcal V_{\gT}(u)}|\Gamma_{c,u,I,v}(f)| \le |\mathcal S\cap\mathcal V_{\gT}(u)|\,\HUC(f;\gC,\{u\})$.
Taking the supremum over $c,I$ proves Eq.~\eqref{eq:subtree-uc-by-huc-subgroups}. The argument applies to every $\mathcal S$ satisfying the ancestor condition with the same right-hand HUC error, proving simultaneous control.

For the refinement, the assumption on $v$ ensures that only its boundary subtree is expanded, replacing its value $\mu_{u,v}(p)$ by $\mu_{u,v_j}(p)$ on each child subtree. Subtracting Eq.~\eqref{eq:internal-subtree-utility} for $\mathcal S$ from the same identity for $\mathcal S^+$ leaves exactly the term indexed by $v$. Both predicted means equal $\mu_u(p)$ by Lemma~\ref{lem:internal-subtree-utility}, proving Eq.~\eqref{eq:subtree-refinement-residual-increment}. Multiplication by $c(X)\mathbf1\{\mu_u(P)\in I\}$ and expectation give Eq.~\eqref{eq:subtree-refinement-audit-increment}.
\end{proof}
Thus a node contribution is precisely the additional calibration residual from one refinement, and its audit moment evaluates the corresponding change in mean utility under the common audit conditions.

\paragraph{Relation to UC.}
For the original utility, the tree decomposition gives
\[
\UC(f;\{1\},\{u\})=\sup_{I\in\gI}\left|\sum_{v\in\mathcal V_{\gT}(u)}\Gamma_{1,u,I,v}(f)\right|.
\]
For $u_{\mathcal S}$, the sum is restricted to nodes in $\mathcal S$, as in Eq.~\eqref{eq:internal-subtree-by-huc}. Both use intervals of the same score $\mu_u(P)$. Positive and negative contributions may cancel in the full sum without making every selected partial sum small. Small UC for the original utility alone therefore need not imply small post-processing gaps for its aggregated utilities. HUC controls the individual contributions and hence provides the guarantee for every fixed aggregation, with the relevant-node factor in Theorem~\ref{thm:decision-guarantee}.

This does not mean that UC cannot evaluate aggregated utilities: each $u_{\mathcal S}$ can be included explicitly in its utility class. The distinction is that the right-hand side of our guarantee uses HUC only for the original utility $u$, while controlling the post-processing gaps for the family of aggregated utilities derived from it. With $\mathcal S=V^\circ$, the utility and the gap reduce to those for the original leaf utility.

\paragraph{Interpretation as excess risk.}
If $\HUC(f;\{1\},\{u\})\le\varepsilon$, Theorem~\ref{thm:decision-guarantee} can be rearranged as
\begin{equation}
\mathcal R_{u,\mathcal S}(\varphi^0_{u,t};t) \le{}\inf_{h\ \mathrm{nondecreasing}} \mathcal R_{u,\mathcal S}(\varphi^h_{u,t};t) +2|\mathcal S\cap\mathcal V_{\gT}(u)|\,\varepsilon.
\label{eq:aggregated-utility-excess-risk}
\end{equation}
Thus, for each aggregated utility, the current threshold rule is close to the best rule obtainable by monotone post-processing of the same score. Its risk is small when both this benchmark risk and the error term are small. Small HUC alone does not guarantee small absolute risk, nor does it imply optimality among arbitrary decision rules.

\section{Empirical HUC-Boost and Fixed-Step Sample Complexity}
\label{app:randomized-initialization}
\label{sec:sample-complexity-appendix}

HUC-Boost is implemented on a finite calibration sample using empirical audit moments. The implementation supports adaptive and fixed node-logit step sizes. The adaptive step size uses the empirical analogue $\widehat\Gamma_t/\widehat\Lambda_t$ of the population step in Algorithm~\ref{alg:huc-boost}; this is the HUC-Boost step used for the main experimental method. The fixed step size uses a prescribed signed magnitude and gives the HUC-Boost (0.05) comparison when that magnitude is $0.05$. Algorithm~\ref{alg:empirical-huc-boost} records both choices. The sample-complexity theorems in this appendix are stated for the fixed-step specialization. We then give the corresponding two-stage construction with an initializer learned from the same sample and finally describe the initialization models used in the experiments.

\subsection{Fixed-Step Sample Complexity of HUC-Boost}
\label{rtf:sec:main-statements}
Let the audit sample be $S=(Z_1,\ldots,Z_n)\sim\mathbb{P}^n$, where $Z_i=(X_i,Y_i)$. The tree, the finite classes $\gC,\gU$, and the functions they contain are fixed independently of $S$. We use the empirical audit moment in~Eq.~\eqref{eq:gamma-as-empirical-process}, the empirical HUC from Section~\ref{sec:huc-definition}, and the empirical log loss defined in~Eq.~\eqref{eq:empirical-log-loss}. Maximization over intervals in empirical HUC is computed by the finite-point procedure in Appendix~\ref{subsec:huc-audit-computation}, and ties are resolved by the fixed order in Appendix~\ref{subsec:tie-breaking-convention}.

The number of audit candidates is $N_{\mathrm{cand}}$ from~Eq.~\eqref{eq:audit-candidate-count}. If $N_{\mathrm{cand}}=0$, HUC is identically zero, so we assume $N_{\mathrm{cand}}\ge1$ below. The following algorithm uses the same one-node logit direction as the population Algorithm~\ref{alg:huc-boost}, but replaces the population moment and curvature by their empirical counterparts.

\begin{algorithm}[htbp]
\caption{Empirical HUC-Boost with adaptive and fixed step sizes}
\label{alg:empirical-huc-boost}
\begin{algorithmic}[1]
\Statex \textbf{Input:} Calibration sample $S$, positive initial predictor $f_0$, fixed tree and classes $\gC,\gU,\gI$, audit threshold $\tau>0$, choice of adaptive or fixed step size, and $\alpha_{\mathrm{fix}}>0$ for the fixed step size
\State $z_{0,v,j}(x)\gets\log q_{v,j}(f_0(x))$ for all $v,j$; $t\gets0$
\While{there exists a candidate satisfying $|\widehat{\Gamma}_{c,u,I,v}(f_t)|>\tau$}
 \State Select any such $a_t=(c_t,u_t,I_t,v_t)$
 \State $\widehat\Gamma_t\gets\widehat{\Gamma}_{c_t,u_t,I_t,v_t}(f_t)$; $\sigma_t\gets\operatorname{sgn}(\widehat\Gamma_t)$
 \State $\boldsymbol h_t(x)\gets c_t(x)\mathbf 1\{\mu_{u_t}(f_t(x))\in I_t\}\boldsymbol\Delta_{u_t,v_t}(f_t(x))$
 \If{the adaptive step size is used}
   \State $\displaystyle \widehat\Lambda_t\gets\frac{1}{4n}\sum_{i=1}^n J_{v_t}(Y_i)\{\max_jh_{t,j}(X_i)-\min_jh_{t,j}(X_i)\}^2,\qquad \alpha_t\gets\widehat\Gamma_t/\widehat\Lambda_t$
 \Else
   \State $\alpha_t\gets\sigma_t\alpha_{\mathrm{fix}}$
 \EndIf
 \State $\boldsymbol z_{t+1,v_t}(x)\gets\boldsymbol z_{t,v_t}(x)+\alpha_t\boldsymbol h_t(x)$ for all $x$; keep all other nodes fixed
 \State Reconstruct $f_{t+1}$ from the softmax branch probabilities and path products; $t\gets t+1$
\EndWhile
\State \Return $\widehat f=f_t$
\end{algorithmic}
\end{algorithm}

When the adaptive step size is used, $|\widehat\Gamma_t|>\tau$ implies $\widehat\Lambda_t>0$ by the empirical analogue of Lemma~\ref{lem:positive-curvature}. The experiments use this step size for HUC-Boost and for the HUC-Boost stage after a parametric recalibrator; HUC-Boost (0.05) uses the fixed step size. Their candidate traversal, stopping tolerance, and update budgets are specified in the experimental appendices.

For the theoretical results below, specialize Algorithm~\ref{alg:empirical-huc-boost} to the fixed step size with $\tau=3\varepsilon/4$ and
\begin{equation}
\alpha_{\mathrm{fix}}=\alpha_\varepsilon:=\frac{\varepsilon}{2},
\qquad 0<\varepsilon\le1.
\label{rtf:eq:fixed-step-size}
\end{equation}

Assume that the positive initial predictor $f_0$ is fixed independently of $S$, and that a deterministic constant $D_{\mathrm{fix}}\ge0$ satisfies
\begin{equation}
\widehat{\mathcal L}_{\log,S}(f_0)\le D_{\mathrm{fix}}
\quad\text{almost surely}.
\label{rtf:eq:fixed-initial-loss}
\end{equation}
For example, this condition holds if $f_0(x)_y\ge e^{-D_{\mathrm{fix}}}$ for every $(x,y)$.

\begin{theorem}[Termination of fixed-step HUC-Boost from a fixed initializer]
\label{thm:fixed-initializer-stopping}
For every fixed sample, Algorithm~\ref{alg:empirical-huc-boost} in the fixed-step specialization above terminates after finitely many updates. If $N_{\mathrm{upd}}$ denotes the number of updates, then
\begin{equation}
N_{\mathrm{upd}}
\le
\left\lceil\frac{4\widehat{\mathcal L}_{\log,S}(f_0)}{\varepsilon^2}\right\rceil,
\qquad
\widehat{\HUC}(\widehat f;\gC,\gU)\le\frac{3\varepsilon}{4}.
\label{rtf:eq:fixed-stopping-realized}
\end{equation}
Under~Eq.~\eqref{rtf:eq:fixed-initial-loss}, the deterministic update bound
\begin{equation}
T_\varepsilon^{\mathrm{fix}}
:=\left\lceil\frac{4D_{\mathrm{fix}}}{\varepsilon^2}\right\rceil
\label{rtf:eq:T-fixed}
\end{equation}
may be used.
\end{theorem}

\begin{theorem}[Population HUC guarantee for fixed-step HUC-Boost from a fixed initializer]
\label{thm:fixed-initializer-sample}
Run Algorithm~\ref{alg:empirical-huc-boost} in the fixed-step specialization above, assume~Eq.~\eqref{rtf:eq:fixed-initial-loss}, and let $0<\delta<1$. If
\begin{equation}
n\ge\frac{512}{\varepsilon^2}
\left[
(T_\varepsilon^{\mathrm{fix}}+1)
\log(2N_{\mathrm{cand}}(2n+1)^2)
+\log\frac1\delta
\right],
\label{rtf:eq:fixed-sample-main}
\end{equation}
then, with probability at least $1-\delta$ over the audit sample,
\begin{equation}
\HUC(\widehat f;\gC,\gU)\le\varepsilon.
\label{rtf:eq:fixed-result-main}
\end{equation}
\end{theorem}

By~Eq.~\eqref{rtf:eq:T-fixed} and $\lceil x\rceil\le x+1$,
$T_\varepsilon^{\mathrm{fix}}+1\le\frac{4D_{\mathrm{fix}}}{\varepsilon^2}+2$.
Hence~Eq.~\eqref{rtf:eq:fixed-sample-main} is implied by
\[
n\ge\frac{512}{\varepsilon^2}
\left[
\left(\frac{4D_{\mathrm{fix}}}{\varepsilon^2}+2\right)
\log(2N_{\mathrm{cand}}(2n+1)^2)
+\log\frac1\delta
\right].
\]
When $D_{\mathrm{fix}}$ and $N_{\mathrm{cand}}$ do not depend on $\varepsilon$, a sufficient condition with no $n$ on the right-hand side is given in Appendix~\ref{rtf:sec:explicit-sample}, especially Corollary~\ref{rtf:cor:explicit-generic}. A sufficient sample-size order is
\[
n=\widetilde O\!\left(\frac{D_{\mathrm{fix}}+1}{\varepsilon^4}+\frac{\log(1/\delta)}{\varepsilon^2}\right).
\]
In particular, for fixed $\delta$, $n=\widetilde O(\varepsilon^{-4})$ is sufficient.

\subsection{Fixed-Step HUC-Boost with a Learned Initializer}
\label{rtf:sec:two-stage-statements}
We now consider a two-stage procedure that learns a finite-dimensional softmax recalibrator from the same sample and then runs Algorithm~\ref{alg:empirical-huc-boost} from its output. The experimental two-stage methods use the adaptive step size. Theorems~\ref{thm:learned-initializer-stopping} and~\ref{thm:learned-initializer-sample} below analyze the fixed step size with $\tau=3\varepsilon/4$ and $\alpha_{\mathrm{fix}}=\alpha_\varepsilon$.

\begin{assumption}[Sample-dependent initial recalibration model]
\label{ass:random-init-main}
Let $f_{0,\theta}(x):=\operatorname{softmax}(s_\theta(x))$, where $\theta\in\R^d$ and $s_\theta(x)\in\R^K$ is finite. Fix constants $R_\theta,L_\theta>0$ and $D_0\ge0$. A measurable initial fitting rule $\mathsf A$ returns $\widehat\theta=\mathsf A(S)$ and satisfies, almost surely,
\begin{equation}
\|\widehat\theta\|_\infty\le R_\theta,
\qquad
\widehat{\mathcal L}_{\log,S}(f_{0,\widehat\theta})\le D_0.
\label{eq:random-fit-main}
\end{equation}
Assume further that, for every $\theta,\theta'\in\R^d$,
\begin{equation}
\sup_x\|s_\theta(x)-s_{\theta'}(x)\|_\infty
\le L_\theta\|\theta-\theta'\|_\infty.
\label{eq:random-lipschitz-main}
\end{equation}
\end{assumption}

Set
\begin{equation}
a_\varepsilon:=\frac{\varepsilon^2}{2L_\theta},
\qquad
\Lambda_\varepsilon
:=d\log\left(1+\frac{R_\theta}{a_\varepsilon}\right)
=d\log\left(1+\frac{2R_\theta L_\theta}{\varepsilon^2}\right).
\label{rtf:eq:learned-scales}
\end{equation}
Draw $\Xi\sim\operatorname{Unif}([-a_\varepsilon,a_\varepsilon]^d)$ independently of $S$, and let $\Theta_0:=\widehat\theta+\Xi$.

\begin{algorithm}[htbp]
\caption{Two-stage empirical HUC-Boost with an initializer learned from the same sample}
\label{alg:randomized-boost}
\begin{algorithmic}[1]
\Statex \textbf{Input:} $S$, initial fitting rule $\mathsf A$, constants in Assumption~\ref{ass:random-init-main}, $0<\varepsilon\le1$, audit threshold $\tau$, choice of adaptive or fixed step size, and $\alpha_{\mathrm{fix}}$ for the fixed step size
\State $\widehat\theta\gets\mathsf A(S)$
\State Draw $\Xi\sim\operatorname{Unif}([-a_\varepsilon,a_\varepsilon]^d)$ independently of $S$
\State $\Theta_0\gets\widehat\theta+\Xi$; $f_0\gets f_{0,\Theta_0}$
\State Run Algorithm~\ref{alg:empirical-huc-boost} from $f_0$ using the same sample $S$, threshold $\tau$, and the selected step-size rule
\State \Return its output $\widehat f$ and $\Theta_0$
\end{algorithmic}
\end{algorithm}

\begin{theorem}[Termination of fixed-step two-stage HUC-Boost from a learned initializer]
\label{thm:learned-initializer-stopping}
Under Assumption~\ref{ass:random-init-main}, Algorithm~\ref{alg:randomized-boost} using the fixed step size with $\tau=3\varepsilon/4$ and $\alpha_{\mathrm{fix}}=\alpha_\varepsilon$ terminates almost surely and satisfies
\begin{equation}
N_{\mathrm{upd}}\le T_\varepsilon^{\mathrm{learn}}
:=\left\lceil\frac{4D_0}{\varepsilon^2}\right\rceil+4,
\qquad
\widehat{\HUC}(\widehat f;\gC,\gU)\le\frac{3\varepsilon}{4}.
\label{rtf:eq:learned-stopping}
\end{equation}
\end{theorem}

\begin{theorem}[Population HUC guarantee for fixed-step two-stage HUC-Boost from a learned initializer]
\label{thm:learned-initializer-sample}
Run Algorithm~\ref{alg:randomized-boost} using the fixed step size with $\tau=3\varepsilon/4$ and $\alpha_{\mathrm{fix}}=\alpha_\varepsilon$ under Assumption~\ref{ass:random-init-main}, and let $0<\delta<1$. If
\begin{equation}
n\ge\frac{512}{\varepsilon^2}
\left[
\Lambda_\varepsilon
+(T_\varepsilon^{\mathrm{learn}}+1)\log(2N_{\mathrm{cand}}(2n+1)^2)
+\log\frac1\delta
\right],
\label{rtf:eq:learned-sample-main}
\end{equation}
then, under the joint distribution of $S$ and $\Xi$,
\begin{equation}
\mathbb{P}_{S,\Xi}(\HUC(\widehat f;\gC,\gU)\le\varepsilon)\ge1-\delta.
\label{rtf:eq:learned-result-main}
\end{equation}
\end{theorem}

Since $T_\varepsilon^{\mathrm{learn}}+1\le4D_0/\varepsilon^2+6$, condition~Eq.~\eqref{rtf:eq:learned-sample-main} is implied by
\[
n\ge\frac{512}{\varepsilon^2}
\left[
d\log\left(1+\frac{2R_\theta L_\theta}{\varepsilon^2}\right)
+\left(\frac{4D_0}{\varepsilon^2}+6\right)\log(2N_{\mathrm{cand}}(2n+1)^2)
+\log\frac1\delta
\right].
\]
When $D_0,R_\theta,L_\theta,d,N_{\mathrm{cand}}$ do not depend on $\varepsilon$, a sufficient condition with no $n$ on the right-hand side is given in Appendix~\ref{rtf:sec:explicit-sample}, especially Corollary~\ref{rtf:cor:explicit-generic}. A sufficient sample-size order is
\[
n=\widetilde O\!\left(\frac{D_0+1}{\varepsilon^4}+\frac{d+\log(1/\delta)}{\varepsilon^2}\right).
\]
Thus, for fixed $\delta$, learning the initial recalibration parameter from the same audit sample retains the same $n=\widetilde O(\varepsilon^{-4})$ dependence on accuracy as the fixed-initializer case. If pretraining uses independent data and its output satisfies Eq.~\eqref{rtf:eq:fixed-initial-loss} with $D_{\mathrm{fix}}$ independent of $\varepsilon$, conditioning on this output and applying Theorem~\ref{thm:fixed-initializer-sample} gives the same $\widetilde O(\varepsilon^{-4})$ sufficient order.

\subsection{Proofs of Theorems H.1 and H.2 for Fixed-Step HUC-Boost}
\label{rtf:sec:fixed-proof}
We first prove the stopping theorem using the one-step empirical log-loss decrease in Appendix~\ref{subsec:huc-boost-theorem-proof}.

\begin{proof}[Proof of Theorem~\ref{thm:fixed-initializer-stopping}]
Whenever the fixed-step specialization of Algorithm~\ref{alg:empirical-huc-boost} updates, $|\widehat{\Gamma}_{c_t,u_t,I_t,v_t}(f_t)|>3\varepsilon/4$. Substituting $\alpha=\alpha_\varepsilon=\varepsilon/2$ into~Eq.~\eqref{eq:empirical-fixed-step-descent} gives
\[
\widehat{\mathcal L}_{\log,S}(f_{t+1})
<\widehat{\mathcal L}_{\log,S}(f_t)-\frac{\varepsilon^2}{4}.
\]
After $t$ updates,
\[
0\le\widehat{\mathcal L}_{\log,S}(f_t)
<\widehat{\mathcal L}_{\log,S}(f_0)-\frac{t\varepsilon^2}{4}.
\]
If updates continued until $t=\lceil4\widehat{\mathcal L}_{\log,S}(f_0)/\varepsilon^2\rceil$, the right-hand side would be nonpositive, a contradiction. Thus the update count satisfies~Eq.~\eqref{rtf:eq:fixed-stopping-realized}. The stopping rule also gives $\widehat{\HUC}(\widehat f;\gC,\gU)\le3\varepsilon/4$.
\end{proof}

Under~Eq.~\eqref{rtf:eq:fixed-initial-loss}, the stopping theorem shows that the output is reached in at most $T_\varepsilon^{\mathrm{fix}}$ updates and has empirical HUC at most $3\varepsilon/4$. It remains to control, uniformly over all predictors reachable in at most $T_\varepsilon^{\mathrm{fix}}$ adaptively selected fixed-step updates, the difference between population and empirical HUC by $\varepsilon/4$.

\begin{definition}[Reachable audit class and uniform generalization gap]
\label{rtf:def:reachable-gap}
Fix a positive initial predictor $f_0$ and an integer $T\ge0$. Let $\mathcal F_{\le T}(f_0)$ be the set of all predictors obtainable from $f_0$ by applying at most $T$ updates of Algorithm~\ref{alg:empirical-huc-boost} in the fixed-step specialization, allowing arbitrary choices of $(c,u,I,v)$ and update signs at each step. Define
\begin{align}
\mathcal G_{\le T}(f_0)
:=\Bigl\{(x,y)\mapsto{}&
\varsigma c(x)\mathbf 1\{\mu_u(f(x))\in I\}
\langle\boldsymbol\Delta_{u,v}(f(x)),\boldsymbol R_v(f(x),y)\rangle:\notag\\[-1mm]
&f\in\mathcal F_{\le T}(f_0),\ c\in\gC,\ u\in\gU,\ I\in\gI,\
v\in\mathcal V_{\gT}(u),\ \varsigma\in\{-1,1\}\Bigr\}
\label{rtf:eq:reachable-audit-class}
\end{align}
and
\begin{equation}
\operatorname{Gap}_T(f_0,S)
:=\sup_{g\in\mathcal G_{\le T}(f_0)}
\left\{\mathbb{E}[g(X,Y)]-\frac1n\sum_{i=1}^n g(X_i,Y_i)\right\}.
\label{rtf:eq:uniform-gap}
\end{equation}
\end{definition}

Because $\varsigma$ is included, every $f\in\mathcal F_{\le T}(f_0)$ satisfies
\begin{equation}
\sup_{\substack{c\in\gC,\,u\in\gU,\,I\in\gI\\v\in\mathcal V_{\gT}(u)}}
|\Gamma_{c,u,I,v}(f)-\widehat{\Gamma}_{c,u,I,v}(f)|
\le\operatorname{Gap}_T(f_0,S).
\label{rtf:eq:moment-by-gap}
\end{equation}
Consequently, by bounding the difference of two suprema by the supremum of the pointwise differences,
\begin{equation}
|\HUC(f;\gC,\gU)-\widehat{\HUC}(f;\gC,\gU)|
\le\operatorname{Gap}_T(f_0,S).
\label{rtf:eq:huc-by-gap}
\end{equation}

The remainder of the proof of Theorem~\ref{thm:fixed-initializer-sample} has three steps. Step 1 counts the value vectors attained by $\mathcal G_{\le T}(f_0)$ on the finite sequence formed by the sample and an independent ghost sample. Step 2 uses this count and Rademacher signs to bound the exponential moment of the uniform generalization gap and derive a high-probability bound. Step 3 then uses the terminal empirical HUC.

\paragraph{Step 1: count reachable audit-value vectors on a finite sequence.}
After symmetrizing~Eq.~\eqref{rtf:eq:uniform-gap}, it is enough to distinguish audit functions on the finite sequence consisting of the original and ghost samples. Since an audit value is determined by the final prediction, we first count the possible prediction-value vectors at each iteration.

\begin{lemma}[Reachable value vectors on a finite sequence]
\label{rtf:lem:reachable-values}
For every fixed positive initial predictor $f_0$, finite sequence $z_{1:m}=((x_k,y_k))_{k=1}^m$, integer $T\ge0$,
\begin{align}
\left|\left\{(f(x_k))_{k=1}^m:f\in\mathcal F_{\le T}(f_0)\right\}\right|
&\le\{2N_{\mathrm{cand}}(m+1)^2\}^{T},
\label{rtf:eq:reachable-prediction-values}\\
\left|\left\{(g(z_k))_{k=1}^m:g\in\mathcal G_{\le T}(f_0)\right\}\right|
&\le\{2N_{\mathrm{cand}}(m+1)^2\}^{T+1}.
\label{rtf:eq:reachable-audit-values}
\end{align}
\end{lemma}

\begin{proof}
Let $\mathcal P_t$ be the set of prediction-value vectors on $x_1,\ldots,x_m$ obtainable after exactly $t$ updates, and write $L_t:=|\mathcal P_t|$. Since the initial predictor is fixed, $L_0=1$.

Fix a current vector $\boldsymbol p=(p_1,\ldots,p_m)\in\mathcal P_t$. There are $N_{\mathrm{cand}}$ choices of $(c,u,v)$ and two update signs. Once these are fixed, $c(x_k)$, $\mu_u(p_k)$, $\boldsymbol\Delta_{u,v}(p_k)$, and the current branch probabilities are determined at every $k$. The remaining information supplied by the interval $I$ is only its membership set
$\{k\in[m]:\mu_u(p_k)\in I\}$.
Set $s_k:=\mu_u(p_k)$ and order its distinct values as $r_1<\cdots<r_q$. If an interval contains $r_a$ and $r_b$ with $a<\ell<b$, then $r_a<r_\ell<r_b$, so the interval property forces it to contain $r_\ell$. Thus $I\cap\{r_1,\ldots,r_q\}$ is either empty or a consecutive rank block $\{r_a,\ldots,r_b\}$ for some $1\le a\le b\le q$. Conversely, every such block is realized by the closed interval $[r_a,r_b]$. Interval membership therefore depends only on ranks, not on distances between scores, and observations with equal scores are always selected together. The nonempty blocks correspond to endpoint pairs $(a,b)$ and number $q(q+1)/2$. Including the empty set, the number of membership sets is at most
$1+\frac{q(q+1)}2\le(q+1)^2\le(m+1)^2$.

For a fixed current prediction vector, fixing $(c,u,v)$, the sign, and the interval membership set uniquely determines the next prediction value at every $x_k$ through the fixed-step branch of Algorithm~\ref{alg:empirical-huc-boost}. Hence
\begin{equation}
L_{t+1}\le2N_{\mathrm{cand}}(m+1)^2L_t.
\label{rtf:eq:reachable-recursion}
\end{equation}
Induction from $L_0=1$ gives $L_T\le\{2N_{\mathrm{cand}}(m+1)^2\}^T$. On the chosen finite sequence, an interval with an empty membership set leaves all prediction values unchanged. Thus every value vector reached within $T$ steps on these points is represented by a length-$T$ sequence, proving Eq.~\eqref{rtf:eq:reachable-prediction-values}.

Once the final prediction vector is fixed, a signed audit-value vector is determined by a choice of $(c,u,v)$, an interval membership set, and $\varsigma$. Counting these choices gives Eq.~\eqref{rtf:eq:reachable-audit-values}, with at most $2N_{\mathrm{cand}}(m+1)^2$ possibilities.
\end{proof}

\paragraph{Step 2: symmetrize the uniform generalization gap.}
\begin{lemma}[Exponential moment for a fixed initializer]
\label{rtf:lem:fixed-mgf}
For every fixed positive initial predictor $f_0$, integer $T\ge0$, and $\lambda>0$,
\begin{equation}
\mathbb{E}_{S\sim\mathbb{P}^n}\left[\exp(\lambda\operatorname{Gap}_T(f_0,S))\right]
\le
\{2N_{\mathrm{cand}}(2n+1)^2\}^{T+1}
\exp\left(\frac{8\lambda^2}{n}\right).
\label{rtf:eq:fixed-mgf}
\end{equation}
\end{lemma}

\begin{proof}
Write $\mathcal G:=\mathcal G_{\le T}(f_0)$, which is independent of $S$.
Let $S'=(Z_1',\ldots,Z_n')\sim\mathbb{P}^n$ be an independent ghost sample,
and let $\xi_1,\ldots,\xi_n$ be independent Rademacher signs, independent of $S,S'$.
Conditional Jensen's inequality and exchangeability of each pair $(Z_i,Z_i')$ give
\[
\mathbb{E}_S e^{\lambda\operatorname{Gap}_T(f_0,S)}
\le\mathbb{E}_{S,S',\xi}\exp\left(
\frac{\lambda}{n}\sup_{g\in\mathcal G}\sum_{i=1}^n\xi_i[g(Z_i')-g(Z_i)]
\right).
\]
Conditional on $S,S'$, Lemma~\ref{rtf:lem:reachable-values} reduces this supremum
to $L\le\{2N_{\mathrm{cand}}(2n+1)^2\}^{T+1}$ value vectors.
Choose representatives $g_1,\ldots,g_L$ and set $b_{j,i}:=g_j(Z_i')-g_j(Z_i)$;
then $|b_{j,i}|\le4$ because $|g_j|\le2$.
Hoeffding's lemma and $e^{\max_j a_j}\le\sum_j e^{a_j}$ yield
\[
\mathbb{E}_\xi\exp\left(\frac{\lambda}{n}\max_{j\in[L]}\sum_{i=1}^n\xi_i b_{j,i}\right)
\le\sum_{j=1}^L\exp\left(\frac{\lambda^2}{2n^2}\sum_{i=1}^n b_{j,i}^2\right)
\le L e^{8\lambda^2/n}.
\]
Taking expectations over $S,S'$ proves Eq.~\eqref{rtf:eq:fixed-mgf}.
\end{proof}

\begin{lemma}[Uniform generalization gap for a fixed initializer]
\label{rtf:lem:fixed-deviation}
For every fixed positive initial predictor $f_0$, integer $T\ge0$, and $0<\delta<1$, with probability at least $1-\delta$,
\begin{equation}
\operatorname{Gap}_T(f_0,S)
\le4\sqrt{\frac2n\left[(T+1)\log(2N_{\mathrm{cand}}(2n+1)^2)+\log\frac1\delta\right]}.
\label{rtf:eq:fixed-deviation}
\end{equation}
\end{lemma}

\begin{proof}
Apply Markov's inequality to Lemma~\ref{rtf:lem:fixed-mgf} with $\lambda=n\rho/16$.
For every $\rho>0$,
\[
\mathbb{P}(\operatorname{Gap}_T(f_0,S)>\rho)
\le\exp\left((T+1)\log(2N_{\mathrm{cand}}(2n+1)^2)-\frac{n\rho^2}{32}\right).
\]
Setting $\rho$ to the right-hand side of Eq.~\eqref{rtf:eq:fixed-deviation}
makes this probability at most $\delta$.
\end{proof}

\paragraph{Step 3: combine empirical HUC at termination with the generalization gap.}
\begin{proof}[Proof of Theorem~\ref{thm:fixed-initializer-sample}]
Theorem~\ref{thm:fixed-initializer-stopping} and~Eq.~\eqref{rtf:eq:fixed-initial-loss} give $\widehat f\in\mathcal F_{\le T_\varepsilon^{\mathrm{fix}}}(f_0)$. They also give $\widehat{\HUC}(\widehat f;\gC,\gU)\le3\varepsilon/4$. Apply Lemma~\ref{rtf:lem:fixed-deviation} with $T=T_\varepsilon^{\mathrm{fix}}$. Under~Eq.~\eqref{rtf:eq:fixed-sample-main}, its right-hand side is at most $\varepsilon/4$, and therefore, with probability at least $1-\delta$,
\[
\HUC(\widehat f;\gC,\gU)
\le\widehat{\HUC}(\widehat f;\gC,\gU)+\operatorname{Gap}_{T_\varepsilon^{\mathrm{fix}}}(f_0,S)
\le\varepsilon.
\]
\end{proof}

\subsection{Proofs for the Fixed-Step Two-Stage Method}
\label{rtf:sec:learned-proof}
\paragraph{Proof strategy.}
For the stopping count, we bound the initial empirical log loss after perturbation and apply Theorem~\ref{thm:fixed-initializer-stopping}. For the population HUC guarantee, we first upper-bound the exponential moment of the uniform generalization gap for the sample-dependent initializer $\Theta_0$ by an exponential moment under a sample-independent reference distribution. We then integrate the fixed-initializer bound of Lemma~\ref{rtf:lem:fixed-mgf} against this reference distribution, obtain a high-probability uniform generalization bound, and combine it with the empirical HUC at termination.

\paragraph{Step 1: initial loss after perturbation and the stopping count.}
\begin{lemma}[Initial log loss after perturbation]
\label{rtf:lem:initial-loss}
Under Assumption~\ref{ass:random-init-main}, almost surely,
\begin{equation}
\widehat{\mathcal L}_{\log,S}(f_{0,\Theta_0})\le D_0+\varepsilon^2.
\label{rtf:eq:epsilon-initial-loss}
\end{equation}
\end{lemma}

\begin{proof}
For $z,z'\in\R^K$, log-sum-exp is 1-Lipschitz with respect to $\|\cdot\|_\infty$. Hence the one-observation log loss $\ell_{\log}(z,y)=\log\sum_ke^{z_k}-z_y$ satisfies
\begin{equation}
|\ell_{\log}(z,y)-\ell_{\log}(z',y)|\le2\|z-z'\|_\infty.
\label{rtf:eq:loss-lipschitz}
\end{equation}
Since $\|\Xi\|_\infty\le a_\varepsilon$ and~Eq.~\eqref{eq:random-lipschitz-main} holds, for every $i$,
\[
\ell_{\log}(s_{\Theta_0}(X_i),Y_i)
\le\ell_{\log}(s_{\widehat\theta}(X_i),Y_i)+2L_\theta a_\varepsilon.
\]
Averaging over $i$ and using~Eq.~\eqref{eq:random-fit-main} and $2L_\theta a_\varepsilon=\varepsilon^2$ proves the claim.
\end{proof}

\begin{proof}[Proof of Theorem~\ref{thm:learned-initializer-stopping}]
Fix a realization of $(S,\Xi)$. Apply Theorem~\ref{thm:fixed-initializer-stopping} to this realized sample and the initial predictor $f_{0,\Theta_0}$. Lemma~\ref{rtf:lem:initial-loss} gives
\[
N_{\mathrm{upd}}
\le\left\lceil\frac{4(D_0+\varepsilon^2)}{\varepsilon^2}\right\rceil
=\left\lceil\frac{4D_0}{\varepsilon^2}\right\rceil+4
=T_\varepsilon^{\mathrm{learn}}.
\]
The stopping rule also gives $\widehat{\HUC}(\widehat f;\gC,\gU)\le3\varepsilon/4$.
\end{proof}

\paragraph{Step 2: dominate the distribution of the sample-dependent initializer by a reference distribution.}
For Theorem~\ref{thm:learned-initializer-sample}, with fixed $T\ge0$ and $\lambda>0$, the quantity to be bounded is
\begin{equation}
\mathbb{E}_{S,\Xi}\left[\exp\left(\lambda\operatorname{Gap}_T(f_{0,\Theta_0},S)\right)\right].
\label{rtf:eq:target-learned-mgf}
\end{equation}
Lemma~\ref{rtf:lem:fixed-mgf} controls this exponential moment for each fixed parameter $\theta$. Introduce the reference distribution
\begin{equation}
\Pi_\varepsilon
:=\operatorname{Unif}\left([-R_\theta-a_\varepsilon,R_\theta+a_\varepsilon]^d\right),
\label{rtf:eq:reference-distribution}
\end{equation}
and dominate the joint distribution of $(S,\Theta_0)$ by the product distribution $\mathbb{P}^n\otimes\Pi_\varepsilon$.

\begin{lemma}[Transfer from a sample-dependent initializer to the reference distribution]
\label{rtf:lem:domination}
For every nonnegative measurable function
$\Phi:(\mathcal X\times\mathcal Y)^n\times\mathbb R^d\to[0,\infty]$,
\begin{equation}
\mathbb{E}_{S,\Xi}[\Phi(S,\Theta_0)]
\le e^{\Lambda_\varepsilon}
\int_{\mathbb R^d}\mathbb{E}_{S\sim\mathbb{P}^n}[\Phi(S,\theta)]\,\Pi_\varepsilon(d\theta).
\label{rtf:eq:reference}
\end{equation}
\end{lemma}

\begin{proof}
Fix a sample value $\mathbf z$ satisfying~Eq.~\eqref{eq:random-fit-main}, which holds for $\mathbb{P}^n$-almost every $\mathbf z$, and write $\widehat\theta_{\mathbf z}:=\mathsf A(\mathbf z)$. Conditional on $S=\mathbf z$, the Lebesgue density of $\Theta_0=\widehat\theta_{\mathbf z}+\Xi$ is
\[
k_{\mathbf z}(\theta)
=(2a_\varepsilon)^{-d}\mathbf 1\left\{\theta\in\widehat\theta_{\mathbf z}+[-a_\varepsilon,a_\varepsilon]^d\right\}.
\]
Because $\|\widehat\theta_{\mathbf z}\|_\infty\le R_\theta$, its support is contained in the support of $\Pi_\varepsilon$, whose Lebesgue density is
\[
\pi_\varepsilon(\theta)
=[2(R_\theta+a_\varepsilon)]^{-d}\mathbf 1\left\{\theta\in[-R_\theta-a_\varepsilon,R_\theta+a_\varepsilon]^d\right\}.
\]
Therefore, for Lebesgue-almost every $\theta$,
\begin{equation}
k_{\mathbf z}(\theta)
\le\left(1+\frac{R_\theta}{a_\varepsilon}\right)^d\pi_\varepsilon(\theta)
=e^{\Lambda_\varepsilon}\pi_\varepsilon(\theta).
\label{rtf:eq:density-domination}
\end{equation}
Tonelli's theorem and~Eq.~\eqref{rtf:eq:density-domination} give
\[
\mathbb{E}_{S,\Xi}[\Phi(S,\Theta_0)]
=\mathbb{E}_{S}\left[\int_{\mathbb R^d}\Phi(S,\theta)k_S(\theta)\,d\theta\right]
\le e^{\Lambda_\varepsilon}\int_{\mathbb R^d}\mathbb{E}_{S\sim\mathbb{P}^n}[\Phi(S,\theta)]\,\Pi_\varepsilon(d\theta).
\qedhere
\]
\end{proof}

Applying Eq.~\eqref{rtf:eq:reference} with $\Phi(S,\theta)=\exp(\lambda\operatorname{Gap}_T(f_{0,\theta},S))$ bounds Eq.~\eqref{rtf:eq:target-learned-mgf} by the $\Pi_\varepsilon$-average of the fixed-$\theta$ exponential moments, multiplied by $e^{\Lambda_\varepsilon}$.

\paragraph{Step 3: a high-probability bound for the uniform generalization gap.}
\begin{lemma}[Uniform generalization gap for a learned initializer]
\label{rtf:lem:random-gap}
Under Assumption~\ref{ass:random-init-main}, for every fixed integer $T\ge0$ and $0<\delta<1$, with probability at least $1-\delta$ under the joint distribution of $S$ and $\Xi$,
\begin{equation}
\operatorname{Gap}_T(f_{0,\Theta_0},S)
\le4\sqrt{\frac2n\left[\Lambda_\varepsilon+(T+1)\log(2N_{\mathrm{cand}}(2n+1)^2)+\log\frac1\delta\right]}.
\label{rtf:eq:random-gap}
\end{equation}
\end{lemma}

\begin{proof}
Apply Lemma~\ref{rtf:lem:domination} with $\Phi(S,\theta)=e^{\lambda\operatorname{Gap}_T(f_{0,\theta},S)}$.
Lemma~\ref{rtf:lem:fixed-mgf} gives the same bound for every fixed $\theta$, so for every $\lambda>0$,
\begin{equation}
\begin{aligned}
\mathbb{E}_{S,\Xi}e^{\lambda\operatorname{Gap}_T(f_{0,\Theta_0},S)}
&\le e^{\Lambda_\varepsilon}\int_{\mathbb R^d}
\mathbb{E}_{S}e^{\lambda\operatorname{Gap}_T(f_{0,\theta},S)}\,\Pi_\varepsilon(d\theta)\\
&\le\exp\left(\Lambda_\varepsilon+(T+1)\log(2N_{\mathrm{cand}}(2n+1)^2)
+\frac{8\lambda^2}{n}\right).
\end{aligned}
\label{rtf:eq:transferred-mgf}
\end{equation}
Markov's inequality with $\lambda=n\rho/16$ gives, for every $\rho>0$,
\[
\mathbb{P}_{S,\Xi}(\operatorname{Gap}_T(f_{0,\Theta_0},S)>\rho)
\le\exp\left(\Lambda_\varepsilon+(T+1)\log(2N_{\mathrm{cand}}(2n+1)^2)
-\frac{n\rho^2}{32}\right).
\]
Setting $\rho$ to the right-hand side of Eq.~\eqref{rtf:eq:random-gap} proves the result.
\end{proof}

\paragraph{Step 4: combine empirical HUC at termination with the generalization gap.}
\begin{proof}[Proof of Theorem~\ref{thm:learned-initializer-sample}]
Theorem~\ref{thm:learned-initializer-stopping} gives, almost surely, $\widehat f\in\mathcal F_{\le T_\varepsilon^{\mathrm{learn}}}(f_{0,\Theta_0})$ and $\widehat{\HUC}(\widehat f;\gC,\gU)\le3\varepsilon/4$. Apply Lemma~\ref{rtf:lem:random-gap} with $T=T_\varepsilon^{\mathrm{learn}}$. Under~Eq.~\eqref{rtf:eq:learned-sample-main}, the right-hand side of~Eq.~\eqref{rtf:eq:random-gap} is at most $\varepsilon/4$. Therefore, with probability at least $1-\delta$,
\[
\HUC(\widehat f;\gC,\gU)
\le\widehat{\HUC}(\widehat f;\gC,\gU)+\operatorname{Gap}_{T_\varepsilon^{\mathrm{learn}}}(f_{0,\Theta_0},S)
\le\varepsilon.
\]
\end{proof}

\subsection{Explicit sample complexity analysis}
\label{rtf:sec:explicit-sample}
The sample conditions in Theorems~\ref{thm:fixed-initializer-sample} and~\ref{thm:learned-initializer-sample} both have the form
\begin{equation}
n\ge\frac{512}{\varepsilon^2}\left[\Lambda+(T+1)\log(2N_{\mathrm{cand}}(2n+1)^2)+\log\frac1\delta\right].
\label{rtf:eq:sample-generic}
\end{equation}
For the fixed initializer, $(T,\Lambda)=(T_\varepsilon^{\mathrm{fix}},0)$; for the learned initializer, $(T,\Lambda)=(T_\varepsilon^{\mathrm{learn}},\Lambda_\varepsilon)$.

\begin{corollary}[Explicit condition with $n$ removed from the right-hand side]
\label{rtf:cor:explicit-generic}
Let $\nu:=T+1$ and define
\begin{equation}
n_\star(T,\Lambda,\varepsilon,\delta)
:=\left\lceil\frac{1024}{\varepsilon^2}\left[\Lambda+\nu\log(18N_{\mathrm{cand}})+\log\frac1\delta+2\nu\log\left(\frac{2048\nu}{\varepsilon^2}\right)\right]\right\rceil.
\label{rtf:eq:sample-explicit-generic}
\end{equation}
Every integer $n\ge n_\star(T,\Lambda,\varepsilon,\delta)$ satisfies~Eq.~\eqref{rtf:eq:sample-generic}.
\end{corollary}

\begin{proof}
For $n\ge1$, $2N_{\mathrm{cand}}(2n+1)^2\le18N_{\mathrm{cand}}n^2$. It is therefore sufficient that $n\ge A+C\log n$, where
\[
A:=\frac{512}{\varepsilon^2}\left[\Lambda+\nu\log(18N_{\mathrm{cand}})+\log(1/\delta)\right],
\qquad
C:=\frac{1024\nu}{\varepsilon^2}.
\]
Applying $\log t\le t-1$ with $t=n/(2C)$ gives $A+C\log n\le A+n/2+C\log(2C)-C$. Hence $n\ge2A+2C\log(2C)$ is sufficient, and substitution yields~Eq.~\eqref{rtf:eq:sample-explicit-generic}.
\end{proof}

\subsection{Dirichlet, Temperature, and Vector Initial Models}
\label{rtf:sec:dirichlet}
For a fixed base prediction $q:\gX\to\Delta^{K-1}$, consider Dirichlet calibration~\citep{Kull2019},
\begin{equation}
f_{0,W,b}(x)=\operatorname{softmax}(W\log q(x)+b),
\qquad W\in\R^{K\times K},\quad b\in\R^K,
\label{rtf:eq:dirichlet-original}
\end{equation}
where the logarithm is applied coordinatewise. As in the experimental representation, all $K$ rows are treated as parameters.

\begin{proposition}[Constants for a Dirichlet initial model]
\label{rtf:prop:dirichlet}
Fix $q_{\min}\in(0,1/K]$ and $\lambda_{\mathrm{reg}}>0$ independently of $S$, and assume $q_y(x)\ge q_{\min}$ for every $x,y$. Let $\theta=\operatorname{vec}(W,b)\in\R^{K(K+1)}$, and suppose the initial fit satisfies
\begin{equation}
\widehat{\mathcal L}_{\log,S}(f_{0,\widehat\theta})
+\frac{\lambda_{\mathrm{reg}}}{2}\|\widehat\theta\|_2^2
\le\log K.
\label{rtf:eq:dirichlet-fit}
\end{equation}
Then Assumption~\ref{ass:random-init-main} holds with
\begin{equation}
d=K(K+1),
\qquad D_0=\log K,
\qquad R_\theta=\sqrt{\frac{2\log K}{\lambda_{\mathrm{reg}}}},
\qquad L_\theta=1+K\log(1/q_{\min}).
\label{rtf:eq:dirichlet-constants}
\end{equation}
\end{proposition}

\begin{proof}
Let $\varphi(x):=(\log q_1(x),\ldots,\log q_K(x),1)$, so $s_{\theta,y}(x)=\theta_y^\top\varphi(x)$. Since $0<q_y(x)\le1$,
$\|\varphi(x)\|_1 =1+\sum_{y=1}^K|\log q_y(x)| \le1+K\log(1/q_{\min})=L_\theta$.
Consequently,
$|s_{\theta,y}(x)-s_{\theta',y}(x)|\le L_\theta\|\theta-\theta'\|_\infty$,
which proves~Eq.~\eqref{eq:random-lipschitz-main}. Nonnegativity of log loss and~Eq.~\eqref{rtf:eq:dirichlet-fit} give
$\widehat{\mathcal L}_{\log,S}(f_{0,\widehat\theta})\le\log K$. They also give
$\|\widehat\theta\|_\infty\le\|\widehat\theta\|_2\le\sqrt{2\log K/\lambda_{\mathrm{reg}}}$,
which is~Eq.~\eqref{eq:random-fit-main}.
\end{proof}
The choice $W=0,b=0$ gives the uniform prediction and makes the left-hand side of~Eq.~\eqref{rtf:eq:dirichlet-fit} equal to $\log K$. Thus it is enough to return an initial fit whose objective is no worse than this fixed candidate. If the base prediction has no positive lower bound, use, for a predetermined $\kappa_0\in(0,1)$,
$q_y^{\kappa_0}(x):=(1-\kappa_0)q_y(x)+\kappa_0/K$ as the input to the initial model; then $q_{\min}=\kappa_0/K$ is valid.

\begin{corollary}[Noise width and sample condition for a Dirichlet initial model]
\label{rtf:cor:dirichlet}
Under Proposition~\ref{rtf:prop:dirichlet}, choose the uniform-noise width
\begin{equation}
a_\varepsilon=\frac{\varepsilon^2}{2\{1+K\log(1/q_{\min})\}}.
\label{rtf:eq:dirichlet-noise}
\end{equation}
Then
\begin{equation}
T_\varepsilon^{\mathrm{learn}}=\left\lceil\frac{4\log K}{\varepsilon^2}\right\rceil+4,
\qquad
\Lambda_\varepsilon=K(K+1)\log\left(1+\frac{2R_\theta L_\theta}{\varepsilon^2}\right).
\label{rtf:eq:dirichlet-final-constants}
\end{equation}
Under the sample condition obtained by substituting these constants into~Eq.~\eqref{rtf:eq:sample-explicit-generic}, $\HUC(\widehat f;\gC,\gU)\le\varepsilon$ with probability at least $1-\delta$. The dependence on accuracy is
\begin{equation}
n_\star=O\!\left(
\frac{1+\log K}{\varepsilon^4}\log\frac{eN_{\mathrm{cand}}(1+\log K)}{\varepsilon^4}
+\frac{K(K+1)}{\varepsilon^2}\log\left(1+\frac{2R_\theta L_\theta}{\varepsilon^2}\right)
+\frac{\log(1/\delta)}{\varepsilon^2}
\right).
\label{rtf:eq:dirichlet-order}
\end{equation}
\end{corollary}

\begin{proof}
Substitute the constants from Proposition~\ref{rtf:prop:dirichlet} into Theorem~\ref{thm:learned-initializer-sample} and Corollary~\ref{rtf:cor:explicit-generic}. Eq.~\eqref{rtf:eq:dirichlet-order} follows from $T_\varepsilon^{\mathrm{learn}}+1=O((1+\log K)/\varepsilon^2)$ and~Eq.~\eqref{rtf:eq:dirichlet-final-constants}.
\end{proof}

Subtracting the $K$th logit from all logits leaves softmax unchanged. If initial fitting and perturbation are defined in coordinates with the $K$th row fixed to zero, one may use $d=(K-1)(K+1)=K^2-1$. The same proof applies when~Eq.~\eqref{rtf:eq:dirichlet-fit} is imposed on all free coefficients in those coordinates.

\paragraph{Application to temperature scaling.}
Let $L_q:=\log(1/q_{\min})$. Temperature scaling, parameterized by the inverse temperature $\beta$, can be written as
\begin{equation}
s_{\beta,y}(x)=\beta\log q_y(x),
\qquad d=1,
\qquad L_\theta=L_q.
\label{rtf:eq:temperature-model}
\end{equation}
Indeed, $\sup_x\|s_\beta(x)-s_{\beta'}(x)\|_\infty\le L_q|\beta-\beta'|$. Restricting the fitted parameter to a bounded interval determines $R_\theta$, and $D_0$ may be obtained from an objective-comparison guarantee or a bound on the logits. To preserve a positive inverse temperature, choose the lower endpoint of the fitting domain so that $\beta_{\min}>a_\varepsilon$; then the perturbed parameter also remains positive.

\paragraph{Application to vector scaling.}
Let $\theta=(a,b)\in\R^{2K}$ and write vector scaling as
\begin{equation}
s_{\theta,y}(x)=a_y\log q_y(x)+b_y,
\qquad d=2K,
\qquad L_\theta=L_q+1.
\label{rtf:eq:vector-model}
\end{equation}
For any $\theta=(a,b)$ and $\theta'=(a',b')$,
$\sup_x\|s_\theta(x)-s_{\theta'}(x)\|_\infty\le(L_q+1)\|\theta-\theta'\|_\infty$.
Thus Assumption~\ref{ass:random-init-main} applies with $d=2K$ and $L_\theta=L_q+1$ whenever the fitted parameter is restricted to a bounded domain and the initial empirical log loss is bounded by $D_0$.

\begin{remark}[Role and Scope of the One-Time Random Perturbation]
\label{rtf:sec:scope}
When the initial predictor is fixed independently of the audit sample, no random perturbation is needed. When $\widehat\theta=\mathsf A(S)$ is learned from the same sample, the perturbation transfers the fixed-initializer uniform generalization bound to the sample-dependent initializer, at the density cost $\Lambda_\varepsilon$.

The theorem guarantees the output obtained after adding the stated perturbation once and running the fixed step $\alpha_\varepsilon=\varepsilon/2$.

The width $a_\varepsilon=\varepsilon^2/(2L_\theta)$ balances the density ratio and the initial loss increase. For a general width $a>0$, the log density ratio is $d\log(1+R_\theta/a)$, whereas the increase in the initial empirical log loss is at most $2L_\theta a$. The selected width limits this increase to $\varepsilon^2$, adding four to the update bound.
\end{remark}

\section{Additional Comparison with Related Work}
\label{app:additional-related-work}
\label{subsec:additional-utility-discussion}

\paragraph{U-calibration and unknown downstream utilities.}
\citet{pmlr-v195-kleinberg23a} study sequential forecasts of
binary outcomes used by agents whose utilities are unknown to
the forecaster.
Their U-calibration error is the worst-case regret over bounded
proper scoring rules, measured against the best fixed forecast
in hindsight.
This connects forecast evaluation to regret guarantees for
agents that best respond to the forecasts.
U-calibration evaluates cumulative regret, whereas the Utility
Calibration of \citet{Hegazy2025} audits expected utility
residuals within predicted-utility intervals.
HUC decomposes these residuals along a fixed label tree and
audits the resulting node contributions for the specified
utility class.

\paragraph{Loss outcome indistinguishability and omniprediction.}
\citet{gopalan_et_al:LIPIcs.ITCS.2023.60} compare actual outcomes with outcomes
sampled from a predictor using statistical tests induced by a
loss family and a comparison hypothesis class.
Their loss outcome indistinguishability condition implies
omniprediction: loss-specific post-processing competes with the
best hypothesis in the class for each loss.
This shares with HUC the use of expected discrepancies evaluated
through a specified family of tests.
The distinction lies in the tests and the resulting guarantees.
HUC uses the node contributions of a fixed target utility under
common subgroup weights and predicted-utility intervals.
Its guarantees concern these calibration moments and the
restricted monotone post-processing comparisons in
Appendix~\ref{app:postprocessing-huc}, rather than loss minimization
against a general hypothesis class.

\paragraph{Utility-directed conformal prediction.}
\citet{cortes-gomez_utility-directed_2025} incorporate downstream
decision losses and user-specified utilities into the construction
of conformal prediction sets while retaining marginal coverage
guarantees.
Their healthcare application also uses a hierarchy of
dermatological diseases to produce diagnostically coherent sets.
This shares our interest in utility-aware use of label hierarchies,
but the statistical targets differ.
Conformal coverage concerns whether the realized label belongs
to the reported set.
HUC instead audits the agreement between predicted and realized
utility through internal-node contributions to a probability
forecast's utility residual.
Thus, utility-directed set construction and hierarchical
utility calibration address different aspects of reliable
decision-making.

\section{Toy Data}
\label{app:toy-data}
\label{app:toy8-experiment-v26}
\label{app:toy-randomized-huc}

We first specify the single-sample eight-leaf correction experiment used in the main text. A separate subsection explains the decision-utility construction.

\subsection{Experimental Settings of the Main Experiment}
\label{subsec:toy-main-experiment}
\paragraph{Inputs, subgroups, and utilities.}
The model has eight leaf labels and seven internal nodes. Its initial prediction is specified directly, without fitting a classifier to input features. All evaluations use the whole-population indicator, $\gC=\{1\}$, so $|\gC|=1$. The main experiment uses the three decision utilities defined below, with $|\gU|=3$.

\paragraph{Tree.}
Consider the binary tree in Figure~\ref{fig:toy-appendix-line-tree}, with eight leaf labels and seven internal nodes. Although synthetic, it can be interpreted as a staged classification of conditions in patients presenting with symptoms. First, the root $v_{\mathrm{cause}}$ distinguishes noninfectious from infectious causes. The noninfectious branch $v_{\mathrm{noninf}}$ then separates mild and severe conditions, $y_1,y_2$. For infectious causes, $v_{\mathrm{spread}}$ distinguishes localized from systemic infection. Within localized infection, $v_{\mathrm{site}}$ separates upper respiratory infection from pneumonia; $v_{\mathrm{upper}}$ and $v_{\mathrm{lower}}$ then distinguish viral from bacterial causes, giving $y_3,y_4$ and $y_5,y_6$, respectively. Within systemic infection, $v_{\mathrm{shock}}$ separates nonsevere and severe conditions, $y_7,y_8$. These synthetic interpretations are not clinical recommendations.

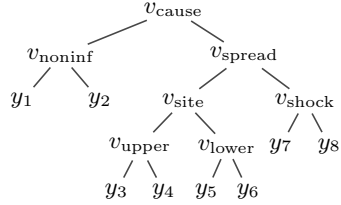
\begin{figure}[htbp]
\centering
\begingroup
\tikzset{every picture/.append style={xscale=0.90,yscale=0.80}}
\begin{tikzpicture}[x=1.08cm,y=0.94cm,
  every node/.style={font=\footnotesize,inner sep=0.8pt,fill=white},
  treeedge/.style={draw=black!72,line width=0.60pt,shorten >=1.6pt,shorten <=1.6pt}]
\path[use as bounding box,draw=none] (-2.55,0.30) rectangle (2.55,-3.480);
\node (r) at (0,0) {$v_{\mathrm{cause}}$};
\node (ni) at (-1.529,-0.800) {$v_{\mathrm{noninf}}$};
\node (sp) at (0.963,-0.800) {$v_{\mathrm{spread}}$};
\node (y1) at (-2.058,-1.600) {$y_1$};
\node (y2) at (-1.001,-1.600) {$y_2$};
\node (si) at (0.142,-1.600) {$v_{\mathrm{site}}$};
\node (sh) at (1.794,-1.600) {$v_{\mathrm{shock}}$};
\node (up) at (-0.453,-2.400) {$v_{\mathrm{upper}}$};
\node (lo) at (0.727,-2.400) {$v_{\mathrm{lower}}$};
\node (y7) at (1.463,-2.400) {$y_7$};
\node (y8) at (2.124,-2.400) {$y_8$};
\node (y3) at (-0.774,-3.200) {$y_3$};
\node (y4) at (-0.132,-3.200) {$y_4$};
\node (y5) at (0.453,-3.200) {$y_5$};
\node (y6) at (1.020,-3.200) {$y_6$};
\foreach \a/\b in {r/ni,r/sp,ni/y1,ni/y2,sp/si,sp/sh,si/up,si/lo,sh/y7,sh/y8,up/y3,up/y4,lo/y5,lo/y6}{\draw[treeedge] (\a)--(\b);}
\end{tikzpicture}
\endgroup
\caption{The eight-leaf binary tree used in the Toy experiments.}
\label{fig:toy-appendix-line-tree}
\end{figure}

\paragraph{Decision utilities.}
Each utility evaluates the action selected from the prediction $p$ against the realized label $y$. For decision $j$, let $D_j$ be the leaves on which the decision is active, $S_j\subseteq D_j$ the leaves favoring action~1 over action~0, and $\tau_j\in(0,1)$ its threshold. Define
\begin{equation}
\begin{aligned}
&d_j(y)=\mathbf{1}\{y\in S_j\}-\tau_j\mathbf{1}\{y\in D_j\}, \qquad  \widehat a_j(p)=\mathbf{1}\!\left\{\sum_y p_y d_j(y)\ge0\right\},\\
&u_j(p,y)=\{2\widehat a_j(p)-1\}d_j(y).
\end{aligned}
\label{eq:toy8-prediction-dependent-utilities}
\end{equation}
The three decisions use
\begin{equation}
\begin{array}{c|c|c|c}
j & D_j & S_j & \tau_j \\
\hline
1 & \{y_7,y_8\} & \{y_8\} & 0.40 \\
2 & \{y_1,\ldots,y_8\} & \{y_2,y_5,y_6,y_7,y_8\} & 0.50 \\
3 & \{y_1,\ldots,y_8\} & \{y_4,y_6,y_7,y_8\} & 0.40
\end{array}
\label{eq:toy8-decision-sets-updated}
\end{equation}
Their relevant-node sets contain 3, 4, and 5 nodes, and their union covers all seven internal nodes.

\paragraph{Population distribution and random initial predictions.}
The true leaf distribution is
\begin{equation}
\eta=
\left(
\frac{13}{80},\frac7{80},\frac{21}{125},\frac{21}{500},
\frac{189}{2000},\frac{441}{2000},\frac9{80},\frac9{80}
\right).
\label{eq:toy8-eta-updated}
\end{equation}
For each repetition, the predicted right-branch probabilities at the seven internal nodes are drawn independently from $\operatorname{Unif}(0.1,0.9)$. Multiplying the sampled branch probabilities along each root-to-leaf path gives a positive initial leaf distribution $p_0$. Both corrections start from this $p_0$.

\paragraph{Calibration sample.}
For each sample size $n$, draw i.i.d. labels $Y_1,\ldots,Y_n$ from $\eta$ and form
\begin{equation}
\widehat\eta_y=\frac1n\sum_{i=1}^n\mathbf{1}\{Y_i=y\},
\qquad n\widehat\eta\sim\operatorname{Multinomial}(n,\eta).
\label{eq:toy8-single-calibration-sample}
\end{equation}
Both corrections use this same calibration sample throughout their updates; $n$ is the total number of calibration labels.

\paragraph{Correction procedures.}
At each UC-Boost iteration, choose the utility with the largest absolute empirical aggregate residual. For the selected utility $u$ and current prediction $p_t$, set
\begin{equation}
\begin{aligned}
h_{t,y}&=u(p_t,y)-\mu_u(p_t),
&\widehat\Gamma_t&=\sum_y\widehat\eta_yh_{t,y},\\
\widehat\Lambda_t&=\frac14\left(\max_y h_{t,y}-\min_y h_{t,y}\right)^2,
&\alpha_t&=\widehat\Gamma_t/\widehat\Lambda_t,
\end{aligned}
\label{eq:toy8-empirical-uc-coefficients}
\end{equation}
and update all leaf logits by
\begin{equation}
p_{t+1}=\operatorname{softmax}\!\left(\log p_t+\alpha_t h_t\right).
\label{eq:toy8-empirical-uc-update}
\end{equation}
HUC-Boost instead selects the utility--node pair with the largest absolute empirical node contribution over all utilities and their relevant nodes. It applies Algorithm~\ref{alg:huc-boost} with every expectation replaced by the average over the same calibration sample: the selected node's branch logits are updated with coefficient $\widehat\Gamma_t/\widehat\Lambda_t$, and the leaf distribution is reconstructed by path products. This preserves the selected node's reach probability, the conditional leaf distribution within each child subtree, and every leaf probability outside the selected subtree.

Both methods apply each selected update once and then recompute the utilities and all audit candidates at the new prediction. They stop when their respective empirical UC or HUC is at most $10^{-10}$, or when the update cap is reached, and return the final prediction. The caps are 100 updates for UC-Boost and $100\times7=700$ updates for HUC-Boost. All runs, including those reaching a cap, are included in the reported results.

\paragraph{Exact population evaluation and repetitions.}
We evaluate the final corrected predictions by exact population UC and HUC under the known $\eta$. Thus the reported errors reflect finite-sample correction without an additional finite-test-sample error. We use
\[
n\in\{250,500,1000,2000,5000,10^4,2\!\times\!10^4,5\!\times\!10^4, 10^5,2\!\times\!10^5,5\!\times\!10^5,10^6\}
\]
and perform 200 repetitions per sample size, reusing the same 200 initial predictions across all sample sizes. Figure~\ref{fig:toy8-sample-size-main}(b) reports the means, with one sample standard deviation at selected sizes.

\subsection{Interpretation of the Decision Utilities}
\label{subsec:toy-decision-utility-interpretation}
The construction in Eq.~\eqref{eq:toy8-prediction-dependent-utilities} evaluates an action selected from the prediction $p$ against the realized label $y$. This subsection explains the roles of its outcome sets and threshold.

\paragraph{Outcome sets and label-wise values.}
For decision $j$, $D_j$ is the set of outcomes on which the action comparison is evaluated. Within $D_j$, the outcomes in $S_j$ favor action~1, whereas those in $D_j\setminus S_j$ favor action~0. Outside $D_j$, both actions receive zero utility. For $0<\tau_j<1$, the definition of $d_j$ gives
\[
d_j(y)=
\begin{cases}
1-\tau_j, & y\in S_j,\\
-\tau_j, & y\in D_j\setminus S_j,\\
0, & y\notin D_j.
\end{cases}
\]
Thus $d_j(y)$ is positive on outcomes favoring action~1 and negative on outcomes favoring action~0.

\paragraph{Payoffs of the two actions.}
Write the action-level payoff as
$U_j(a,y):=(2a-1)d_j(y)$, $a\in\{0,1\}$.
Action~1 receives $d_j(y)$, and action~0 receives $-d_j(y)$:
\begin{center}
\begin{tabular}{lcc}
\toprule
Realized outcome & $U_j(1,y)$ & $U_j(0,y)$ \\
\midrule
$y\in S_j$ & $1-\tau_j$ & $-(1-\tau_j)$ \\
$y\in D_j\setminus S_j$ & $-\tau_j$ & $\tau_j$ \\
$y\notin D_j$ & $0$ & $0$ \\
\bottomrule
\end{tabular}
\end{center}
On $D_j$, choosing the preferred action gives positive utility, and choosing the other action gives negative utility. The parameter $\tau_j$ controls the relative magnitudes of these payoffs: choosing action~0 on $S_j$ gives $-(1-\tau_j)$, whereas choosing action~1 on $D_j\setminus S_j$ gives $-\tau_j$. These two penalties have equal magnitude when $\tau_j=1/2$.

\paragraph{Prediction-based choice and the threshold.}
For $A\subseteq\gY$, write $p(A):=\sum_{y\in A}p_y$. The predicted expected payoffs of actions~1 and~0 are, respectively,
\[
\sum_y p_yU_j(1,y)=p(S_j)-\tau_jp(D_j),\qquad
\sum_y p_yU_j(0,y)=-\{p(S_j)-\tau_jp(D_j)\}.
\]
The rule $\widehat a_j(p)$ therefore selects the action with the larger predicted expected payoff, choosing action~1 in a tie. Since $p(D_j)>0$ in this experiment,
$\widehat a_j(p)=1 \quad\Longleftrightarrow\quad \frac{p(S_j)}{p(D_j)}\ge\tau_j$.
Thus $\tau_j$ is the threshold for the predicted probability of an outcome favoring action~1 conditional on the outcome lying in $D_j$. When $D_j=\gY$, the condition reduces to $p(S_j)\ge\tau_j$. The prediction selects the action before the outcome is observed, and the realized label determines its payoff:
$u_j(p,y)=U_j\bigl(\widehat a_j(p),y\bigr)$.

\section{Real Data}
\label{app:real-data}
\label{app:experimental-details}
\label{app:real-data-details-v16}
\label{app:additional-real-data}
\label{app:additional-results}
\label{app:additional-realdata-v2}
\label{app:additional-real-data-v26}

The exact interval maximization for empirical UC and HUC is described in Appendix~\ref{subsec:exact-interval-scan} (Algorithm~\ref{alg:exact-interval-scan}).

We evaluate iNaturalist 2019, SUPPORT2, MASSIVE, NSL-KDD, and N-BaIoT. The first two provide the main-text comparisons; the others cover text and network-traffic classification.

\subsection{Common Experimental Settings}
\label{subsec:real-common-settings}
\label{subsec:experimental-global-recalibration}

\paragraph{Training and evaluation.}
Base predictors are fitted on training data, and recalibration uses separate calibration data. Validation data select regularization and the returned corrected predictor; test data are reserved for final evaluation. We evaluate UC and HUC, as well as accuracy and AUC. AUC is the unweighted average of one-vs-rest ROC AUC values over classes with both positive and negative examples in the evaluation set. Each dataset uses five data splits, with its sample sizes and allocation described below. Within each split, all calibration methods share the same base predictions. Results are reported as means and sample standard deviations. For the six-classifier comparisons, pooled results give equal weight to the six classifiers and five splits.

\paragraph{Base predictors.}
Except for MASSIVE, we compare logistic regression (LR), Gaussian naive Bayes (GNB), decision trees (DT), random forests (RF), gradient-boosted trees (XGB), and a multilayer perceptron (MLP). Each classifier predicts the leaf labels directly. The probability of an internal node is the sum of its descendant-leaf probabilities; no separate classifier is trained at an internal node. The input features and their dimensions are specified for each dataset below.

\paragraph{Calibration methods.}
Baseline denotes the common positive base prediction. For a fitted base predictor $f_{\mathrm{raw}}$, we set
\begin{equation}
f_{\mathrm{base}}(x)=(1-\kappa_{\mathrm{init}})f_{\mathrm{raw}}(x)
+\frac{\kappa_{\mathrm{init}}}{K}\boldsymbol 1,
\qquad \kappa_{\mathrm{init}}=10^{-10},
\label{eq:experimental-interiorization-v16}
\end{equation}
using the same coefficient for every real-data dataset and calibration method. This transformation is applied once to the base predictions before recalibration; it is not repeated after boosting updates. We compare direct UC-Boost and HUC-Boost, temperature scaling, vector scaling~\citep{Guo2017}, and Dirichlet calibration~\citep{Kull2019}. The three parametric transformations are defined in Table~\ref{tab:global-calibration-parameters-v16}.

\begin{table}[htbp]
\caption{Whole-distribution recalibration baselines and the parameters learned on calibration data.}
\label{tab:global-calibration-parameters-v16}
\centering
\begin{tabularx}{\textwidth}{@{}lXl@{}}
\toprule
Method & Transformation & Parameters learned on calibration data \\
\midrule
Temperature & $g_{\mathrm{Temp},T}(p)=\operatorname{softmax}(\log p/T)$ & scalar $T>0$ \\
Vector & $g_{\mathrm{Vec},a,b_{\mathrm V}}(p)=\operatorname{softmax}(a\odot\log p+b_{\mathrm V})$ & $a,b_{\mathrm V}\in\R^K$ \\
Dirichlet & $g_{\mathrm{Dir},W,b_{\mathrm D}}(p)=\operatorname{softmax}(W\log p+b_{\mathrm D})$ & $W\in\R^{K\times K}$, $b_{\mathrm D}\in\R^K$ \\
\bottomrule
\end{tabularx}

\end{table}

For $p=f_{\mathrm{base}}(x)$, the parameters of each transformation are fitted by minimizing calibration negative log likelihood,
\begin{equation}
\widehat{\mathcal L}_{\mathrm{cal}}(\theta_g)
=-\frac{1}{|\gD_{\mathrm{cal}}|}
\sum_{(x_i,y_i)\in\gD_{\mathrm{cal}}}
\log\left(\left[g_{\theta_g}(f_{\mathrm{base}}(x_i))\right]_{y_i}\right).
\label{eq:global-calibration-fit-v16}
\end{equation}
Vector and Dirichlet calibration additionally use regularization, whose strength is selected on validation data. After fitting each parametric recalibrator, we apply the one-time independent uniform perturbation to its free parameters described in Appendix~\ref{app:randomized-initialization}, using perturbation parameter $\varepsilon_{\mathrm{noise}}=10^{-3}$ and master random seed $20260907$. Let $\widetilde\theta_g$ denote the resulting perturbed parameter. The same perturbed map $g_{\widetilde\theta_g}$ is used for the standalone parametric method and as the initializer of the corresponding two-stage HUC-Boost procedure. We therefore start HUC-Boost from
\begin{equation}
f_g^{(0)}(x)=g_{\widetilde\theta_g}(f_{\mathrm{base}}(x)).
\label{eq:stacked-initial-predictor-v16}
\end{equation}
An arrow denotes this two-stage composition. HUC-Boost uses the step $\alpha=\Gamma/\Lambda$. HUC-Boost (0.05) denotes the comparison with constant step magnitude 0.05.

\paragraph{Iteration budgets.}
UC-Boost uses at most 100 updates. Each HUC-Boost procedure, including the fixed-step and two-stage variants, uses at most 100 passes over the nodes with at least two children, for a node-audit budget of 100 times the number of such nodes. A node is updated only when its audited violation exceeds $0.001$; one-child nodes are excluded. UC-Boost stops when its empirical violation is at most $0.001$, and HUC-Boost stops when a complete pass finds no violation, or when the respective budget is exhausted. Validation selects the returned update prefix.

\paragraph{Utilities, subgroups, and reported quantities.}
Each dataset specifies its target utilities $\gU$ and subgroup family $\gC$. For a fixed prediction $p$, its predicted mean utility is $\mu_u(p)=\sum_{y\in\gY}p_yu(p,y)$; any action or reported node selected from $p$ is held fixed while averaging over the label $y$. Subgroups are represented by binary indicators $c_A(x)=\mathbf{1}\{x\in A\}$, including $c\equiv1$. UC and HUC use the same utility and subgroup families within each comparison. We also report the number of updates retained in the selected predictor.

\subsection[iNaturalist 2019]{iNaturalist 2019~\citep{VanHorn2018}}
\label{app:inat-details-v16}

\paragraph{Data and label hierarchy.}
iNaturalist associates each image with a species label and a biological taxonomy. We use $3{,}344$ images from 30 species, with at most 120 images per species. Each of five data splits contains $1{,}505$ training, 669 calibration, 503 validation, and 667 test images. The species are the leaves of the supplied biological taxonomy, shown in Figure~\ref{fig:inat-full-line-tree}. Let $V_\ell$ denote its nodes at depth $\ell$; depths $3,4,5,6,7$ correspond to class, order, family, genus, and species, respectively.

\begin{figure}[htbp]
\centering
\resizebox{0.85\textwidth}{!}{%
\begin{tikzpicture}[x=1cm,y=1cm,
 every node/.style={font=\footnotesize,inner sep=1pt,fill=white},
 treeedge/.style={draw=black!72,line width=0.60pt,shorten >=1pt,shorten <=1pt}]
\node[] (n0) at (0.000,-5.557) {$\rho$};
\node[] (n1) at (1.250,-3.124) {Animalia};
\node[] (n2) at (3.050,-1.530) {Chordata};
\node[] (n3) at (5.150,-1.530) {Amphibia};
\node[] (n4) at (7.250,-1.530) {Anura};
\node[] (n5) at (9.200,-1.530) {Ranidae};
\node[] (n6) at (10.850,-1.530) {Lithobates};
\node[anchor=west] (n7) at (12.350,0.000) {\textit{L.\ blairi}};
\node[anchor=west] (n8) at (12.350,-0.340) {\textit{L.\ sylvaticus}};
\node[anchor=west] (n9) at (12.350,-0.680) {\textit{L.\ palustris}};
\node[anchor=west] (n10) at (12.350,-1.020) {\textit{L.\ septentrionalis}};
\node[anchor=west] (n11) at (12.350,-1.360) {\textit{L.\ clamitans}};
\node[anchor=west] (n12) at (12.350,-1.700) {\textit{L.\ pipiens}};
\node[anchor=west] (n13) at (12.350,-2.040) {\textit{L.\ sphenocephalus}};
\node[anchor=west] (n14) at (12.350,-2.380) {\textit{L.\ berlandieri}};
\node[anchor=west] (n15) at (12.350,-2.720) {\textit{L.\ grylio}};
\node[anchor=west] (n16) at (12.350,-3.060) {\textit{L.\ catesbeianus}};
\node[] (n17) at (1.250,-7.990) {Fungi};
\node[] (n18) at (3.050,-7.990) {Basidiomycota};
\node[] (n19) at (5.150,-7.990) {Agaricomycetes};
\node[] (n20) at (7.250,-7.990) {Agaricales};
\node[] (n21) at (9.200,-7.990) {Amanitaceae};
\node[] (n22) at (10.850,-7.990) {Amanita};
\node[anchor=west] (n23) at (12.350,-6.120) {\textit{A.\ flavoconia}};
\node[anchor=west] (n24) at (12.350,-6.460) {\textit{A.\ rubescens}};
\node[anchor=west] (n25) at (12.350,-6.800) {\textit{A.\ gemmata}};
\node[anchor=west] (n26) at (12.350,-7.140) {\textit{A.\ phalloides}};
\node[anchor=west] (n27) at (12.350,-7.480) {\textit{A.\ muscaria}};
\node[anchor=west] (n28) at (12.350,-7.820) {\textit{A.\ pantherina}};
\node[anchor=west] (n29) at (12.350,-8.160) {\textit{A.\ calyptroderma}};
\node[anchor=west] (n30) at (12.350,-8.500) {\textit{A.\ bisporigera}};
\node[anchor=west] (n31) at (12.350,-8.840) {\textit{A.\ jacksonii}};
\node[anchor=west] (n32) at (12.350,-9.180) {\textit{A.\ augusta}};
\node[anchor=west] (n33) at (12.350,-9.520) {\textit{A.\ velosa}};
\node[anchor=west] (n34) at (12.350,-9.860) {\textit{A.\ vaginata}};
\node[] (n35) at (3.050,-4.718) {Arthropoda};
\node[] (n36) at (5.150,-4.718) {Insecta};
\node[] (n37) at (7.250,-3.825) {Hymenoptera};
\node[] (n38) at (9.200,-3.570) {Vespidae};
\node[] (n39) at (10.850,-3.570) {Polistes};
\node[anchor=west] (n40) at (12.350,-3.400) {\textit{P.\ carnifex}};
\node[] (n41) at (7.250,-4.760) {Odonata};
\node[] (n42) at (9.000,-4.760) {Coenagrionidae};
\node[] (n43) at (11.050,-4.760) {Enallagma};
\node[anchor=west] (n44) at (12.350,-4.420) {\textit{E.\ aspersum}};
\node[] (n45) at (7.250,-5.610) {Lepidoptera};
\node[] (n46) at (9.200,-5.440) {Nymphalidae};
\node[] (n47) at (10.850,-5.440) {Junonia};
\node[anchor=west] (n48) at (12.350,-5.440) {\textit{J.\ genoveva}};
\node[anchor=west] (n49) at (12.350,-4.760) {\textit{E.\ traviatum}};
\node[anchor=west] (n50) at (12.350,-3.740) {\textit{P.\ rubiginosus}};
\node[] (n51) at (9.200,-4.080) {Apidae};
\node[] (n52) at (10.850,-4.080) {Bombus};
\node[anchor=west] (n53) at (12.350,-4.080) {\textit{B.\ rufocinctus}};
\node[] (n54) at (9.200,-5.780) {Papilionidae};
\node[] (n55) at (10.850,-5.780) {Papilio};
\node[anchor=west] (n56) at (12.350,-5.780) {\textit{P.\ glaucus}};
\node[anchor=west] (n57) at (12.350,-5.100) {\textit{E.\ exsulans}};
\draw[treeedge] (n0.east)--(n1.west);
\draw[treeedge] (n0.east)--(n17.west);
\draw[treeedge] (n1.east)--(n2.west);
\draw[treeedge] (n1.east)--(n35.west);
\draw[treeedge] (n2.east)--(n3.west);
\draw[treeedge] (n3.east)--(n4.west);
\draw[treeedge] (n4.east)--(n5.west);
\draw[treeedge] (n5.east)--(n6.west);
\draw[treeedge] (n6.east)--(n7.west);
\draw[treeedge] (n6.east)--(n8.west);
\draw[treeedge] (n6.east)--(n9.west);
\draw[treeedge] (n6.east)--(n10.west);
\draw[treeedge] (n6.east)--(n11.west);
\draw[treeedge] (n6.east)--(n12.west);
\draw[treeedge] (n6.east)--(n13.west);
\draw[treeedge] (n6.east)--(n14.west);
\draw[treeedge] (n6.east)--(n15.west);
\draw[treeedge] (n6.east)--(n16.west);
\draw[treeedge] (n17.east)--(n18.west);
\draw[treeedge] (n18.east)--(n19.west);
\draw[treeedge] (n19.east)--(n20.west);
\draw[treeedge] (n20.east)--(n21.west);
\draw[treeedge] (n21.east)--(n22.west);
\draw[treeedge] (n22.east)--(n23.west);
\draw[treeedge] (n22.east)--(n24.west);
\draw[treeedge] (n22.east)--(n25.west);
\draw[treeedge] (n22.east)--(n26.west);
\draw[treeedge] (n22.east)--(n27.west);
\draw[treeedge] (n22.east)--(n28.west);
\draw[treeedge] (n22.east)--(n29.west);
\draw[treeedge] (n22.east)--(n30.west);
\draw[treeedge] (n22.east)--(n31.west);
\draw[treeedge] (n22.east)--(n32.west);
\draw[treeedge] (n22.east)--(n33.west);
\draw[treeedge] (n22.east)--(n34.west);
\draw[treeedge] (n35.east)--(n36.west);
\draw[treeedge] (n36.east)--(n37.west);
\draw[treeedge] (n36.east)--(n41.west);
\draw[treeedge] (n36.east)--(n45.west);
\draw[treeedge] (n37.east)--(n38.west);
\draw[treeedge] (n37.east)--(n51.west);
\draw[treeedge] (n38.east)--(n39.west);
\draw[treeedge] (n39.east)--(n40.west);
\draw[treeedge] (n39.east)--(n50.west);
\draw[treeedge] (n41.east)--(n42.west);
\draw[treeedge] (n42.east)--(n43.west);
\draw[treeedge] (n43.east)--(n44.west);
\draw[treeedge] (n43.east)--(n49.west);
\draw[treeedge] (n43.east)--(n57.west);
\draw[treeedge] (n45.east)--(n46.west);
\draw[treeedge] (n45.east)--(n54.west);
\draw[treeedge] (n46.east)--(n47.west);
\draw[treeedge] (n47.east)--(n48.west);
\draw[treeedge] (n51.east)--(n52.west);
\draw[treeedge] (n52.east)--(n53.west);
\draw[treeedge] (n54.east)--(n55.west);
\draw[treeedge] (n55.east)--(n56.west);
\end{tikzpicture}}
\caption{The biological label tree for the selected iNaturalist species. Species labels abbreviate the genus, which is written at their parent node. Every line denotes a parent--child edge.}
\label{fig:inat-full-line-tree}
\end{figure}
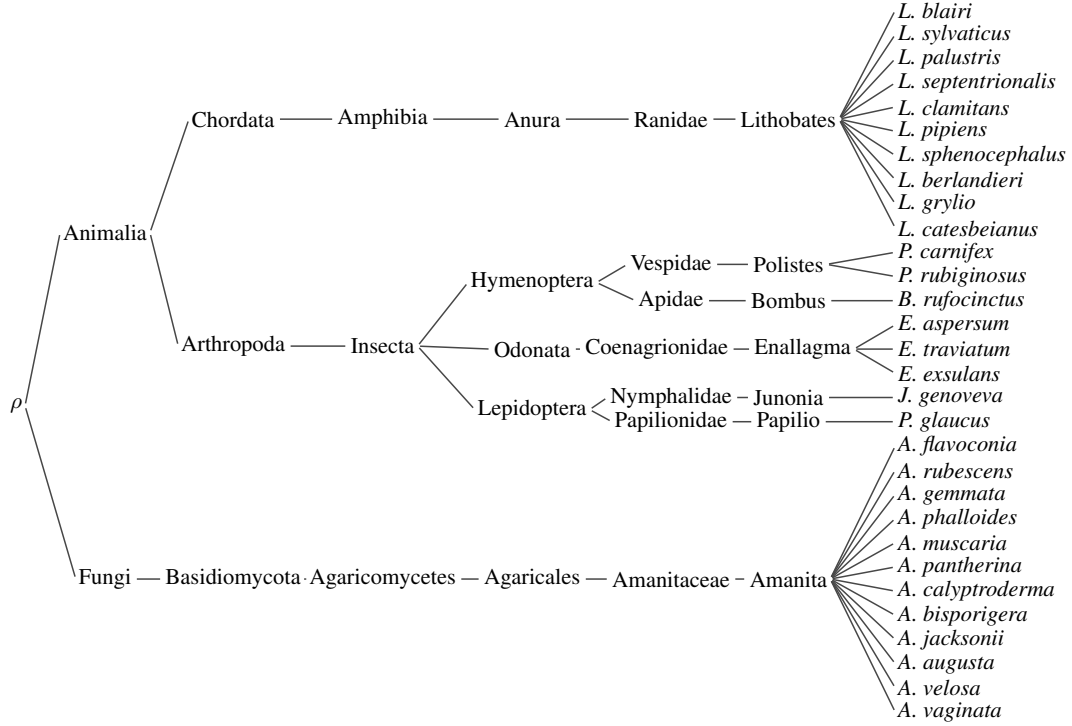

\paragraph{Features.}
Each image is represented by the 512-dimensional feature vector before the final classification layer of an ImageNet-pretrained ResNet-18, which is not fine-tuned on iNaturalist. The six base predictors in Appendix~\ref{subsec:real-common-settings} are trained on these features to predict the 30 species.

\paragraph{Utilities.}
For each node $v$, let $\gY_v$ be the species below $v$ and define
$J_v(y):=\mathbf{1}\{y\in\gY_v\}$, $\pi_v(p):=\sum_{z\in\gY_v}p_z$.
Thus $J_v(y)$ is one if the true species $y$ belongs to the category represented by $v$, and zero otherwise; $\pi_v(p)$ is that category's predicted probability.
At each level $\ell\in\{3,4,5,6,7\}$, select the single node with the largest predicted probability,
\begin{equation}
\widehat v_\ell(p):=\operatorname*{arg\,max}_{\prec;\,v\in V_\ell}\pi_v(p),
\label{eq:inat-rank-top1-selection}
\end{equation}
with ties resolved by the fixed rule in Appendix~\ref{subsec:tie-breaking-convention}. Its utility and predicted mean are
\begin{equation}
\begin{aligned}
u_\ell(p,y)&:=J_{\widehat v_\ell(p)}(y)
=\mathbf{1}\{y\in\gY_{\widehat v_\ell(p)}\},\\
\mu_{u_\ell}(p)&=\sum_{z\in\gY}p_z\mathbf{1}\{z\in\gY_{\widehat v_\ell(p)}\}
=\sum_{z\in\gY_{\widehat v_\ell(p)}}p_z
=\pi_{\widehat v_\ell(p)}(p).
\end{aligned}
\label{eq:inat-rank-top1-utility}
\end{equation}
The sets $\{\gY_v:v\in V_\ell\}$ partition the species, so each true species belongs to exactly one category at level $\ell$. The indicator above is therefore one exactly when the most probable category $\widehat v_\ell(p)$ is the true category at that level. This is top-1 correctness at class, order, family, genus, and species for $\ell=3,4,5,6,7$, respectively. At the species level, $\gY_v=\{v\}$, so the rule reduces to selecting the most probable leaf and checking whether it equals $y$. At higher levels, selection uses the total probability of each category's descendant species. UC and HUC compare this realized 0--1 correctness with the model's probability that the selected category is correct.

\paragraph{Subgroups.}
The five utilities give $|\gU|=5$. There are seven subgroup indicators, $|\gC|=7$: the whole population; brightness below and at least its training median; contrast below and at least its training median; and aspect ratio below one and at least one. Aspect ratio is image width divided by height. Indicators use only these separate attributes, without intersections or other input features.

\paragraph{Results.}
Table~\ref{tab:inat-complete-updated} gives the pooled results for the five utilities. HUC-Boost gives lower mean HUC than the fixed-step correction. The fixed-step correction retains more updates and has higher mean UC and lower mean AUC.

\begin{table}[!htbp]
\caption{Complete iNaturalist results for the five rank-wise correctness utilities. Values are mean $\pm$ sample standard deviation over six base predictors and five splits. UC-Boost and HUC-Boost are abbreviated as UC-B and HUC-B, respectively.}
\label{tab:inat-complete-updated}
\centering
\begingroup
\scriptsize
\setlength{\tabcolsep}{3pt}
\renewcommand{\arraystretch}{1.0}
\begin{tabular}{lccccc}
\toprule
Method & Acc. $\uparrow$ & AUC $\uparrow$ & UC $\downarrow$ & HUC $\downarrow$ & Updates \\
\midrule
Baseline & .4886$\pm$.0895 & .9001$\pm$.1082 & .2721$\pm$.1615 & .1251$\pm$.0595 & 0.00$\pm$0.00 \\
UC-B & .4618$\pm$.0877 & .9138$\pm$.0664 & \textbf{.0452$\pm$.0177} & .0361$\pm$.0108 & 25.43$\pm$24.07 \\
HUC-B & .4633$\pm$.0896 & .9223$\pm$.0599 & .0748$\pm$.0317 & .0345$\pm$.0090 & 145.93$\pm$180.47 \\
HUC-B (0.05) & .4853$\pm$.0888 & .8971$\pm$.1106 & .1837$\pm$.1767 & .0815$\pm$.0725 & 329.23$\pm$270.41 \\
Temp. & .4886$\pm$.0895 & .8902$\pm$.1097 & .1303$\pm$.1143 & .0920$\pm$.0756 & 0.00$\pm$0.00 \\
Temp.$\to$HUC & .4702$\pm$.0878 & \textbf{.9301$\pm$.0501} & .0473$\pm$.0135 & \textbf{.0289$\pm$.0081} & 43.07$\pm$25.92 \\
Vec. & \textbf{.4888$\pm$.0952} & .8986$\pm$.0999 & .1144$\pm$.0922 & .0827$\pm$.0726 & 0.00$\pm$0.00 \\
Vec.$\to$HUC & .4781$\pm$.0950 & .9291$\pm$.0528 & .0496$\pm$.0176 & .0313$\pm$.0079 & 28.23$\pm$22.38 \\
Dir. & .4365$\pm$.1125 & .9119$\pm$.0626 & .1047$\pm$.0532 & .0551$\pm$.0230 & 0.00$\pm$0.00 \\
Dir.$\to$HUC & .4349$\pm$.1116 & .9120$\pm$.0623 & .1093$\pm$.0523 & .0556$\pm$.0227 & 4.27$\pm$10.53 \\
\bottomrule
\end{tabular}
\endgroup
\end{table}

\paragraph{Results by base predictor.}
Tables~\ref{tab:inat-base-baseline-v2} and~\ref{tab:inat-base-huc-v2} separate all procedures by base predictor.  HUC-Boost gives a much smaller HUC error than HUC-Boost (0.05) for LR, GNB, and DT, which are the principal sources of the pooled difference.  The fixed-step method uses more updates for most base predictors; DT is the exception because its fixed path stops with far fewer selected updates but also has much larger HUC.  RF and XGB provide small exceptions in which the fixed step has a slightly smaller HUC, so the aggregate conclusion does not imply the same ordering for every base predictor.

\begin{table}[!htbp]
\caption{iNaturalist (all utilities): baseline methods by base predictor. Values are mean $\pm$ sample standard deviation over five data splits. UC-Boost and HUC-Boost are abbreviated as UC-B and HUC-B, respectively.}
\label{tab:inat-base-baseline-v2}
\centering
\begingroup
\scriptsize
\setlength{\tabcolsep}{1.0pt}
\renewcommand{\arraystretch}{1.0}
\resizebox{\linewidth}{!}{%
\begin{tabular}{@{}l*{5}{c}@{\hspace{5pt}}*{5}{c}@{}}
\toprule
& \multicolumn{5}{c}{\textbf{LR}} & \multicolumn{5}{c}{\textbf{RF}}\\
\cmidrule(lr){2-6}\cmidrule(lr){7-11}
Method & Acc. $\uparrow$ & AUC $\uparrow$ & UC $\downarrow$ & HUC $\downarrow$ & Updates & Acc. $\uparrow$ & AUC $\uparrow$ & UC $\downarrow$ & HUC $\downarrow$ & Updates\\
\midrule
Baseline & .5436$\pm$.0222 & .9522$\pm$.0028 & .2439$\pm$.0136 & .1225$\pm$.0066 & 0.00$\pm$0.00 & .5187$\pm$.0081 & .9509$\pm$.0027 & .2957$\pm$.0089 & .1246$\pm$.0089 & 0.00$\pm$0.00\\
UC-B & .4786$\pm$.0206 & .9531$\pm$.0031 & .0721$\pm$.0135 & .0536$\pm$.0043 & 9.80$\pm$3.42 & .5178$\pm$.0069 & .9550$\pm$.0027 & .0299$\pm$.0037 & .0328$\pm$.0067 & 18.80$\pm$7.46\\
Temp. & .5436$\pm$.0222 & .9537$\pm$.0030 & .0704$\pm$.0072 & .0574$\pm$.0073 & 0.00$\pm$0.00 & .5187$\pm$.0081 & .9572$\pm$.0022 & \textbf{.0285$\pm$.0043} & .0249$\pm$.0049 & 0.00$\pm$0.00\\
Vec. & \textbf{.5460$\pm$.0192} & \textbf{.9556$\pm$.0028} & \textbf{.0639$\pm$.0082} & \textbf{.0379$\pm$.0060} & 0.00$\pm$0.00 & \textbf{.5220$\pm$.0155} & \textbf{.9573$\pm$.0021} & .0390$\pm$.0184 & \textbf{.0215$\pm$.0053} & 0.00$\pm$0.00\\
Dir. & .5055$\pm$.0115 & .9521$\pm$.0028 & .1810$\pm$.0101 & .0891$\pm$.0098 & 0.00$\pm$0.00 & .2984$\pm$.0157 & .8933$\pm$.0067 & .1638$\pm$.0076 & .0769$\pm$.0067 & 0.00$\pm$0.00\\
\midrule
& \multicolumn{5}{c}{\textbf{GNB}} & \multicolumn{5}{c}{\textbf{XGB}}\\
\cmidrule(lr){2-6}\cmidrule(lr){7-11}
Method & Acc. $\uparrow$ & AUC $\uparrow$ & UC $\downarrow$ & HUC $\downarrow$ & Updates & Acc. $\uparrow$ & AUC $\uparrow$ & UC $\downarrow$ & HUC $\downarrow$ & Updates\\
\midrule
Baseline & .5244$\pm$.0213 & .9245$\pm$.0107 & .4512$\pm$.0212 & .1977$\pm$.0164 & 0.00$\pm$0.00 & .5145$\pm$.0148 & .9526$\pm$.0031 & .0778$\pm$.0105 & .0525$\pm$.0061 & 0.00$\pm$0.00\\
UC-B & .4774$\pm$.0181 & .8791$\pm$.0114 & \textbf{.0533$\pm$.0098} & \textbf{.0400$\pm$.0119} & 34.60$\pm$2.07 & .5040$\pm$.0090 & .9559$\pm$.0022 & .0420$\pm$.0125 & .0298$\pm$.0050 & 11.20$\pm$4.76\\
Temp. & .5244$\pm$.0213 & .8582$\pm$.0161 & .2089$\pm$.0098 & .1426$\pm$.0087 & 0.00$\pm$0.00 & .5145$\pm$.0148 & .9525$\pm$.0031 & .0814$\pm$.0125 & .0526$\pm$.0078 & 0.00$\pm$0.00\\
Vec. & .5253$\pm$.0161 & .8807$\pm$.0132 & .1787$\pm$.0120 & .1355$\pm$.0090 & 0.00$\pm$0.00 & \textbf{.5178$\pm$.0118} & .9526$\pm$.0034 & .0732$\pm$.0073 & .0494$\pm$.0078 & 0.00$\pm$0.00\\
Dir. & \textbf{.5277$\pm$.0184} & \textbf{.9298$\pm$.0080} & .0754$\pm$.0144 & .0444$\pm$.0097 & 0.00$\pm$0.00 & .5127$\pm$.0090 & \textbf{.9573$\pm$.0032} & \textbf{.0417$\pm$.0051} & \textbf{.0269$\pm$.0053} & 0.00$\pm$0.00\\
\midrule
& \multicolumn{5}{c}{\textbf{DT}} & \multicolumn{5}{c}{\textbf{MLP}}\\
\cmidrule(lr){2-6}\cmidrule(lr){7-11}
Method & Acc. $\uparrow$ & AUC $\uparrow$ & UC $\downarrow$ & HUC $\downarrow$ & Updates & Acc. $\uparrow$ & AUC $\uparrow$ & UC $\downarrow$ & HUC $\downarrow$ & Updates\\
\midrule
Baseline & \textbf{.2960$\pm$.0181} & .6637$\pm$.0096 & .4778$\pm$.0188 & .1948$\pm$.0110 & 0.00$\pm$0.00 & .5343$\pm$.0142 & .9565$\pm$.0026 & .0860$\pm$.0513 & .0587$\pm$.0176 & 0.00$\pm$0.00\\
UC-B & .2750$\pm$.0167 & .7825$\pm$.0115 & \textbf{.0437$\pm$.0152} & \textbf{.0337$\pm$.0049} & 70.60$\pm$20.31 & .5178$\pm$.0143 & .9571$\pm$.0033 & \textbf{.0304$\pm$.0051} & \textbf{.0267$\pm$.0022} & 7.60$\pm$5.22\\
Temp. & \textbf{.2960$\pm$.0181} & .6629$\pm$.0096 & .3446$\pm$.0140 & .2348$\pm$.0075 & 0.00$\pm$0.00 & .5343$\pm$.0142 & .9567$\pm$.0027 & .0478$\pm$.0077 & .0397$\pm$.0044 & 0.00$\pm$0.00\\
Vec. & .2828$\pm$.0107 & .6883$\pm$.0089 & .2865$\pm$.0163 & .2175$\pm$.0166 & 0.00$\pm$0.00 & \textbf{.5388$\pm$.0114} & \textbf{.9572$\pm$.0026} & .0451$\pm$.0093 & .0342$\pm$.0070 & 0.00$\pm$0.00\\
Dir. & .2666$\pm$.0280 & \textbf{.7840$\pm$.0082} & .0971$\pm$.0119 & .0481$\pm$.0109 & 0.00$\pm$0.00 & .5079$\pm$.0140 & .9549$\pm$.0022 & .0695$\pm$.0267 & .0455$\pm$.0090 & 0.00$\pm$0.00\\
\bottomrule
\end{tabular}%
}
\endgroup
\end{table}
\begin{table}[!htbp]
\caption{iNaturalist (all utilities): HUC-Boost results by base predictor. Values are mean $\pm$ sample standard deviation over five data splits. UC-Boost and HUC-Boost are abbreviated as UC-B and HUC-B, respectively.}
\label{tab:inat-base-huc-v2}
\centering
\begingroup
\scriptsize
\setlength{\tabcolsep}{1.0pt}
\renewcommand{\arraystretch}{1.0}
\resizebox{\linewidth}{!}{%
\begin{tabular}{@{}l*{5}{c}@{\hspace{5pt}}*{5}{c}@{}}
\toprule
& \multicolumn{5}{c}{\textbf{LR}} & \multicolumn{5}{c}{\textbf{RF}}\\
\cmidrule(lr){2-6}\cmidrule(lr){7-11}
Method & Acc. $\uparrow$ & AUC $\uparrow$ & UC $\downarrow$ & HUC $\downarrow$ & Updates & Acc. $\uparrow$ & AUC $\uparrow$ & UC $\downarrow$ & HUC $\downarrow$ & Updates\\
\midrule
HUC-B & .5094$\pm$.0049 & .9550$\pm$.0027 & .0876$\pm$.0050 & .0354$\pm$.0048 & 41.80$\pm$7.09 & .5034$\pm$.0094 & .9513$\pm$.0030 & .0429$\pm$.0116 & .0286$\pm$.0034 & 79.20$\pm$21.49\\
HUC-B (0.05) & .5286$\pm$.0144 & .9559$\pm$.0031 & .1124$\pm$.0198 & .0446$\pm$.0081 & 400.20$\pm$62.97 & .5196$\pm$.0063 & .9544$\pm$.0021 & .0410$\pm$.0081 & .0206$\pm$.0033 & 385.20$\pm$70.12\\
Temp.$\to$HUC & .5181$\pm$.0093 & .9560$\pm$.0025 & .0584$\pm$.0114 & \textbf{.0336$\pm$.0090} & 46.00$\pm$31.02 & .5070$\pm$.0116 & .9564$\pm$.0028 & .0315$\pm$.0072 & \textbf{.0204$\pm$.0007} & 29.60$\pm$22.13\\
Vec.$\to$HUC & \textbf{.5406$\pm$.0104} & \textbf{.9566$\pm$.0028} & \textbf{.0551$\pm$.0109} & .0349$\pm$.0055 & 17.80$\pm$14.86 & \textbf{.5226$\pm$.0122} & \textbf{.9571$\pm$.0022} & \textbf{.0298$\pm$.0024} & .0232$\pm$.0073 & 12.00$\pm$12.90\\
Dir.$\to$HUC & .5055$\pm$.0115 & .9521$\pm$.0028 & .1810$\pm$.0101 & .0891$\pm$.0098 & 0.00$\pm$0.00 & .2984$\pm$.0157 & .8933$\pm$.0067 & .1638$\pm$.0076 & .0769$\pm$.0067 & 0.00$\pm$0.00\\
\midrule
& \multicolumn{5}{c}{\textbf{GNB}} & \multicolumn{5}{c}{\textbf{XGB}}\\
\cmidrule(lr){2-6}\cmidrule(lr){7-11}
Method & Acc. $\uparrow$ & AUC $\uparrow$ & UC $\downarrow$ & HUC $\downarrow$ & Updates & Acc. $\uparrow$ & AUC $\uparrow$ & UC $\downarrow$ & HUC $\downarrow$ & Updates\\
\midrule
HUC-B & .4894$\pm$.0244 & .9225$\pm$.0092 & .1108$\pm$.0140 & .0426$\pm$.0074 & 190.00$\pm$23.82 & .4948$\pm$.0133 & .9545$\pm$.0026 & .0586$\pm$.0089 & .0360$\pm$.0038 & 30.80$\pm$13.88\\
HUC-B (0.05) & \textbf{.5280$\pm$.0228} & .8986$\pm$.0262 & .3707$\pm$.0226 & .1644$\pm$.0181 & 751.00$\pm$57.04 & .5106$\pm$.0210 & .9557$\pm$.0025 & .0528$\pm$.0138 & .0319$\pm$.0067 & 152.20$\pm$72.97\\
Temp.$\to$HUC & .5013$\pm$.0223 & \textbf{.9349$\pm$.0069} & \textbf{.0598$\pm$.0080} & \textbf{.0345$\pm$.0086} & 58.80$\pm$18.36 & .4978$\pm$.0124 & .9545$\pm$.0027 & .0563$\pm$.0045 & .0338$\pm$.0076 & 34.60$\pm$11.74\\
Vec.$\to$HUC & .5124$\pm$.0309 & .9345$\pm$.0071 & .0660$\pm$.0241 & .0362$\pm$.0071 & 43.40$\pm$13.03 & .4939$\pm$.0109 & .9544$\pm$.0022 & .0577$\pm$.0163 & .0378$\pm$.0057 & 31.80$\pm$8.76\\
Dir.$\to$HUC & .5190$\pm$.0232 & .9298$\pm$.0082 & .0928$\pm$.0185 & .0439$\pm$.0051 & 14.80$\pm$11.86 & \textbf{.5127$\pm$.0090} & \textbf{.9573$\pm$.0032} & \textbf{.0417$\pm$.0051} & \textbf{.0269$\pm$.0053} & 0.00$\pm$0.00\\
\midrule
& \multicolumn{5}{c}{\textbf{DT}} & \multicolumn{5}{c}{\textbf{MLP}}\\
\cmidrule(lr){2-6}\cmidrule(lr){7-11}
Method & Acc. $\uparrow$ & AUC $\uparrow$ & UC $\downarrow$ & HUC $\downarrow$ & Updates & Acc. $\uparrow$ & AUC $\uparrow$ & UC $\downarrow$ & HUC $\downarrow$ & Updates\\
\midrule
HUC-B & .2699$\pm$.0158 & .7936$\pm$.0047 & .1006$\pm$.0236 & .0392$\pm$.0060 & 507.00$\pm$126.11 & .5127$\pm$.0225 & .9567$\pm$.0022 & .0484$\pm$.0339 & .0252$\pm$.0136 & 26.80$\pm$11.90\\
HUC-B (0.05) & \textbf{.2939$\pm$.0210} & .6599$\pm$.0163 & .4741$\pm$.0119 & .1955$\pm$.0123 & 172.80$\pm$386.39 & \textbf{.5310$\pm$.0130} & \textbf{.9581$\pm$.0024} & .0514$\pm$.0268 & .0319$\pm$.0112 & 114.00$\pm$56.09\\
Temp.$\to$HUC & .2810$\pm$.0198 & \textbf{.8216$\pm$.0067} & \textbf{.0386$\pm$.0069} & \textbf{.0261$\pm$.0022} & 68.20$\pm$29.15 & .5160$\pm$.0230 & .9570$\pm$.0027 & \textbf{.0389$\pm$.0081} & .0247$\pm$.0067 & 21.20$\pm$8.32\\
Vec.$\to$HUC & .2747$\pm$.0180 & .8149$\pm$.0071 & .0493$\pm$.0109 & .0316$\pm$.0042 & 52.40$\pm$33.13 & .5241$\pm$.0156 & .9574$\pm$.0027 & .0400$\pm$.0094 & \textbf{.0245$\pm$.0052} & 12.00$\pm$8.40\\
Dir.$\to$HUC & .2660$\pm$.0290 & .7847$\pm$.0085 & .1068$\pm$.0193 & .0515$\pm$.0114 & 10.80$\pm$19.51 & .5079$\pm$.0140 & .9549$\pm$.0022 & .0695$\pm$.0267 & .0455$\pm$.0090 & 0.00$\pm$0.00\\
\bottomrule
\end{tabular}%
}
\endgroup
\end{table}

\subsection[SUPPORT2]{SUPPORT2~\citep{Knaus1995SUPPORT}}
\label{app:support2-details-v20}

\paragraph{Data and label hierarchy.}
SUPPORT2 contains 9,105 seriously ill hospitalized adults
from five centers.\footnote{
Data obtained from \url{http://hbiostat.org/data}
courtesy of the Vanderbilt University Department of Biostatistics.} We use a stratified subset of $5{,}000$ patients with an observed two-month functional outcome. The five outcome categories define the leaves $y_1,\ldots,y_5$. Each of five data splits contains $2{,}250$ training, $1{,}001$ calibration, 750 validation, and 999 test examples.

The root $v_0$ separates survival from death within two months. Among survivors, $v_1$ separates milder from major limitations. Node $v_2$ separates $y_1$ and $y_2$, and $v_3$ separates $y_3$ and $y_4$. Thus $V^\circ=\{v_0,v_1,v_2,v_3\}$.

\begin{figure}[htbp]
\centering
\begin{tikzpicture}[x=0.85cm,y=0.75cm,
 every node/.style={font=\small,inner sep=1pt,fill=white},
 treeedge/.style={draw=black!72,line width=0.60pt,shorten >=1.6pt,shorten <=1.6pt}]
\node (v0) at (0.9,0) {$v_0$};
\node (v1) at (-0.9,-0.85) {$v_1$};
\node (y5) at (2.7,-0.85) {$y_5$};
\node (v2) at (-2.3,-1.7) {$v_2$};
\node (v3) at (0.5,-1.7) {$v_3$};
\node (y1) at (-3.0,-2.55) {$y_1$};
\node (y2) at (-1.6,-2.55) {$y_2$};
\node (y3) at (-0.2,-2.55) {$y_3$};
\node (y4) at (1.2,-2.55) {$y_4$};
\foreach \a/\b in {v0/v1,v0/y5,v1/v2,v1/v3,v2/y1,v2/y2,v3/y3,v3/y4}
 {\draw[treeedge] (\a)--(\b);}
\end{tikzpicture}
\caption{SUPPORT2 outcome tree. The leaves are: $y_1$, no moderate-or-greater limitation; $y_2$, assistance with daily activities; $y_3$, major functional limitation; $y_4$, intubation or coma; $y_5$, death within two months. The hierarchy and utilities are synthetic, not clinical recommendations.}
\label{fig:support2-tree-v20}
\end{figure}
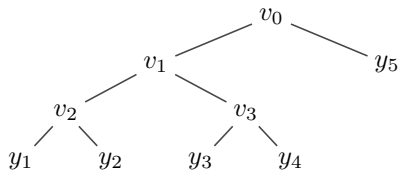

\paragraph{Features.}
There are 30 input variables before categorical expansion: 24 numerical variables and six categorical variables. The categorical variables are sex, disease group, disease class, income, race, and cancer status. The numerical variables describe patient background and baseline physiological and laboratory measurements. Outcome variables, post-outcome variables, costs, and precomputed prognostic scores are excluded. Numerical variables are standardized after median imputation and are accompanied by missing-value indicators; categorical variables are represented by indicators. The resulting indicators and numerical variables form the classifier input; the output is a probability vector over five leaves.

The six base predictors in Appendix~\ref{subsec:real-common-settings} are trained on these features to predict the five leaf labels.

\paragraph{Utilities.}
\label{app:support2-utility-v20}
These utilities use the same binary action-comparison construction as the Toy decision utilities. The interpretation of the outcome sets, thresholds, and action payoffs is given in Appendix~\ref{subsec:toy-decision-utility-interpretation}.
For decisions $j\in\{1,2,3\}$, let $D_j$ be the leaves on which the decision is active and $S_j\subseteq D_j$ the leaves for which action 1 is preferred:
\begin{equation}
\begin{array}{c|cc|l}
 j & D_j & S_j & \text{two actions}\\
\hline
1 & \{y_1,y_2,y_3,y_4,y_5\} & \{y_5\}
  & \text{prepare for death within two months / usual response}\\
2 & \{y_1,y_2\} & \{y_2\}
  & \text{prepare assistance with daily living / usual follow-up}\\
3 & \{y_3,y_4\} & \{y_4\}
  & \text{prepare intensive support / usual support}
\end{array}
\label{eq:support2-decision-sets-v20}
\end{equation}
For $\tau\in\Theta:=\{0.25,0.50,0.75\}$, define
\begin{equation}
d_j^\tau(y)=\mathbf{1}\{y\in S_j\}-\tau\mathbf{1}\{y\in D_j\},
\qquad
U_j^\tau(a,y)=(2a-1)d_j^\tau(y),\quad a\in\{0,1\},
\end{equation}
and let
\begin{align}
\widehat a_j^\tau(p)
&:=\operatorname*{arg\,max}_{\prec;\,a\in\{0,1\}}
\sum_{y\in\gY}p_yU_j^\tau(a,y),\\
u_j^\tau(p,y)&:=U_j^\tau\!\left(\widehat a_j^\tau(p),y\right).
\end{align}
The fixed order $\prec$ resolves ties. These synthetic utilities encode three action comparisons at three parameter values, not treatment effects or clinical guidelines.

For $A\subseteq\gY$, write $p(A):=\sum_{z\in A}p_z$. Averaging the selected action's payoff gives
\[
\mu_{u_j^\tau}(p) =\sum_{z\in\gY}p_z u_j^\tau(p,z) =\{2\widehat a_j^\tau(p)-1\}\{p(S_j)-\tau p(D_j)\} =\left|p(S_j)-\tau p(D_j)\right|.
\]
The last equality holds because $\widehat a_j^\tau(p)$ maximizes the predicted expected payoff. Thus the three decisions have predicted means
\[
\mu_{u_1^\tau}(p)=|p_{y_5}-\tau|, \qquad \mu_{u_2^\tau}(p)=|p_{y_2}-\tau(p_{y_1}+p_{y_2})|, \qquad \mu_{u_3^\tau}(p)=|p_{y_4}-\tau(p_{y_3}+p_{y_4})|.
\]
These are expected payoffs, not probabilities of correct classification. The average is over all five outcomes; for decisions 2 and 3, outcomes outside $D_j$ contribute zero, without division by $p(D_j)$.

The main comparison uses all nine utilities,
$\gU=\gU^{(3)}:=\{u_j^\tau:j\in\{1,2,3\},\ \tau\in\Theta\}$,
whose relevant internal nodes cover $\{v_0,v_1,v_2,v_3\}$.

The SUPPORT2 results in the main-text Table~\ref{tab:real-data-compact-main} use $\gU=\gU^{(3)}$: all three decisions at all three thresholds are included in both correction and UC/HUC evaluation.

\paragraph{Subgroups.}
\label{par:support2-subgroup-indicators}
Each subgroup is represented by $c_A(x)=\mathbf{1}\{x\in A\}$. There are ten candidate indicators: one for the whole population, three for age (below 65, 65 to below 80, and at least 80 years), two for sex, and four for disease class (ARF/MOSF, COPD/CHF/cirrhosis, coma, and cancer). The indicators are evaluated from the corresponding patient attributes; neither predicted labels nor prediction errors define membership. These groups may overlap, but their intersections are not added. Missing age values belong to none of the three age groups.

The whole-population indicator is always included. Each other indicator is retained when its group contains at least five observations in both calibration and validation. Thus the retained family satisfies $|\gC|\le10$. The same retained family is used by all methods and utility classes within a split, with no further selection on test data.

\subsubsection{Main Results}
\label{subsec:support2-main-results}
Table~\ref{tab:support2-complete-updated} gives the pooled results for $\gU^{(3)}$ over the six base predictors and five data splits. This is the nine-utility SUPPORT2 setting used in Table~\ref{tab:real-data-compact-main}; the present table additionally includes HUC-Boost (0.05) as a supplementary step-size comparison. HUC-Boost reduces mean HUC relative to Baseline and gives lower mean HUC than the fixed-step correction. The fixed-step correction retains more updates and has higher mean UC and lower mean Accuracy.

\begin{table}[!htbp]
\caption{Complete SUPPORT2 results for the nine-utility class $\gU^{(3)}$. Values are mean $\pm$ sample standard deviation over six base predictors and five data splits. UC-Boost and HUC-Boost are abbreviated as UC-B and HUC-B, respectively.}
\label{tab:support2-complete-updated}
\centering
\begingroup
\scriptsize
\setlength{\tabcolsep}{3pt}
\renewcommand{\arraystretch}{1.0}
\begin{tabular}{lccccc}
\toprule
Method & Acc. $\uparrow$ & AUC $\uparrow$ & UC $\downarrow$ & HUC $\downarrow$ & Updates \\
\midrule
Baseline & .4990$\pm$.1994 & .6763$\pm$.0828 & .1671$\pm$.1777 & .1481$\pm$.1422 & 0.00$\pm$0.00 \\
UC-B & .5746$\pm$.0506 & .6813$\pm$.0789 & .0394$\pm$.0139 & .0504$\pm$.0313 & 33.90$\pm$39.40 \\
HUC-B & .5754$\pm$.0453 & .6985$\pm$.0600 & .0364$\pm$.0101 & .0351$\pm$.0089 & 45.20$\pm$46.89 \\
HUC-B (0.05) & .5023$\pm$.1967 & .6844$\pm$.0787 & .1466$\pm$.1828 & .1270$\pm$.1444 & 198.80$\pm$148.38 \\
Temp. & .4990$\pm$.1994 & .6865$\pm$.0722 & .0769$\pm$.0692 & .0757$\pm$.0693 & 0.00$\pm$0.00 \\
Temp.$\to$HUC & .5872$\pm$.0319 & .7031$\pm$.0568 & .0324$\pm$.0064 & .0323$\pm$.0062 & 20.40$\pm$22.61 \\
Vec. & .5750$\pm$.0532 & .6907$\pm$.0654 & .0380$\pm$.0154 & .0380$\pm$.0154 & 0.00$\pm$0.00 \\
Vec.$\to$HUC & .5886$\pm$.0322 & .7005$\pm$.0505 & \textbf{.0314$\pm$.0081} & \textbf{.0309$\pm$.0079} & 14.27$\pm$12.55 \\
Dir. & .5794$\pm$.0457 & .7125$\pm$.0657 & .0351$\pm$.0129 & .0351$\pm$.0128 & 0.00$\pm$0.00 \\
Dir.$\to$HUC & \textbf{.5889$\pm$.0322} & \textbf{.7181$\pm$.0500} & .0323$\pm$.0077 & .0322$\pm$.0072 & 13.27$\pm$14.08 \\
\bottomrule
\end{tabular}
\endgroup
\end{table}

\paragraph{Results by base predictor.}
Tables~\ref{tab:support-u3-base-baseline-v2} and~\ref{tab:support-u3-base-huc-v2} give the full base-predictor results for the nine-utility class used in the main text.  The fixed step is especially poor for GNB and DT and uses more updates than HUC-Boost for every base predictor.  RF and XGB give small base-specific exceptions in HUC, but these exceptions do not change the pooled conclusion that the fixed step is inferior on SUPPORT2.

\begin{table}[!htbp]
\caption{SUPPORT2 ($\gU^{(3)}$): baseline methods by base predictor. Values are mean $\pm$ sample standard deviation over five data splits. UC-Boost and HUC-Boost are abbreviated as UC-B and HUC-B, respectively.}
\label{tab:support-u3-base-baseline-v2}
\centering
\begingroup
\scriptsize
\setlength{\tabcolsep}{1.0pt}
\renewcommand{\arraystretch}{1.0}
\resizebox{\linewidth}{!}{%
\begin{tabular}{@{}l*{5}{c}@{\hspace{5pt}}*{5}{c}@{}}
\toprule
& \multicolumn{5}{c}{\textbf{LR}} & \multicolumn{5}{c}{\textbf{RF}}\\
\cmidrule(lr){2-6}\cmidrule(lr){7-11}
Method & Acc. $\uparrow$ & AUC $\uparrow$ & UC $\downarrow$ & HUC $\downarrow$ & Updates & Acc. $\uparrow$ & AUC $\uparrow$ & UC $\downarrow$ & HUC $\downarrow$ & Updates\\
\midrule
Baseline & .6088$\pm$.0125 & .7621$\pm$.0218 & .0385$\pm$.0075 & .0385$\pm$.0075 & 0.00$\pm$0.00 & \textbf{.6152$\pm$.0053} & .7462$\pm$.0289 & .0473$\pm$.0083 & .0473$\pm$.0083 & 0.00$\pm$0.00\\
UC-B & .6088$\pm$.0120 & .7614$\pm$.0210 & .0255$\pm$.0019 & .0255$\pm$.0019 & 3.80$\pm$2.17 & \textbf{.6152$\pm$.0128} & .7436$\pm$.0289 & .0350$\pm$.0104 & .0350$\pm$.0104 & 7.60$\pm$10.41\\
Temp. & .6088$\pm$.0125 & .7621$\pm$.0216 & .0284$\pm$.0055 & .0268$\pm$.0067 & 0.00$\pm$0.00 & \textbf{.6152$\pm$.0053} & .7442$\pm$.0284 & .0407$\pm$.0083 & .0407$\pm$.0083 & 0.00$\pm$0.00\\
Vec. & .6094$\pm$.0119 & .7633$\pm$.0217 & .0269$\pm$.0035 & .0269$\pm$.0035 & 0.00$\pm$0.00 & .6140$\pm$.0038 & .7286$\pm$.0394 & \textbf{.0306$\pm$.0071} & \textbf{.0306$\pm$.0071} & 0.00$\pm$0.00\\
Dir. & \textbf{.6100$\pm$.0120} & \textbf{.7677$\pm$.0182} & \textbf{.0251$\pm$.0028} & \textbf{.0251$\pm$.0028} & 0.00$\pm$0.00 & .6124$\pm$.0087 & \textbf{.7475$\pm$.0133} & .0310$\pm$.0095 & .0310$\pm$.0095 & 0.00$\pm$0.00\\
\midrule
& \multicolumn{5}{c}{\textbf{GNB}} & \multicolumn{5}{c}{\textbf{XGB}}\\
\cmidrule(lr){2-6}\cmidrule(lr){7-11}
Method & Acc. $\uparrow$ & AUC $\uparrow$ & UC $\downarrow$ & HUC $\downarrow$ & Updates & Acc. $\uparrow$ & AUC $\uparrow$ & UC $\downarrow$ & HUC $\downarrow$ & Updates\\
\midrule
Baseline & .0701$\pm$.0131 & .5729$\pm$.0136 & .5091$\pm$.0272 & .3978$\pm$.0098 & 0.00$\pm$0.00 & .6036$\pm$.0066 & .7406$\pm$.0207 & .0993$\pm$.0088 & .0993$\pm$.0088 & 0.00$\pm$0.00\\
UC-B & \textbf{.4987$\pm$.0361} & .5828$\pm$.0288 & .0603$\pm$.0051 & .1099$\pm$.0144 & 98.20$\pm$2.39 & .6034$\pm$.0054 & .7512$\pm$.0132 & .0354$\pm$.0092 & .0354$\pm$.0092 & 16.40$\pm$7.64\\
Temp. & .0701$\pm$.0131 & .6206$\pm$.0250 & .2101$\pm$.0046 & .2101$\pm$.0046 & 0.00$\pm$0.00 & .6036$\pm$.0066 & .7446$\pm$.0198 & .0316$\pm$.0086 & .0311$\pm$.0092 & 0.00$\pm$0.00\\
Vec. & .4821$\pm$.0440 & \textbf{.6289$\pm$.0263} & .0605$\pm$.0138 & .0605$\pm$.0138 & 0.00$\pm$0.00 & .6086$\pm$.0093 & .7466$\pm$.0154 & .0307$\pm$.0094 & .0307$\pm$.0094 & 0.00$\pm$0.00\\
Dir. & .4985$\pm$.0239 & .6142$\pm$.0239 & \textbf{.0559$\pm$.0149} & \textbf{.0559$\pm$.0149} & 0.00$\pm$0.00 & \textbf{.6118$\pm$.0085} & \textbf{.7637$\pm$.0152} & \textbf{.0292$\pm$.0065} & \textbf{.0297$\pm$.0059} & 0.00$\pm$0.00\\
\midrule
& \multicolumn{5}{c}{\textbf{DT}} & \multicolumn{5}{c}{\textbf{MLP}}\\
\cmidrule(lr){2-6}\cmidrule(lr){7-11}
Method & Acc. $\uparrow$ & AUC $\uparrow$ & UC $\downarrow$ & HUC $\downarrow$ & Updates & Acc. $\uparrow$ & AUC $\uparrow$ & UC $\downarrow$ & HUC $\downarrow$ & Updates\\
\midrule
Baseline & .5007$\pm$.0194 & .5820$\pm$.0165 & .2745$\pm$.0275 & .2745$\pm$.0275 & 0.00$\pm$0.00 & .5956$\pm$.0174 & .6538$\pm$.0444 & .0337$\pm$.0088 & .0314$\pm$.0075 & 0.00$\pm$0.00\\
UC-B & .5209$\pm$.0197 & .5979$\pm$.0225 & .0501$\pm$.0073 & .0667$\pm$.0080 & 71.00$\pm$27.46 & \textbf{.6008$\pm$.0180} & .6509$\pm$.0375 & .0299$\pm$.0057 & .0301$\pm$.0053 & 6.40$\pm$10.57\\
Temp. & .5007$\pm$.0194 & .5939$\pm$.0225 & .1203$\pm$.0097 & .1168$\pm$.0128 & 0.00$\pm$0.00 & .5956$\pm$.0174 & .6537$\pm$.0443 & .0305$\pm$.0073 & .0286$\pm$.0031 & 0.00$\pm$0.00\\
Vec. & .5407$\pm$.0188 & .6104$\pm$.0326 & .0515$\pm$.0079 & .0515$\pm$.0079 & 0.00$\pm$0.00 & .5950$\pm$.0201 & .6664$\pm$.0337 & \textbf{.0276$\pm$.0049} & \textbf{.0276$\pm$.0049} & 0.00$\pm$0.00\\
Dir. & \textbf{.5469$\pm$.0174} & \textbf{.6388$\pm$.0320} & \textbf{.0395$\pm$.0044} & \textbf{.0395$\pm$.0044} & 0.00$\pm$0.00 & .5970$\pm$.0155 & \textbf{.7428$\pm$.0183} & .0298$\pm$.0059 & .0297$\pm$.0061 & 0.00$\pm$0.00\\
\bottomrule
\end{tabular}%
}
\endgroup
\end{table}
\begin{table}[!htbp]
\caption{SUPPORT2 ($\gU^{(3)}$): HUC-Boost results by base predictor. Values are mean $\pm$ sample standard deviation over five data splits. UC-Boost and HUC-Boost are abbreviated as UC-B and HUC-B, respectively.}
\label{tab:support-u3-base-huc-v2}
\centering
\begingroup
\scriptsize
\setlength{\tabcolsep}{1.0pt}
\renewcommand{\arraystretch}{1.0}
\resizebox{\linewidth}{!}{%
\begin{tabular}{@{}l*{5}{c}@{\hspace{5pt}}*{5}{c}@{}}
\toprule
& \multicolumn{5}{c}{\textbf{LR}} & \multicolumn{5}{c}{\textbf{RF}}\\
\cmidrule(lr){2-6}\cmidrule(lr){7-11}
Method & Acc. $\uparrow$ & AUC $\uparrow$ & UC $\downarrow$ & HUC $\downarrow$ & Updates & Acc. $\uparrow$ & AUC $\uparrow$ & UC $\downarrow$ & HUC $\downarrow$ & Updates\\
\midrule
HUC-B & .6086$\pm$.0121 & .7580$\pm$.0208 & .0282$\pm$.0019 & .0271$\pm$.0017 & 12.80$\pm$7.16 & .6100$\pm$.0166 & .7459$\pm$.0397 & .0342$\pm$.0108 & .0342$\pm$.0107 & 19.20$\pm$12.77\\
HUC-B (0.05) & \textbf{.6112$\pm$.0138} & .7608$\pm$.0197 & .0285$\pm$.0034 & .0285$\pm$.0034 & 72.80$\pm$54.21 & .6156$\pm$.0109 & .7485$\pm$.0314 & .0332$\pm$.0095 & .0332$\pm$.0095 & 120.00$\pm$90.77\\
Temp.$\to$HUC & .6098$\pm$.0117 & .7597$\pm$.0200 & .0303$\pm$.0052 & .0303$\pm$.0052 & 9.60$\pm$15.13 & .6158$\pm$.0097 & \textbf{.7499$\pm$.0330} & .0306$\pm$.0070 & .0318$\pm$.0067 & 16.80$\pm$17.30\\
Vec.$\to$HUC & .6088$\pm$.0112 & .7555$\pm$.0230 & .0296$\pm$.0062 & .0278$\pm$.0038 & 8.00$\pm$10.20 & \textbf{.6162$\pm$.0117} & .7281$\pm$.0297 & \textbf{.0305$\pm$.0077} & \textbf{.0301$\pm$.0082} & 16.80$\pm$12.46\\
Dir.$\to$HUC & .6098$\pm$.0118 & \textbf{.7641$\pm$.0171} & \textbf{.0256$\pm$.0033} & \textbf{.0256$\pm$.0033} & 0.80$\pm$1.79 & .6138$\pm$.0097 & .7267$\pm$.0230 & .0314$\pm$.0071 & .0316$\pm$.0071 & 20.00$\pm$12.33\\
\midrule
& \multicolumn{5}{c}{\textbf{GNB}} & \multicolumn{5}{c}{\textbf{XGB}}\\
\cmidrule(lr){2-6}\cmidrule(lr){7-11}
Method & Acc. $\uparrow$ & AUC $\uparrow$ & UC $\downarrow$ & HUC $\downarrow$ & Updates & Acc. $\uparrow$ & AUC $\uparrow$ & UC $\downarrow$ & HUC $\downarrow$ & Updates\\
\midrule
HUC-B & .5075$\pm$.0185 & .6407$\pm$.0165 & .0443$\pm$.0107 & .0421$\pm$.0065 & 86.40$\pm$4.56 & .6042$\pm$.0069 & .7484$\pm$.0111 & .0348$\pm$.0077 & .0347$\pm$.0077 & 22.40$\pm$8.76\\
HUC-B (0.05) & .0809$\pm$.0247 & .5756$\pm$.0171 & .5067$\pm$.0277 & .3901$\pm$.0190 & 367.20$\pm$73.34 & .6058$\pm$.0110 & .7490$\pm$.0109 & .0318$\pm$.0062 & .0318$\pm$.0062 & 141.60$\pm$48.71\\
Temp.$\to$HUC & .5483$\pm$.0198 & .6441$\pm$.0409 & \textbf{.0317$\pm$.0033} & \textbf{.0307$\pm$.0032} & 16.00$\pm$7.48 & .6048$\pm$.0133 & .7458$\pm$.0124 & .0361$\pm$.0088 & .0361$\pm$.0088 & 19.20$\pm$27.77\\
Vec.$\to$HUC & .5451$\pm$.0236 & .6556$\pm$.0352 & .0376$\pm$.0137 & .0376$\pm$.0137 & 18.40$\pm$7.80 & .6098$\pm$.0075 & .7430$\pm$.0124 & \textbf{.0286$\pm$.0038} & \textbf{.0286$\pm$.0038} & 9.60$\pm$8.29\\
Dir.$\to$HUC & \textbf{.5485$\pm$.0218} & \textbf{.6655$\pm$.0234} & .0362$\pm$.0092 & .0364$\pm$.0091 & 23.20$\pm$13.68 & \textbf{.6112$\pm$.0082} & \textbf{.7629$\pm$.0141} & .0295$\pm$.0070 & .0299$\pm$.0064 & 6.40$\pm$6.69\\
\midrule
& \multicolumn{5}{c}{\textbf{DT}} & \multicolumn{5}{c}{\textbf{MLP}}\\
\cmidrule(lr){2-6}\cmidrule(lr){7-11}
Method & Acc. $\uparrow$ & AUC $\uparrow$ & UC $\downarrow$ & HUC $\downarrow$ & Updates & Acc. $\uparrow$ & AUC $\uparrow$ & UC $\downarrow$ & HUC $\downarrow$ & Updates\\
\midrule
HUC-B & .5241$\pm$.0160 & .6256$\pm$.0227 & .0465$\pm$.0076 & .0443$\pm$.0052 & 124.80$\pm$27.33 & .5980$\pm$.0143 & .6721$\pm$.0333 & \textbf{.0302$\pm$.0063} & \textbf{.0282$\pm$.0037} & 5.60$\pm$8.29\\
HUC-B (0.05) & .4981$\pm$.0220 & .6089$\pm$.0274 & .2463$\pm$.0260 & .2463$\pm$.0260 & 398.40$\pm$2.19 & \textbf{.6024$\pm$.0104} & .6637$\pm$.0471 & .0333$\pm$.0093 & .0320$\pm$.0077 & 92.80$\pm$98.48\\
Temp.$\to$HUC & .5481$\pm$.0148 & .6454$\pm$.0208 & .0335$\pm$.0052 & .0335$\pm$.0052 & 41.60$\pm$29.61 & .5962$\pm$.0203 & .6738$\pm$.0317 & .0325$\pm$.0088 & .0317$\pm$.0078 & 19.20$\pm$26.89\\
Vec.$\to$HUC & \textbf{.5520$\pm$.0151} & \textbf{.6582$\pm$.0267} & \textbf{.0312$\pm$.0047} & \textbf{.0316$\pm$.0053} & 28.80$\pm$14.81 & .5998$\pm$.0145 & .6626$\pm$.0381 & .0310$\pm$.0093 & .0301$\pm$.0082 & 4.00$\pm$5.66\\
Dir.$\to$HUC & .5489$\pm$.0180 & .6503$\pm$.0300 & .0371$\pm$.0072 & .0369$\pm$.0074 & 19.60$\pm$21.00 & .6010$\pm$.0162 & \textbf{.7394$\pm$.0253} & .0341$\pm$.0081 & .0329$\pm$.0049 & 9.60$\pm$10.43\\
\bottomrule
\end{tabular}%
}
\endgroup
\end{table}
\subsubsection{Additional Results}
\label{subsec:support2-additional-results}
The nested classes add one decision at a time, including all three thresholds for that decision:
\[
\begin{aligned}
\gU^{(1)}&=\{u_1^{0.25},u_1^{0.50},u_1^{0.75}\},\\
\gU^{(2)}&=\gU^{(1)}\cup\{u_2^{0.25},u_2^{0.50},u_2^{0.75}\},\\
\gU^{(3)}&=\gU^{(2)}\cup\{u_3^{0.25},u_3^{0.50},u_3^{0.75}\}.
\end{aligned}
\]
Thus $\gU^{(1)}$ contains only the comparison concerning preparation for death within two months; $\gU^{(2)}$ adds the comparison concerning assistance with daily living; and $\gU^{(3)}$ additionally includes the intensive-support comparison. Their numbers of utilities are $3,6,9$. The unions of their relevant internal nodes are, respectively, $\{v_0\}$, $\{v_0,v_1,v_2\}$, and $\{v_0,v_1,v_2,v_3\}$, of sizes $1,3,4$.

\paragraph{Utility-class expansion.}
Table~\ref{tab:support2-utility-expansion-updated} compares $\gU^{(1)}$, $\gU^{(2)}$, and $\gU^{(3)}$, varying the target class rather than fixing it to $\gU^{(3)}$ as in the main-text comparison.
HUC-Boost keeps the resulting HUC at a similar level as the class expands, while the number of selected updates grows with the number of target utilities and relevant nodes.  For the two larger classes, HUC-Boost yields lower mean HUC than direct UC-Boost.  The fixed-step correction has higher mean HUC for every pooled class, and its retained update count grows as the class expands.

\begin{table}[!htbp]
\caption{Supplementary SUPPORT2 utility-class expansion. Each correction is fitted separately from Baseline for each target class. HUC-B uses $\alpha=\Gamma/\Lambda$; HUC-B (0.05) uses constant step magnitude $0.05$. UC-Boost and HUC-Boost are abbreviated as UC-B and HUC-B, respectively.}
\label{tab:support2-utility-expansion-updated}
\centering
\begingroup
\scriptsize
\setlength{\tabcolsep}{3pt}
\renewcommand{\arraystretch}{1.0}
\resizebox{\linewidth}{!}{%
\begin{tabular}{@{}lcccccccc@{}}
\toprule
\shortstack{Target\\class} & Utilities & \shortstack{Relevant\\nodes} & \multicolumn{3}{c}{HUC error after correction} & \multicolumn{3}{c}{Updates} \\
\cmidrule(lr){4-6}\cmidrule(lr){7-9}
&  &  & UC-B & HUC-B & HUC-B (0.05) & UC-B & HUC-B & HUC-B (0.05) \\
\midrule
$\mathcal U^{(1)}$ & 3 & 1 & .0333$\pm$.0083 & .0343$\pm$.0081 & .1264$\pm$.1451 & 11.57$\pm$11.41 & 11.90$\pm$11.79 & 47.77$\pm$36.81 \\
$\mathcal U^{(2)}$ & 6 & 3 & .0557$\pm$.0431 & .0346$\pm$.0081 & .1268$\pm$.1448 & 29.30$\pm$34.74 & 32.70$\pm$35.19 & 138.60$\pm$108.89 \\
$\mathcal U^{(3)}$ & 9 & 4 & .0504$\pm$.0313 & .0351$\pm$.0089 & .1270$\pm$.1444 & 33.90$\pm$39.40 & 45.20$\pm$46.89 & 198.80$\pm$148.38 \\
\bottomrule
\end{tabular}
}
\endgroup
\end{table}

\subsection[MASSIVE]{MASSIVE~\citep{FitzGerald2023MASSIVE,bastianelli-etal-2020-slurp}}
\label{subsec:massive-data}
\paragraph{Data, features, and label hierarchy.}
MASSIVE contains $16{,}521$ English utterances with 60 intent labels grouped into 18 scenarios. The root has one child for each scenario, and the children of each scenario are its intent labels. Thus every root-to-leaf path has the form root--scenario--intent. If $\mathcal S$ denotes the set of 18 scenario nodes, then
\begin{equation}
\operatorname{ch}(\rho)=\mathcal S,\qquad
\operatorname{ch}(s)=\gY_s\quad(s\in\mathcal S),\qquad
\gY=\bigsqcup_{s\in\mathcal S}\gY_s,\quad |\gY|=60.
\label{eq:massive-tree-structure}
\end{equation}
Each utterance is represented by TF--IDF (term frequency--inverse document frequency) features. TF--IDF weights term frequencies within each utterance by inverse document frequency, reducing the influence of terms shared by many utterances~\citep{DBLP:books/daglib/0021593}.
We concatenate word unigram and bigram TF--IDF features with character three- to five-gram TF--IDF features. The two vocabularies are restricted to at most $20{,}000$ and $30{,}000$ features, respectively, and logistic regression is trained as the base predictor. We compare post-hoc calibration methods using the probability outputs of this TF--IDF logistic-regression predictor as their common input.

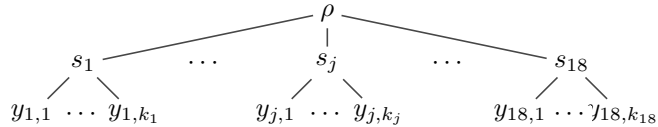
\begin{figure}[htbp]
\centering
\begin{tikzpicture}[x=0.85cm,y=0.75cm,
 every node/.style={font=\small,inner sep=1pt,fill=white},
 treeedge/.style={draw=black!72,line width=0.60pt,shorten >=1.6pt,shorten <=1.6pt}]
\node (r) at (0,0) {$\rho$};
\node (s1) at (-3.8,-0.9) {$s_1$};
\node (sj) at (0,-0.9) {$s_j$};
\node (s18) at (3.8,-0.9) {$s_{18}$};
\node at (-1.9,-0.9) {$\cdots$}; \node at (1.9,-0.9) {$\cdots$};
\foreach \a in {s1,sj,s18}{\draw[treeedge] (r)--(\a);}
\foreach \a/\xx/\lab in {s1/-3.8/1,sj/0/j,s18/3.8/18}{
 \node (\a L) at (\xx-0.8,-1.8) {$y_{\lab,1}$};
 \node (\a R) at (\xx+0.8,-1.8) {$y_{\lab,k_{\lab}}$};
 \node at (\xx,-1.8) {$\cdots$};
 \draw[treeedge] (\a)--(\a L); \draw[treeedge] (\a)--(\a R);
}
\end{tikzpicture}
\caption{MASSIVE label tree. Dots omit scenario and intent nodes. The intents below $s_j$ are $\gY_{s_j}=\{y_{j,1},\ldots,y_{j,k_j}\}$, with $\sum_{j=1}^{18}k_j=60$; each intent belongs to exactly one scenario.}
\label{fig:massive-line-tree}
\end{figure}

\paragraph{Data allocation.}
The official training data are divided in proportions $0.60/0.20/0.20$ for base training, initial recalibration, and correction by UC-Boost or HUC-Boost. The development data are divided equally for initial-recalibrator selection and validation; the official test data are used for evaluation. Results are averaged over five data splits.

\paragraph{Utilities.}
For prediction $p$, first choose the most probable scenario and then the most probable intent inside it:
\begin{equation}
\widehat s(p)=\operatorname*{arg\,max}_{\prec;\,s\in\mathcal S}\pi_s(p),
\qquad
\widehat y(p)=\operatorname*{arg\,max}_{\prec;\,y\in\gY_{\widehat s(p)}}p_y.
\label{eq:massive-selected-scenario-intent}
\end{equation}
The scenario confidence is $\pi_{\widehat s(p)}(p)$, and the intent confidence is the conditional probability $p_{\widehat y(p)}/\pi_{\widehat s(p)}(p)$, not its unconditional leaf probability. For each pair $(\tau_s,\tau_i)\in\{0.55,0.65,0.75\}^2$, the action is
\begin{equation}
a_{\tau_s,\tau_i}(p)=
\begin{cases}
\text{abstain},&\pi_{\widehat s(p)}(p)<\tau_s,\\
\widehat s(p),&\pi_{\widehat s(p)}(p)\ge\tau_s\text{ and }
                 p_{\widehat y(p)}/\pi_{\widehat s(p)}(p)<\tau_i,\\
\widehat y(p),&\pi_{\widehat s(p)}(p)\ge\tau_s\text{ and }
                 p_{\widehat y(p)}/\pi_{\widehat s(p)}(p)\ge\tau_i.
\end{cases}
\label{eq:massive-action}
\end{equation}
For a realized intent $y$, the resulting utility is
\begin{equation}
u_{\tau_s,\tau_i}(p,y)=
\begin{cases}
0,&a_{\tau_s,\tau_i}(p)=\text{abstain},\\
\tfrac14,&a_{\tau_s,\tau_i}(p)=\widehat s(p),\ y\in\gY_{\widehat s(p)},\\
-1,&a_{\tau_s,\tau_i}(p)=\widehat s(p),\ y\notin\gY_{\widehat s(p)},\\
1,&a_{\tau_s,\tau_i}(p)=\widehat y(p),\ y=\widehat y(p),\\
-\tfrac12,&a_{\tau_s,\tau_i}(p)=\widehat y(p),\ y\in\gY_{\widehat s(p)}\setminus\{\widehat y(p)\},\\
-1,&a_{\tau_s,\tau_i}(p)=\widehat y(p),\ y\notin\gY_{\widehat s(p)}.
\end{cases}
\label{eq:massive-complete-utility}
\end{equation}
Hence $|\gU|=9$. The fixed alphabetical order $\prec$ on scenario and intent labels resolves all ties.

Write $s_p:=\pi_{\widehat s(p)}(p)$ and $t_p:=p_{\widehat y(p)}$. The predicted masses of the correct intent, the other intents in the selected scenario, and the intents outside that scenario are $t_p$, $s_p-t_p$, and $1-s_p$, respectively. Averaging Eq.~\eqref{eq:massive-complete-utility} therefore gives
\[
\begin{aligned}
\mu_{u_{\tau_s,\tau_i}}(p)
&=\sum_{z\in\gY}p_z u_{\tau_s,\tau_i}(p,z)\\
&=\begin{cases}
0,&a_{\tau_s,\tau_i}(p)=\text{abstain},\\
\tfrac14 s_p-(1-s_p),&a_{\tau_s,\tau_i}(p)=\widehat s(p),\\
t_p-\tfrac12(s_p-t_p)-(1-s_p),&a_{\tau_s,\tau_i}(p)=\widehat y(p).
\end{cases}
\end{aligned}
\]
The scenario and intent expressions simplify to $\tfrac54s_p-1$ and $\tfrac32t_p+\tfrac12s_p-1$, respectively. This is the predicted mean payoff of the threshold-based reporting rule, rather than a correctness probability or the conditional intent confidence $t_p/s_p$.

\paragraph{Subgroups.}
The family has $|\gC|=3$: $c\equiv1$, $\mathbf{1}\{\text{word count}\le m\}$, and $\mathbf{1}\{\text{word count}>m\}$, where $m$ is the training median of word count.

\paragraph{Results.}
Table~\ref{tab:massive-complete-restored} reports all ten methods. Both HUC-Boost and the fixed-step correction reduce mean HUC relative to Baseline. This dataset does not show the same ordering between the two update rules as the main experiments.

\begin{table}[!htbp]
\caption{MASSIVE results. Values are mean $\pm$ sample standard deviation over five data splits of the TF--IDF logistic-regression predictor. HUC-B (0.05) uses constant step magnitude $0.05$. UC-Boost and HUC-Boost are abbreviated as UC-B and HUC-B, respectively.}
\label{tab:massive-complete-restored}
\centering
\begingroup
\scriptsize
\setlength{\tabcolsep}{3pt}
\renewcommand{\arraystretch}{1.0}
\begin{tabular}{lccccc}
\toprule
Method & Acc. $\uparrow$ & AUC $\uparrow$ & UC $\downarrow$ & HUC $\downarrow$ & Updates \\
\midrule
Baseline & .8222$\pm$.0025 & .9913$\pm$.0008 & .1380$\pm$.0039 & .1190$\pm$.0035 & 0.00$\pm$0.00 \\
UC-B & .8223$\pm$.0026 & .9899$\pm$.0009 & .0216$\pm$.0077 & .0142$\pm$.0043 & 44.60$\pm$33.25 \\
HUC-B & .8197$\pm$.0036 & .9891$\pm$.0009 & .0302$\pm$.0078 & .0170$\pm$.0065 & 251.60$\pm$50.97 \\
HUC-B (0.05) & .8223$\pm$.0027 & .9897$\pm$.0009 & \textbf{.0213$\pm$.0083} & \textbf{.0140$\pm$.0071} & 395.80$\pm$66.34 \\
Temp. & .8222$\pm$.0025 & .9906$\pm$.0007 & .0335$\pm$.0075 & .0189$\pm$.0040 & 0.00$\pm$0.00 \\
Temp.$\to$HUC & .8206$\pm$.0029 & .9904$\pm$.0008 & .0371$\pm$.0057 & .0181$\pm$.0048 & 109.40$\pm$110.03 \\
Vec. & \textbf{.8295$\pm$.0010} & .9909$\pm$.0009 & .0373$\pm$.0066 & .0178$\pm$.0016 & 0.00$\pm$0.00 \\
Vec.$\to$HUC & .8268$\pm$.0014 & .9904$\pm$.0009 & .0476$\pm$.0087 & .0235$\pm$.0070 & 183.00$\pm$71.87 \\
Dir. & .8286$\pm$.0038 & \textbf{.9917$\pm$.0008} & .0714$\pm$.0042 & .0505$\pm$.0035 & 0.00$\pm$0.00 \\
Dir.$\to$HUC & .8260$\pm$.0025 & .9908$\pm$.0007 & .0338$\pm$.0128 & .0229$\pm$.0075 & 224.80$\pm$109.76 \\
\bottomrule
\end{tabular}
\endgroup
\end{table}

\subsection[NSL-KDD]{NSL-KDD~\citep{Tavallaee2009NSLKDD}}
\label{subsec:nsl-kdd-data}
\paragraph{Data, features, and label hierarchy.}
NSL-KDD\footnote{
NSL-KDD dataset:
\url{https://www.unb.ca/cic/datasets/nsl.html}.} uses $25{,}192$ records with five leaf labels: normal, dos, probe, r2l, and u2r. The root separates normal traffic from attacks, and the attack node separates the four attack families. The 41 input attributes give 118 numerical features after categorical encoding. The six base predictors use five stratified data splits, with training/calibration/validation/test proportions $0.45/0.20/0.15/0.20$.

\begin{figure}[htbp]
\centering
\begin{tikzpicture}[x=0.85cm,y=0.75cm,
 every node/.style={font=\small,inner sep=1pt,fill=white},
 treeedge/.style={draw=black!72,line width=0.60pt,shorten >=1.6pt,shorten <=1.6pt}]
\node (r) at (0,0) {$\rho$};
\node (n) at (-2.0,-0.9) {normal};\node (a) at (1.0,-0.9) {attack};
\node (d) at (-1.4,-1.8) {dos};\node (p) at (0.2,-1.8) {probe};
\node (rl) at (1.8,-1.8) {r2l};\node (ur) at (3.4,-1.8) {u2r};
\foreach \a/\b in {r/n,r/a,a/d,a/p,a/rl,a/ur}{\draw[treeedge] (\a)--(\b);}
\end{tikzpicture}
\caption{NSL-KDD label tree. The five leaves are normal traffic and the four attack families.}
\label{fig:nsl-kdd-line-tree}
\end{figure}
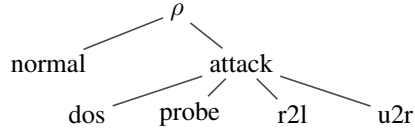

\paragraph{Utilities.}
For a node $v$, let $\gY_v$ be its descendant-leaf set and write $J_v(y):=\mathbf{1}\{y\in\gY_v\}$ and $\pi_v(p):=\sum_{z\in\gY_v}p_z$. Thus $J_v(y)$ records whether the returned category contains the true label.
Set $p_{\mathrm{normal}}:=\pi_{\mathrm{normal}}(p)$. Also set
$\pi_{\mathrm{attack}}(p):=p_{\mathrm{dos}}+p_{\mathrm{probe}}+p_{\mathrm{r2l}}+p_{\mathrm{u2r}}$.
The root decision is
\[
\widehat v_0(p)=
\begin{cases}
\mathrm{normal},&p_{\mathrm{normal}}\ge\pi_{\mathrm{attack}}(p),\\
\mathrm{attack},&p_{\mathrm{normal}}<\pi_{\mathrm{attack}}(p).
\end{cases}
\]
The family decision returns normal when the root selects normal; otherwise it selects the most probable attack leaf:
\[
\widehat v_1(p)=
\begin{cases}
\mathrm{normal},&\widehat v_0(p)=\mathrm{normal},\\
\operatorname*{arg\,max}_{\prec;\,v\in\{\mathrm{dos},\mathrm{probe},\mathrm{r2l},\mathrm{u2r}\}}p_v,
&\widehat v_0(p)=\mathrm{attack}.
\end{cases}
\]
The third decision is $\widehat y(p)=\operatorname*{arg\,max}_{\prec;\,y\in\gY}p_y$ over all five leaves. The three utilities are
\begin{equation}
\begin{aligned}
u_{\mathrm{status}}(p,y)&=J_{\widehat v_0(p)}(y)=\mathbf{1}\{y\in\gY_{\widehat v_0(p)}\},\\
u_{\mathrm{family}}(p,y)&=\mathbf{1}\{y=\widehat v_1(p)\},\qquad
u_{\mathrm{leaf}}(p,y)=\mathbf{1}\{y=\widehat y(p)\}.
\end{aligned}
\label{eq:nsl-complete-utilities}
\end{equation}
The corresponding predicted mean utilities are
\[
\begin{aligned}
&\mu_{u_{\mathrm{status}}}(p)
=\sum_{z\in\gY_{\widehat v_0(p)}}p_z
=\pi_{\widehat v_0(p)}(p), \qquad  \mu_{u_{\mathrm{family}}}(p)
=\sum_{z\in\gY}p_z\mathbf{1}\{z=\widehat v_1(p)\}
=p_{\widehat v_1(p)},\\
&\mu_{u_{\mathrm{leaf}}}(p)
=\sum_{z\in\gY}p_z\mathbf{1}\{z=\widehat y(p)\}
=p_{\widehat y(p)}.
\end{aligned}
\]
The status mean is the total probability of the selected normal/attack category. Since $\widehat v_1(p)$ is a leaf in this tree, the family mean is its unconditional probability, not its probability conditional on attack. The leaf mean is the probability of the globally most probable label.

Thus $|\gU|=3$. The family and global-leaf decisions need not coincide: the root can select attack while the single most probable leaf is normal. Ties at the root select normal; other ties use the order normal, dos, probe, r2l, u2r.

\paragraph{Subgroups.}
The nine indicators ($|\gC|=9$) are the whole population, the three protocol groups, the three most frequent training services, and the two groups with source-byte count at most and greater than its training median. Protocol groups correspond to tcp, udp, and icmp. Each indicator tests only its stated attribute condition; no intersections are added.

\paragraph{Results.}
Table~\ref{tab:nsl-complete-restored} reports all ten methods, with means and sample standard deviations over the six base predictors and five splits.

\begin{table}[!htbp]
\caption{NSL-KDD results. Values are mean $\pm$ sample standard deviation over six base predictors and five data splits. HUC-B (0.05) uses constant step magnitude $0.05$. UC-Boost and HUC-Boost are abbreviated as UC-B and HUC-B, respectively.}
\label{tab:nsl-complete-restored}
\centering
\begingroup
\scriptsize
\setlength{\tabcolsep}{3pt}
\renewcommand{\arraystretch}{1.0}
\begin{tabular}{lccccc}
\toprule
Method & Acc. $\uparrow$ & AUC $\uparrow$ & UC $\downarrow$ & HUC $\downarrow$ & Updates \\
\midrule
Baseline & .9504$\pm$.1004 & .9577$\pm$.0653 & .0471$\pm$.1011 & .0452$\pm$.0980 & 0.00$\pm$0.00 \\
UC-B & .9816$\pm$.0244 & .9687$\pm$.0444 & .0039$\pm$.0034 & .0035$\pm$.0031 & 62.07$\pm$39.19 \\
HUC-B & .9801$\pm$.0283 & .9672$\pm$.0514 & .0039$\pm$.0031 & .0033$\pm$.0026 & 92.37$\pm$69.38 \\
HUC-B (0.05) & .9505$\pm$.1004 & .9621$\pm$.0605 & .0448$\pm$.1014 & .0432$\pm$.0983 & 85.37$\pm$75.29 \\
Temp. & .9504$\pm$.1004 & .9619$\pm$.0563 & .0285$\pm$.0555 & .0275$\pm$.0552 & 0.00$\pm$0.00 \\
Temp.$\to$HUC & .9838$\pm$.0195 & .9766$\pm$.0363 & .0036$\pm$.0028 & .0032$\pm$.0026 & 64.43$\pm$66.43 \\
Vec. & .9735$\pm$.0429 & .9647$\pm$.0553 & .0129$\pm$.0195 & .0125$\pm$.0196 & 0.00$\pm$0.00 \\
Vec.$\to$HUC & \textbf{.9844$\pm$.0181} & \textbf{.9812$\pm$.0324} & \textbf{.0033$\pm$.0025} & \textbf{.0029$\pm$.0023} & 58.93$\pm$58.73 \\
Dir. & .9738$\pm$.0411 & .9652$\pm$.0465 & .0115$\pm$.0178 & .0112$\pm$.0177 & 0.00$\pm$0.00 \\
Dir.$\to$HUC & .9842$\pm$.0176 & .9690$\pm$.0400 & \textbf{.0033$\pm$.0023} & \textbf{.0029$\pm$.0020} & 44.23$\pm$55.22 \\
\bottomrule
\end{tabular}
\endgroup
\end{table}

\setlength{\intextsep}{7pt plus 2pt minus 2pt}
\subsection[N-BaIoT]{N-BaIoT~\citep{Meidan2018NBaIoT}}
\label{subsec:nbaiot-data}
\paragraph{Data, features, and label hierarchy.}
The N-BaIoT\footnote{
Dataset: UCI Machine Learning Repository,
\url{https://doi.org/10.24432/C5RC8J}.} experiment uses $1{,}000$ records, approximately balanced across 11 leaf labels: benign traffic and five attack types for each of Gafgyt and Mirai. The root separates benign traffic from attacks; the attack node separates the two families, whose children are their respective attack types. Classifiers use the first ten numerical features. The six base predictors and training/calibration/validation/test proportions match those for NSL-KDD.

\begin{figure}[htbp]
\centering
\begin{tikzpicture}[x=0.95cm,y=0.75cm,
 every node/.style={font=\footnotesize,inner sep=0.8pt,fill=white},
 treeedge/.style={draw=black!72,line width=0.60pt,shorten >=1.6pt,shorten <=1.6pt}]
\node (r) at (-1.1,0) {$\rho$};
\node (b) at (-4.7,-0.9) {benign}; \node (a) at (0.7,-0.9) {attack};
\node (g) at (-2.3,-1.8) {Gafgyt};\node (m) at (3.0,-1.8) {Mirai};
\node (gc) at (-4.6,-2.75) {combo};\node (gj) at (-3.45,-2.75) {junk};
\node (gs) at (-2.3,-2.75) {scan};\node (gt) at (-1.15,-2.75) {tcp};\node (gu) at (0,-2.75) {udp};
\node (ma) at (0.8,-2.75) {ack};\node (ms) at (1.95,-2.75) {scan};\node (my) at (3.1,-2.75) {syn};
\node (mu) at (4.25,-2.75) {udp};\node (mp) at (5.5,-2.75) {udpplain};
\foreach \a/\b in {r/b,r/a,a/g,a/m,g/gc,g/gj,g/gs,g/gt,g/gu,m/ma,m/ms,m/my,m/mu,m/mp}{\draw[treeedge] (\a)--(\b);}
\end{tikzpicture}
\caption{N-BaIoT label tree. Identically named attack types under different families are different leaves; for example, Gafgyt scan and Mirai scan are distinct labels.}
\label{fig:nbaiot-line-tree}
\end{figure}
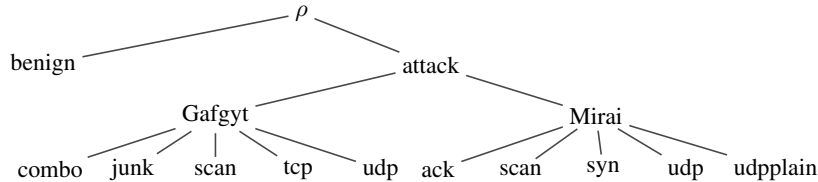

\paragraph{Utilities.}
For a node $v$, let $\gY_v$ be its descendant-leaf set and write $J_v(y):=\mathbf{1}\{y\in\gY_v\}$ and $\pi_v(p):=\sum_{z\in\gY_v}p_z$. The indicator is one exactly when the true label lies in the returned node's subtree.
The root decision $\widehat v_0(p)$ is benign if $p_{\mathrm{benign}}\ge\pi_{\mathrm{attack}}(p)$ and attack otherwise. Conditional on choosing attack, the family decision compares the total masses of the Gafgyt and Mirai subtrees:
\[
\widehat v_1(p)=
\begin{cases}
\mathrm{benign},&\widehat v_0(p)=\mathrm{benign},\\
\operatorname*{arg\,max}_{\prec;\,v\in\{\mathrm{Gafgyt},\mathrm{Mirai}\}}\pi_v(p),
&\widehat v_0(p)=\mathrm{attack}.
\end{cases}
\]
The third decision $\widehat y(p)=\operatorname*{arg\,max}_{\prec;\,y\in\gY}p_y$ selects the most probable of all 11 leaves, independently of these two preceding decisions. The utilities are
\begin{equation}
\begin{aligned}
u_{\mathrm{status}}(p,y)&=J_{\widehat v_0(p)}(y)=\mathbf{1}\{y\in\gY_{\widehat v_0(p)}\},\\
u_{\mathrm{family}}(p,y)&=J_{\widehat v_1(p)}(y)=\mathbf{1}\{y\in\gY_{\widehat v_1(p)}\},\\
u_{\mathrm{type}}(p,y)&=\mathbf{1}\{y=\widehat y(p)\}.
\end{aligned}
\label{eq:nbaiot-complete-utilities}
\end{equation}
The corresponding predicted mean utilities are
\[
\begin{aligned}
&\mu_{u_{\mathrm{status}}}(p)
=\sum_{z\in\gY_{\widehat v_0(p)}}p_z
=\pi_{\widehat v_0(p)}(p), \qquad  \mu_{u_{\mathrm{family}}}(p)
=\sum_{z\in\gY_{\widehat v_1(p)}}p_z
=\pi_{\widehat v_1(p)}(p),\\
&\mu_{u_{\mathrm{type}}}(p)
=\sum_{z\in\gY}p_z\mathbf{1}\{z=\widehat y(p)\}
=p_{\widehat y(p)}.
\end{aligned}
\]
When an attack family is returned, its mean utility is the total unconditional probability of all attack types in that family, not a probability normalized by $\pi_{\mathrm{attack}}(p)$. When benign is returned, this mean reduces to $p_{\mathrm{benign}}$. The type mean is the probability of the selected leaf among all 11 labels.

Thus $|\gU|=3$. A root tie selects benign and a family tie selects Gafgyt. Global-leaf ties use benign first, then Gafgyt types in the order combo, junk, scan, tcp, udp, and then Mirai types in the order ack, scan, syn, udp, udpplain.

\paragraph{Subgroups.}
The indicators are the whole population, each device present in training, and the two groups with first-feature value at most and greater than its training median. Thus the family contains one device indicator per represented device and three additional indicators. Device identity is used only for subgroup membership, not as an eleventh input feature. Indicators with identical values, such as the whole-population and device indicators in a single-device sample, are not merged.

\paragraph{Results.}
Table~\ref{tab:nbaiot-complete-restored} reports all ten methods, with means and sample standard deviations over the six base predictors and five splits. Both the calibration and validation subsets are small.

\begin{table}[!htbp]
\caption{N-BaIoT results. Values are mean $\pm$ sample standard deviation over six base predictors and five data splits. HUC-B (0.05) uses constant step magnitude $0.05$. UC-Boost and HUC-Boost are abbreviated as UC-B and HUC-B, respectively.}
\label{tab:nbaiot-complete-restored}
\centering
\begingroup
\scriptsize
\setlength{\tabcolsep}{3pt}
\renewcommand{\arraystretch}{1.0}
\begin{tabular}{lccccc}
\toprule
Method & Acc. $\uparrow$ & AUC $\uparrow$ & UC $\downarrow$ & HUC $\downarrow$ & Updates \\
\midrule
Baseline & .7852$\pm$.1560 & .9703$\pm$.0382 & .1021$\pm$.0998 & .0672$\pm$.0633 & 0.00$\pm$0.00 \\
UC-B & .7960$\pm$.1262 & .9720$\pm$.0300 & .0338$\pm$.0221 & .0324$\pm$.0197 & 31.50$\pm$30.98 \\
HUC-B & .7987$\pm$.1254 & .9736$\pm$.0307 & .0357$\pm$.0212 & .0278$\pm$.0186 & 35.27$\pm$28.58 \\
HUC-B (0.05) & .7960$\pm$.1281 & .9749$\pm$.0283 & .0415$\pm$.0332 & .0325$\pm$.0316 & 137.33$\pm$80.47 \\
Temp. & .7852$\pm$.1560 & .9711$\pm$.0396 & .0415$\pm$.0354 & .0391$\pm$.0364 & 0.00$\pm$0.00 \\
Temp.$\to$HUC & .8010$\pm$.1264 & .9764$\pm$.0265 & .0326$\pm$.0239 & .0265$\pm$.0189 & 24.80$\pm$33.13 \\
Vec. & .8077$\pm$.1243 & .9770$\pm$.0229 & .0361$\pm$.0293 & .0290$\pm$.0242 & 0.00$\pm$0.00 \\
Vec.$\to$HUC & .8104$\pm$.1133 & .9787$\pm$.0171 & .0363$\pm$.0272 & .0271$\pm$.0160 & 13.57$\pm$22.93 \\
Dir. & \textbf{.8380$\pm$.0439} & \textbf{.9843$\pm$.0043} & .0413$\pm$.0412 & .0269$\pm$.0200 & 0.00$\pm$0.00 \\
Dir.$\to$HUC & .8357$\pm$.0418 & .9841$\pm$.0044 & \textbf{.0311$\pm$.0116} & \textbf{.0243$\pm$.0104} & 16.77$\pm$23.76 \\
\bottomrule
\end{tabular}
\endgroup
\end{table}

\end{document}